%% file: rmp-wafr-2018.tex
\documentclass{svproc}
\usepackage[dvipsnames]{xcolor}

\usepackage{overpic}
\graphicspath{{images/}}
\usepackage{color}
\usepackage{soul}
\usepackage{bm}
\usepackage{multirow}
\usepackage{multicol}
\usepackage{xspace}
\usepackage{graphicx}
\usepackage{float}
\usepackage{subfig}
\usepackage{url}

\usepackage{xspace}
\def\flow{RMP{flow}\xspace}
\def\algebra{RMP-algebra\xspace}
\def\tree{RMP-tree\xspace}

\def\pushforward{\texttt{pushforward}\xspace}
\def\pullback{\texttt{pullback}\xspace}
\def\resolve{\texttt{resolve}\xspace}

\input{shortcuts}

\input{commands}

\usepackage{url}

\newif\ifLONG
\LONGfalse

\newif\ifAPP
\APPtrue

\begin{document}
\mainmatter              
%
\title{RMP\textit{flow}: A Computational Graph for \\ Automatic Motion Policy Generation}
%
\titlerunning{\flow}
%
\vspace{-2mm}\author{
Ching-An Cheng\inst{1,2} 
\and 
Mustafa Mukadam\inst{1,2}
\and 
Jan Issac\inst{1} 
\and 
Stan Birchfield\inst{1}
\and 
\\
Dieter Fox\inst{1,3} 
\and 
Byron Boots\inst{1,2} 
\and 
Nathan Ratliff\inst{1} 
}

\institute{
    NVIDIA, Seattle Robotics Lab, Seattle, WA, USA
    \and
    Georgia Institute of Technology, 
    Robot Learning Lab,
    Atlanta, GA, USA\\
    \and
    University of Washington,
    Robotics and State Estimation Lab,
    Seattle, WA, USA
}

\authorrunning{Cheng, Mukadam, Issac, Birchfield, Fox, Boots, Ratliff} 

\maketitle

\vspace{-4mm}
\begin{abstract} 
%
We develop a novel policy synthesis algorithm, 
\flow, based
on geometrically consistent transformations of Riemannian Motion Policies (RMPs). RMPs are a class of reactive motion policies designed to parameterize non-Euclidean behaviors as dynamical systems in intrinsically nonlinear task spaces.  Given a set of RMPs designed for individual tasks, \flow can consistently combine these
local policies to generate an expressive global policy, while simultaneously exploiting
sparse structure for computational efficiency.  We
study the geometric properties of \flow  and provide sufficient
conditions for stability.
Finally, we experimentally demonstrate that accounting for the 
geometry of task policies can simplify 
classically difficult problems, such as 
planning 
through clutter on
high-DOF manipulation systems.

\vspace{-1mm}
\end{abstract}

\vspace{-8mm}
\section{Introduction}\vspace{-2mm}


In this work, we develop a new  motion generation and control
framework that enables globally stable controller design within {\it
intrinsically}
non-Euclidean spaces.\footnote{Spaces defined by non-constant Riemannian metrics with non-trivial curvature.}
 Non-Euclidean geometries are not often modeled
explicitly in robotics, but are nonetheless common in the natural world.  One important example is the apparent non-Euclidean 
behavior of obstacle avoidance.
Obstacles become holes in this setting. 
As a result,  straight lines are no longer a reasonable definition of shortest distance---geodesics must, therefore, naturally flow around them. This behavior implies a form of non-Euclidean geometry: the space is naturally curved by the presence of obstacles.

The planning literature has made substantial progress in modeling non-Euclidean task-space behaviors, but at the expense of efficiency and reactivity.  
Starting with early differential geometric models of obstacle avoidance~\cite{rimon-ams-1991}
and building toward modern planning algorithms 
and optimization techniques~\cite{RIEMORatliff2015ICRA,VijayakumarTopologyMotionPlanning2013,Watterson-TrajOptManifolds-RSS-18,ToussaintTrajOptICML2009,LavallePlanningAlgorithms06,KaramanRRTStar2011,GammellBitStar2014,mukadam2017continuous},
these techniques can calculate highly nonlinear trajectories. However, they are often computationally intensive, sensitive to noise, and unresponsive to perturbation. In addition, the internal nonlinearities of robots due to kinematic constraints are sometimes simplified in the optimization.

At the same time, a separate thread of literature, emphasizing fast reactive control over computationally expensive planning, developed efficient closed-loop
control techniques such as Operational Space Control
(OSC)~\cite{khatib1987unified}. But while these techniques account for internal geometries from the robot's kinematic structure, they assume simple Euclidean geometry in task spaces~\cite{Peters_AR_2008,UdwadiaGaussPrincipleControl2003}, failing to provide a complete treatment of the external geometries. 
As a result, obstacle avoidance, e.g., has to rely on \emph{extrinsic} potential functions, leading to undesirable deacceleartion behavior when the robot is close to the obstacle. If the non-Euclidean geometry can be \emph{intrinsically} considered, then fast obstacle avoidance motion would naturally arise as traveling along the induced geodesic. 
The need for a holistic solution to motion generation and control has motivated a number of recent system architectures tightly integrating planning and control~\cite{2017_rss_system,Mukadam-ICRA-17}.


We develop a new approach to synthesizing control policies that can 
accommodate and leverage the modeling capacity of 
intrinsically non-Euclidean robotics tasks. 
Taking inspiration from Geometric Control
Theory~\cite{bullo2004geometric},\footnote{ See
Appendix~\ref{apx:WhyGeometricMechanics} for a discussion of why geometric
mechanics and geometric control theory constitute a good starting point.} we
design a novel recursive algorithm, \flow, based on a recently proposed
mathematical object for representing nonlinear policies known as the Riemannian
Motion Policy (RMP)~\cite{ratliff2018riemannian}. This algorithm
enables the
geometrically consistent fusion of many component policies defined across non-Euclidean task spaces that are related through a tree structure.  We show that \flow, which
generates behavior by calculating how the robot should accelerate, 
mimics the 
Recursive Newton-Euler algorithm~\cite{walker1982efficient} in structure,
but generalizes it beyond rigid-body systems to a broader class of highly-nonlinear 
transformations and spaces.

In contrast to existing frameworks, 
our framework 
naturally models non-Euclidean task spaces with Riemannian metrics that are not
only configuration dependent, but also \emph{velocity} dependent.
This allows \flow to consider, e.g., the \emph{direction} a robot travels to define the importance weights in combing policies.
For example, an obstacle, despite
being close to the robot, can usually be ignored if robot is heading away from
it.
This new class of policies leads to an extension of Geometric Control Theory,
building on a new class of non-physical mechanical systems 
we call Geometric Dynamical Systems (GDS).

We also show that \flow is Lyapunov-stable and coordinate-free. In particular, when using \flow,
robots can be viewed each as different parameterizations of the same task space, defining a precise notion of behavioral
consistency between robots. 
Additionally, under this framework, the implicit curvature arising from
non-constant Riemannian metrics (which may be roughly viewed as position-velocity dependent inertia matrices in OSC) produces nontrivial and intuitive policy contributions
that 
are critical to guaranteeing stability and generalization across embodiments. 
%
Our experimental results illustrate how these { curvature terms} can be
impactful in practice, generating nonlinear geodesics that
result in curving or orbiting around obstacles. 
Finally, we demonstrate the utility of our framework with a fully reactive real-world system 
on multiple dual-arm manipulation problems.

\vspace{-3mm}
\section{Motion Generation and Control}
\vspace{-2mm}
Motion generation and control can be formulated as the problem of transforming curves from the configuration space $\CC$ to the task space $\TT$. 
Specifically, let 
$\CC$ be a $d$-dimensional smooth manifold. A robot's motion can be described as a curve $q:[0,\infty) \to \CC$ such that the robot's configuration at time $t$ is a point $q(t) \in \CC$.
Without loss of generality, suppose $\CC$ has a global coordinate $\q: \CC \to \R^d$, called the \emph{generalized coordinate}; for short, we would identify the curve $q$ with its coordinate 
and write $\q(q(t))$ as $\q(t) \in \R^d$. 
A typical example of the generalized coordinate is the joint angles of a $d$-DOF (degrees-of-freedom) robot: we denote $\q(t)$ as the joint angles at time $t$ and $\qd(t)$, $\qdd(t)$ as the joint velocities and accelerations. 
To describe the tasks, we consider another manifold $\TT$, the task space, which is related to the configuration space $\CC$ through a smooth \emph{task map} $\psi: \CC \to \TT$. The task space $\TT$ can be the end-effector position/orientation~\cite{khatib1987unified,albu2002cartesian}, or more generally can be a space that describes whole-body robot motion, e.g., in simultaneous tracking and collision avoidance~\cite{sentis2006whole,lo2016virtual}. 
Thus, the goal of motion generation and control is to design the curve $q$ so that the transformed curve $\psi \circ q $ exhibits desired behaviors on the task space $\TT$.

\vspace{-2mm}
\paragraph{Notation}
For clarity, we use boldface to distinguish the coordinate-dependent representations from abstract objects; e.g. we write $ q(t) \in \CC$ and $\q(t) \in \R^d$.  In addition, we will often omit the time- and input-dependency of objects unless necessary; e.g. we may write $ q \in \CC$ and $(\q,\qd, \qdd)$.
For derivatives, we use both symbols $\nabla$ and $\partial$, with a transpose relationship: for $\x \in \R^m$ and a differential map $\y:\R^m \to \R^n$, we write $\nabla_\x \y(\x) = \partial_\x \y(\x)^\t \in \R^{m \times n}$. 
For a matrix $\M \in \R^{m\times m}$, we denote $\mbb_i = (\Mb)_i$ as its $i$th column and $M_{ij} = (\Mb)_{ij}$ as its $(i,j)$ element. To compose a matrix, 
we use $(\cdot)_{\cdot}^\cdot$ for vertical (or matrix) concatenation  and $[\cdot]_{\cdot}^\cdot$ for horizontal concatenation. For example, we write $\M = [\mbb_i]_{i=1}^m = (M_{ij})_{i,j=1}^m$ and $\M^\t = (\mbb_i^\t)_{i=1}^m =  (M_{ji})_{i,j=1}^m$. We use $\R^{m\times m}_{+}$ and $\R^{m\times m}_{++}$ to denote the symmetric, positive semi-definite/definite matrices, respectively.

\vspace{-3mm}
\subsection{Motion Policies and the Geometry of Motion}
\vspace{-2mm}
We model motion as a second-order
differential equation\footnote{ We assume the system has been feedback linearized. 
A torque-based setup can be similarly derived by setting the robot inertia matrix as the intrinsic metric on $\CC$~\cite{Peters_AR_2008}.} 
of $\qdd = \pi(\q, \qd)$, where we call $\pi$ a \textit{motion policy} and $(\q, \qd)$  the \emph{state}. 
In contrast to an open-loop trajectory, which forms the basis
of many motion planners, a motion policy expresses the entire continuous collection of its integral trajectories{} and therefore is robust to perturbations. 
Motion policies can model many adaptive behaviors, such as reactive obstacle avoidance
~\cite{DRCIntegratedSystemTodorov2013,2017_rss_system} or responses driven by planned Q-functions~\cite{OptimalControlTheoryTodorov06}, and their  second-order formulation enables rich behavior that cannot be realized by the velocity-based approach~\cite{liegeois1977automatic}.

The geometry of motion has been considered by many planning and control algorithms.
Geometrical modeling of task spaces is used in topological motion
planning~\cite{VijayakumarTopologyMotionPlanning2013}, and motion optimization
has leveraged Hessian to exploit the natural geometry of costs~\cite{RatliffCHOMP2009,ToussaintTrajOptICML2009,Mukadam-ICRA-16,Dong-RSS-16}.
Ratliff et al.~\cite{RIEMORatliff2015ICRA}, e.g., use the workspace geometry inside a Gauss-Newton optimizer and generate natural obstacle-avoiding reaching motion through traveling along geodesics of curved spaces.

Geometry-aware motion policies were also developed in parallel in controls. OSC is the best example~\cite{khatib1987unified}. 
Unlike the planning approaches, OSC focuses on the internal geometry of the robot and considers only simple task-space geometry. 
It reshapes the workspace dynamics into a simple spring-mass-damper system with a {constant} inertia matrix, enforcing a form of Euclidean geometry in the task space.
Variants of OSC have been proposed to consider different metrics~\cite{Nakanishi_IJRR_2008,Peters_AR_2008,lo2016virtual}, task hierarchies~\cite{sentis2006whole,platt2011multiple}, and non-stationary inputs~\cite{IjspeertDMPs2013}.

While these algorithms have led to many advances, we argue that their isolated focus on either the internal or the external geometry limits the performance. The planning approach fails to consider reactive dynamic behavior; the control approach cannot model the effects of velocity dependent metrics, which are critical to generating sensible obstacle avoidance motions, as discussed in the introduction.
 While the benefits of velocity dependent metrics was recently explored using RMPs~\cite{ratliff2018riemannian}, a systematic understanding is still an open question.

\vspace{-4mm}
\section{Automatic Motion Policy Generation with \flow} \label{sec:RMP algebra}
\vspace{-2mm}
\flow is an efficient manifold-oriented computational graph for automatic generation of motion policies. It is aimed for problems with a task space $\TT = \{\TT_{l_i}\}$ that is related to the configuration space $\CC$ through a tree-structured task map $\psi$, where $\TT_{l_i}$ is the $i$th subtask.
Given user-specified motion policies $\{\pi_{l_i}\}$ on 
$\{\TT_{l_i}\}$ as RMPs, \flow is designed to \emph{consistently} combine these subtask policies into a global policy $\pi$ on $\CC$. To this end, \flow introduces 1) a data structure, called the \emph{\tree}, to describe the tree-structured task map $\psi$ and the policies, and 2) a set of operators, called the \emph{\algebra}, to propagate information across the \tree. 
To compute $\pi(\q(t),\qd(t))$ at time $t$, \flow operates in two steps: it first performs a \emph{forward pass} to propagate the state from the root node (i.e. $\CC$) to the leaf nodes (i.e. $\{\TT_{l_i}\}$); then it performs a \emph{backward pass} to propagate the RMPs from the leaf nodes to the root node{}. 
These two steps are realized by recursive use of \algebra, exploiting shared computation paths arising from the tree structure to maximize efficiency.


\vspace{-4mm}
\subsection{Structured Task Maps}
\vspace{-2mm}
\begin{figure}[t]\vspace{-4mm}
	\centering
	\centering
	\subfloat[\label{fig:taskmap1}]{
		\includegraphics[trim={0 30 0 20},clip, width=0.25\columnwidth,keepaspectratio]{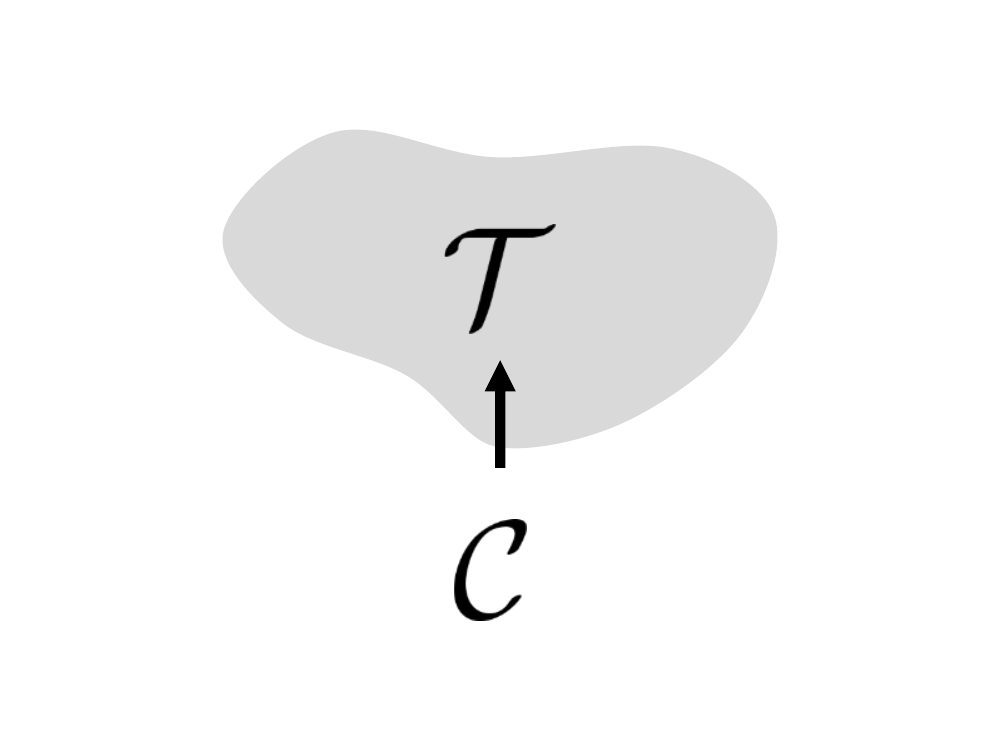}
	}
	\subfloat[\label{fig:taskmap2}]{
	\includegraphics[trim={0 30 0 20},clip, width=0.25\columnwidth,keepaspectratio]{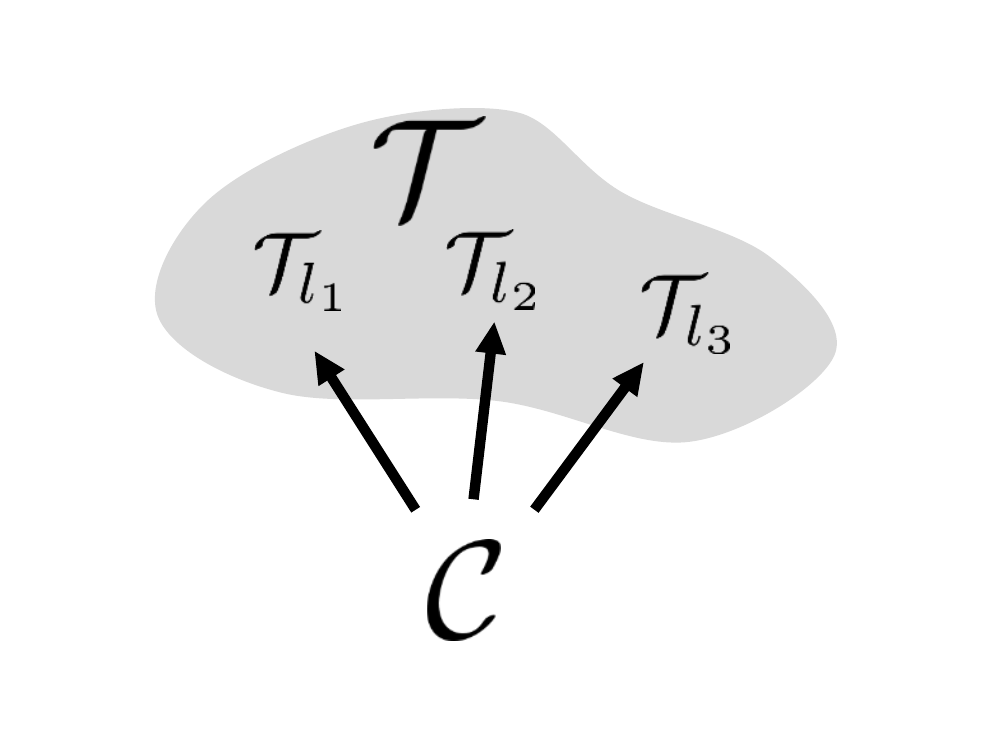}
	}
	\subfloat[\label{fig:taskmap3}]{
	\includegraphics[trim={0 20 0 20},clip, width=0.25\columnwidth,keepaspectratio]{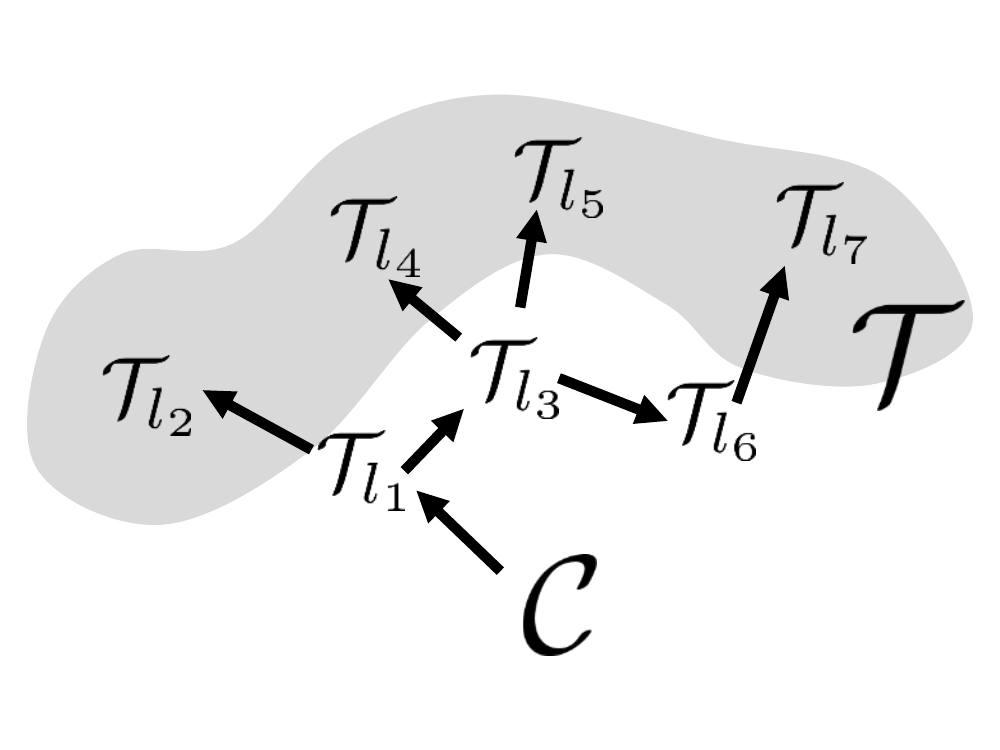}
	}
	\vspace{-2mm}
	\caption{Tree-structured task maps}
	\label{fig:taskmaps}
	\vspace{-6mm}
\end{figure}
In most cases, the task-space manifold $\TT$ is structured. In this paper, we consider the case where the task map $\psi$ can be expressed through a tree-structured composition of transformations $\{\psi_{e_i}\}$, where $\psi_{e_i}$ is the $i$th transformation. Fig.~\ref{fig:taskmaps} illustrates some common examples. Each node denotes a manifold and each edge denotes a transformation. This family trivially includes the unstructured task space $\TT$ (Fig.~\ref{fig:taskmaps}a) and the product manifold $\TT = \TT_{l_1} \times \dots \times \TT_{l_K}$ (Fig.~\ref{fig:taskmaps}b), where $K$ is the number of subtasks. A more interesting example is the kinematic tree (Fig.~\ref{fig:taskmaps}c), where, e.g., 
the subtask spaces on the leaf nodes can describe the tracking and obstacle avoidance tasks along a multi-DOF robot. 

The main motivation of explicitly handling the structure in the task map $\psi$ is two-fold. First, it allows \flow to exploit computation shared across different subtask maps. Second, it allows the user to focus on designing motion policies for each subtask individually, which is easier than directly designing a global policy for the entire task space $\TT$. For example, $\TT$ may describe the problem of humanoid walking, which includes staying balanced, scheduling contacts, and avoiding collisions. Directly parameterizing a policy to satisfy all these objectives can be daunting, whereas designing a policy for each subtask is more feasible. 

\vspace{-4mm}
\subsection{Riemannian Motion Policies (RMPs)}
\vspace{-1mm}

Knowing the structure of the task map is not sufficient for consistently combining subtask policies: we require some geometric information about the motion policies' behaviors~\cite{ratliff2018riemannian}. 
Toward this end, we adopt an abstract description of motion policies, called RMPs~\cite{ratliff2018riemannian},  for the nodes of the \tree. 
Specifically, let $\MM$ be an $m$-dimensional manifold with coordinate $\x \in \R^m$. 
The \emph{canonical form} of an RMP on $\MM$ is a pair $(\ab, \M)^\MM$, where 
$\ab : \R^m \times \R^m \to \R^m$ is a continuous motion policy and $\M: \R^m \times \R^m \to \R_+^{m\times m}$ is a differentiable map. 
Borrowing terminology from mechanics, we call $\ab(\x,\xd)$ the \emph{desired acceleration} and $\M(\x,\xd)$ the \emph{inertia matrix} at $(\x,\xd)$, respectively.\footnote{Here we adopt a slightly different terminology from~\cite{ratliff2018riemannian}. We note that $\M$ and $\f$ do not necessarily correspond to the inertia and force of a physical mechanical system. }	
$\M$ defines the directional importance of $\ab$ when it is combined with other motion policies. Later in Section~\ref{sec:analysis}, we will show that $\M$ is closely related to Riemannian metric, which describes how the space is stretched along the curve generated by $\ab$; when $\M$ depends on the state, the space becomes \emph{non-Euclidean}.
We additionally introduce a new RMP form, called the \emph{natural form}. Given an RMP in its canonical form $(\ab, \M)^\MM$, the natural form is a pair $[\f, \M]^\MM$, where $\f = \M \ab$ is the \emph{desired force} map. While the transformation between these two forms may look trivial, their distinction will be useful later when we introduce the \algebra. 

\vspace{-4mm}
\subsection{\tree}
\vspace{-1mm}

The \tree is the core data structure used by \flow. 
An \tree is a directed tree, in which each node represents an RMP and its state, and each edge corresponds to a transformation between manifolds. 
The root node of the \tree describes the global policy $\pi$ on $\CC$, and the leaf nodes describe the local policies $\{\pi_{l_i}\}$ on $\{\TT_{l_i}\}$. 
To illustrate, let us consider a node $u$ and its $K$ child nodes $\{v_i \}_{i=1}^K$. Suppose $u$ describes an RMP $[\f, \M ]^\MM$ and $v_i$ describes an RMP $ [\f_i, \M_i ]^{\NN_i}$, where $\NN_i = \psi_{e_i}(\MM)$ for some $\psi_{e_i}$. Then we write $u = ((\x,\xd), [\f, \M ]^\MM)$ and $v_i = ((\y_i,\yd_i), [\f_{i}, \M_i ]^{\NN_i})$; the edge connecting $u$ and $v_i$ points from  $u$ to $v_i$ along $\psi_{e_i}$. We will continue to use this example to illustrate how \algebra propagates the information across the \tree.

\subsection{\algebra} \label{sec:RMPAlgebra}
\vspace{-1mm}
The \algebra consists of three operators (\pushforward, \pullback, and \resolve) to propagate information.\footnote{Precisely it propagates the numerical values of RMPs and states at a particular time.} They form the basis of the forward and backward passes for automatic policy generation, described in the next section.
\begin{enumerate}
\item  \pushforward is the operator to forward propagate the \emph{state} from a parent node to its child nodes. Using the previous example, given $(\x,\xd)$ from $u$, it computes $(\y_i, \yd_i) = (\psi_{e_i}(\x) , \J_i (\x) \xd )$ for each child node $v_i$, where $\J_i = \partial_\x \psi_{e_i}$ is a Jacobian matrix. The name ``pushforward" comes from the linear transformation of tangent vector $\xd$ to the image tangent vector $\yd_i$. 

\item \pullback is the operator to backward propagate the natural-formed RMPs from the child nodes to the parent node. It is done by setting $[\f, \M ]^\MM$ with 
\begin{align} \label{eq:natural pullback}
\textstyle
\f = \sum_{i=1}^{K} \J_i^\t (\f_{i} - \M_i \dot{\J}_i \xd)
\qquad \text{and} \qquad
\M = \sum_{i=1}^{K} \J_i^\t \M_i \J_i
\end{align}
The name ``pullback" comes from the linear transformations of the cotangent vector (1-form) $\f_{i} - \M_i \dot{\J_i} \xd$  and the inertia matrix (2-form) $\M_i$.  
In summary, velocities can be pushfowarded along the direction of $\psi_i$, and forces and inertial matrices can be pullbacked in the opposite direction. 

To gain more intuition of \pullback, we write \pullback in the canonical form of RMPs. It can be shown that the canonical form $(\ab, \M)^{\MM}$ of the natural form $[\f,\M]^\MM$ above is the solution to a  least-squared problem:
\begin{align} \label{eq:least-square problem of pullback}
\ab 
&=\textstyle  
	\argmin_{\ab'} \frac{1}{2} \sum_{i=1}^{K} \norm{\J_i \ab' + \dot\J_i \xd -  \ab_{i} }_{\M_i}^2 
\end{align}
where $ \ab_i = \M_i^{\dagger}\f_i$ and $\norm{\cdot}_{\M_i}^2 = \lr{\cdot}{\M_i \cdot}$. Because $\ydd_i =  \J_i \xdd + \dot\J_i \xd$, \pullback attempts to find an $\ab$ that can realize the desired accelerations $\{\ab_{i}\}$ while trading off approximation errors with an importance weight defined by the inertia matrix $\M_i(\y_i,\yd_i)$. 
The use of state dependent importance weights is a distinctive feature of \flow. It allows \flow to activate different RMPs according to \emph{both} configuration and velocity (see Section~\ref{sec:example RMPs} for examples).
Finally, we note that the \pullback operator defined in this paper is slightly different from the original definition given in~\cite{ratliff2018riemannian}, which ignores the term $\dot\J_i \xd$ in~\eqref{eq:least-square problem of pullback}. While ignoring $\dot\J_i \xd$ does not necessary destabilize the system~\cite{lo2016virtual}, its inclusion is critical to implement consistent policy behaviors. {}

\item \resolve is the last operator of \algebra. It maps an RMP from its natural form to its canonical form. Given $[\f, \M]^{\MM}$, it outputs $(\ab, \M)^{\MM}$ with $\ab = \M^{\dagger} \f$, where $\dagger$ denotes Moore-Penrose inverse. The use of pseudo-inverse is because in general the inertia matrix is only positive semi-definite. Therefore, we also call the natural form of $[\f, \M]^{\MM}$ the \emph{unresolved form}, as potentially it can be realized by multiple RMPs in the canonical form. 
\end{enumerate}

\vspace{-4mm}
\subsection{Algorithm: Motion Policy Generation}
\vspace{-2mm}
Now we show how \flow uses the \tree and \algebra to generate a global policy $\pi$ on $\CC$. Suppose each subtask policy is provided as an RMP. 
First, we construct an \tree with the same structure as $\psi$, where we assign subtask RMPs as the leaf nodes and the global RMP $[\f_r, \M_r]^\CC$ as the root node. 
With the \tree specified, \flow can perform automatic policy generation. At every time instance, it first performs a forward pass: it recursively calls \pushforward from the root node to the leaf nodes to update the state information in each node in the \tree. 
Second, it performs a backward pass: it recursively calls \pullback from the leaf nodes to the root node to back propagate the values of the RMPs in the natural form, and finally calls \resolve at the root node to transform the global RMP $[\f_r, \M_r]^\CC$ into its canonical form $(\ab_r, \M_r)^\CC$ for policy execution (i.e. setting $\pi(\q,\qd) = \ab_r$).

The process of policy generation of \flow uses the tree structure for computational efficiency. For $K$ subtasks, it has time complexity $O(K)$ in the worst case as opposed to $O(K\log K)$ of a naive implementation which does not exploit the tree structure.
Furthermore, all computations of \flow are carried out using matrix-multiplications, except for the final \resolve call, because the RMPs are expressed in the natural form in \pullback instead of the canonical form suggested originally in~\cite{ratliff2018riemannian}. 
This design makes \flow numerically stable, as only one matrix inversion $\M_r^\dagger \f_r$ is performed at the root node with both $\f_r$ and $\M_r$  in the span of the same Jacobian matrix due to \pullback. 

\vspace{-4mm}
\subsection{Example RMPs} \label{sec:example RMPs}
\vspace{-2mm}
We give a quick overview of some RMPs useful in practice (a complete discussion of these RMPs are postponed to Appendix~\ref{apx:Practice}). 
We recall from~\eqref{eq:least-square problem of pullback} that $\M$ dictates the directional importance of an RMP.

\vspace{-4mm}
\subsubsection{Collision/joint limit avoidance}
Barrier-type RMPs are examples that use velocity dependent inertia matrices, which can express importance as a function of robot heading (a property that traditional mechanical principles fail to capture).
Here we demonstrate a collision avoidance policy in the 1D distance space $x = d(\q)$ to an obstacle. Let $g(x,\dot{x}) = w(x)u(\dot{x}) > 0$ for some functions $w$ and $u$. We consider a motion policy such that $ m(x,\dot{x})\ddot{x}+ \frac{1}{2}\dot{x}^2\partial_x g(x,\dot{x}) 
= -\partial_x \Phi(x) - b\dot{x}$ and define its inertia matrix $m(x,\dot{x}) = g(x,\dot{x}) + \frac{1}{2}\dot{x}\partial_{\dot{x}}g(x,\dot{x})$, where $\Phi$ is a potential and $b>0$ is a damper.  
We choose $w(x)$ to increase as $x$ decreases
(close to the obstacle),  $u(\dot{x})$ to increase 
when $\dot{x}<0$ (moving toward the obstacle), and $u(\dot{x})$ to be constant when $\dot{x}\geq0$.
With this choice, the RMP can be turned off in \pullback when the robot heads away from the obstacle. 
This motion policy is a GDS and $g$ is its metric (cf. Section~\ref{sec:GDS}); the terms $\frac{1}{2}\dot{x}\partial_{\dot{x}}g(x, \dot{x})$ and $\frac{1}{2}\dot{x}^2\partial_x g(x, \dot{x})$ are due to non-Euclidean geometry and produce natural repulsive behaviors.

\vspace{-4mm}
\subsubsection{Target attractors}
Designing an attractor policy is relatively straightforward. For a task space with coordinate $\x$, we can consider an inertia matrix $\M(\x) \succ 0$ and a motion policy such that
$
\xdd = -\nabla\wt{\Phi} - \beta(\x)\xd - \M^{-1}\bm\xi_\M
$, where $\wt{\Phi}(\x) \approx \|\x\|$ is a smooth attractor potential, $\beta(\x)\geq0$ is a damper, and $\bm\xi_{\M}$ is a curvature term. 
It can be shown that this differential equation is also a GDS (see Appendix~\ref{apx:Attractors}).

\vspace{-4mm}
\subsubsection{Orientations} 
As \flow directly works with manifold objects, orientation controllers become straightforward to design, independent of the choice of coordinate (cf. Section~\ref{sec:geometric properties}). 
For example, we can define RMPs on a robotic link's surface in any preferred coordinate (e.g. in one or two axes attached to an arbitrary point) with the above described attractor to control the orientation.
This follows a similar idea outlined in the Appendix of~\cite{ratliff2018riemannian}.

\vspace{-4mm}
\subsubsection{Q-functions}
Perhaps surprising, RMPs can be constructed using Q-functions as metrics (we invite readers to read~\cite{ratliff2018riemannian} for details on how motion optimizers can be reduced to Q-functions and the corresponding RMPs).
While these RMPs may not satisfy the conditions of a GDS that we later analyze, they represent a broader class of RMPs that leads to substantial benefits (e.g. escaping local minima) in practice. Also, Q-functions are closely related to Lyapunov functions and geometric control~\cite{lewis2000geometry}; we will further explore this direction in future work.

\vspace{-4mm}
\section{Theoretical Analysis of \flow} \label{sec:analysis}
\vspace{-3mm}
\newcommand{\sdot}[2]{\overset{\lower0.1em\hbox{$\scriptscriptstyle #2$}}{#1}}

We investigate the properties of \flow when the child-node motion policies belong to a class of differential equations, which we call \emph{structured geometric dynamical systems} (structured GDSs). We present the following results.
\begin{enumerate}
	\item {\bf Closure}: We show that the \pullback operator retains a closure of structured GDSs. When the child-node motion policies are structured GDSs, the parent-node dynamics also belong to the same class.
	
	\item {\bf Stability}: Using the closure property, we provide sufficient conditions for the feedback policy of \flow to be stable. In particular, we cover a class of dynamics with \emph{velocity-dependent} metrics that are new to the literature.
	
	\item {\bf Invariance}: As its name suggests, \flow is closely related to differential geometry. We show that 
	\flow is intrinsically coordinate-free. This means that a set of subtask RMPs designed for one robot can be transferred to another robot  while maintaining the same task-space behaviors.  
\end{enumerate}


\vspace{-4mm}
\paragraph{Setup}
\ifLONG
\else
We assume that all manifolds and maps are sufficiently smooth. 
For now, we also assume that each manifold  has a single chart; the coordinate-free analysis is postponed to Section~\ref{sec:geometric properties}. All the proofs are provided in Appendix~\ref{app:proof of analysis}.
\fi

\vspace{-4mm}
\subsection{Geometric Dynamical Systems (GDSs)} \label{sec:GDS}
\vspace{-2mm}
We define a new family of dynamics useful to specify RMPs on manifolds. 
Let manifold $\MM$ be $m$-dimensional with chart $(\MM, \x)$. Let $\Gb: \R^m \times \R^m \to \R^{m\times m}_{+}$, $\B: \R^m \times \R^m \to \R^{m\times m}_{+}$, and $\Phi: \R^m \to \R$.
The tuple $(\MM, \Gb, \B, \Phi)$ is called a \emph{GDS} if and only if
\begin{align} \label{eq:GDS}
\left(\Gb(\x,\xd) + \bm\Xi_{\G}(\x,\xd)\right) \xdd 
+ \bm\xi_{\G}(\x,\xd)  = - \nabla_\x \Phi(\x) - \Bb(\x,\xd)\xd,
\end{align}
where
$\bm\Xi_{\G}(\x,\xd) \coloneqq \frac{1}{2} \sum_{i=1}^m  \dot{x}_i \partial_{\xd} \gb_{i}(\x,\xd)$, $\bm\xi_{\G}(\x,\xd) \coloneqq  \sdot{\Gb}{\x}(\x,\xd) \xd - \frac{1}{2} \nabla_\x (\xd^\t \Gb(\x,\xd) \xd)$, and
$\sdot{\Gb}{\xb}(\x,\xd) \coloneqq  [\partial_{\x}  \gb_{i} (\x,\xd) \xd]_{i=1}^m$. 
We refer to $\Gb(\x,\xd)$ as the \emph{metric} matrix,
$\B(\x,\xd)$ as the \emph{damping} matrix, and $\Phi(\x)$ as the \emph{potential} function which is lower-bounded. 
In addition, we define $\Mb(\x, \xd) \coloneqq \Gb(\x,\xd) + \bm\Xi_{\G}(\x,\xd)$ as the \emph{inertia} matrix, which can be asymmetric.
We say a GDS is \emph{non-degenerate} if $\M(\x,\xd) $ is nonsingular. We will assume~\eqref{eq:GDS} is non-degenerate so that it uniquely defines a differential equation and discuss the general case in Appendix~\ref{app:GDSs}.
$\Gb(\x,\xd)$ induces a metric of $\xd$, 
measuring its length as $\frac{1}{2} \xd^\t \Gb(\x,\xd) \xd$. 
When $\Gb(\x,\xd)$ depends on $\x$ and $\xd$, it also induces the \emph{curvature} terms $\bm\Xi(\x,\xd)$ and $\bm\xi(\x,\xd)$. 
In a particular case when $\G(\x,\xd) = \G(\x)$, the GDSs reduce to the widely studied \emph{simple mechanical systems} (SMSs)~\cite{bullo2004geometric},
$\Mb(\x) \xdd 
+  \C(\x,\xd)\xd + \nabla_\x \Phi(\x) = -\Bb(\x,\xd)\xd$; in this case $\Mb(\x) = \Gb(\x)$
and the Coriolis force $\Cb(\x,\xd) \xd$ is equal to $\bm\xi_{\Gb}(\x,\xd)$.
The extension to velocity-dependent $\G(\x,\xd)$ is important and non-trivial. 
As discussed in Section~\ref{sec:example RMPs}, it generalizes the dynamics of classical rigid-body systems, allowing the space to morph according to the velocity direction.

As its name suggests, GDSs possess geometric properties. Particularly, when $\Gb(\x,\xd)$ is invertible, the left-hand side of~\eqref{eq:GDS} is related to a quantity $\ab_{\G} = \xdd + \Gb(\x,\xd)^{-1} ( \bm\Xi_{\G}(\x,\xd) \xdd 
+ \bm\xi_{\G}(\x,\xd) ) $, known as the \emph{geometric acceleration} (cf. Section~\ref{sec:geometric properties}). 
In short, we can think of~\eqref{eq:GDS} as setting $\ab_{\G}$ along the negative natural gradient $-\G(\x,\xd)^{-1}\nabla_\x \Phi(\x)$ while imposing damping $-\G(\x,\xd)^{-1}\Bb(\x,\xd)\xd$.


\vspace{-4mm}
\subsection{Closure} \label{sec:consistency}
\vspace{-2mm}

	Earlier, we mentioned that by tracking the geometry in \pullback in~\eqref{eq:natural pullback}, the task properties can be preserved. Here, we 
formalize the consistency of \flow as a closure of differential equations, named structured GDSs. Structured GDSs augment GDSs with information on how the metric matrix factorizes. 
Suppose $\Gb$ has a structure $\SS$ that factorizes $\G(\x,\xd) = \J(\x)^\t \Hb(\y,\yd) \J(\x)$, where 
$\y: \x \mapsto \y(\x) \in \R^n$  and $\Hb: \R^n \times \R^n \to \R^{n\times n}_+$, and $\J(\x) = \partial_\x \y$. 
We say the tuple $(\MM, \G, \B, \Phi)_{\SS}$ is a \emph{structured GDS} if and only if
\begin{align} \label{eq:structured GDS}
\left(\Gb(\x,\xd) + \bm\Xi_{\G}(\x,\xd)\right) \xdd 
+ \bm\eta_{\G;\SS}(\x,\xd)  = - \nabla_\x \Phi(\x) - \Bb(\x,\xd)\xd 
\end{align}
where
$
\bm\eta_{\G;\SS}(\x,\xd) 
\coloneqq  \J(\x)^\t ( \bm\xi_{\Hb}(\y,\yd) + 
(\Hb(\y,\dot\y) + \bm\Xi_{\Hb}(\y,\yd) ) 
\dot\J(\x,\xd) \xd  )
$. 
Note the metric and factorization \emph{in combination} defines $\bm\eta_{\G;\SS}$. 
As a special case, GDSs are  structured GDSs with a \emph{trivial} structure (i.e. $\y =\x$). Also, structured GDSs reduce to GDSs (i.e. the structure offers no extra information) if $\G(\x,\xd)=\G(\x)$, or if $n,m=1$ (cf. Appendix~\ref{app:proof of consistency}).  
Given two structures, we say $\SS_a$ \emph{preserves} $\SS_b$ if $\SS_a$ has the factorization (of $\Hb$) made by $\SS_b$.
In Section~\ref{sec:geometric properties}, we will show that structured GDSs are related to a geometric object, pullback connection, which turns out to be the coordinate-free version of \pullback.

\ifLONG\else{To show the closure property, }\fi we consider a parent node on $\MM$ with $K$ child nodes on $\{\NN_i\}_{i=1}^K$. \ifLONG\else{ We }\fi note that $\G_i$ and $\B_i$ can be functions of both $\y_i$ and $\yd_i$. \vspace{-4mm}
\begin{restatable}{theorem}{theoremConsistency} \label{th:consistency}
Let the $i$th child node follow $(\NN_i, \G_i, \B_i, \Phi_i)_{\SS_i}$ and have coordinate $\y_i$. 
Let $\fb_i = -\bm\eta_{\G_i;\SS_i} - \nabla_{\y_i} \Phi_i - \B_{i}\yd_i $ and $\M_i =\G_i + \bm\Xi_{\G_i}$.
If $[\fb,\Mb]^\MM$ of the parent node is given by \emph{\pullback} with $\{[\fb_i, \M_i]^{\NN_i} \}_{i=1}^K$ and $\Mb$ is non-singular, the parent node
follows 
$(\MM, \G, \B, \Phi)_\SS$, 
where $\Gb = \sum_{i=1}^{K}\J_i^\t\G_i\J_i$, $\B = \sum_{i=1}^{K}\J_i^\t \B_i \J_i$, $\Phi =  \sum_{i=1}^{K}\Phi_i \circ \y_i $, $\SS$ preserves $\SS_i$, and $\J_i = \partial_\x \y_i$.
Particularly, if $\G_i$ is velocity-free and the child nodes are GDSs, the parent node follows $(\MM, \G, \B, \Phi)$.
\end{restatable}
\noindent Theorem~\ref{th:consistency} shows structured GDSs are closed under \pullback. 
It means that the differential equation of a structured GDS with a tree-structured task map can be computed by recursively applying \pullback from the leaves to the root. 
\begin{corollary}\label{cr:consistency}
If all leaf nodes follow GDSs and $\Mb_r$ at the root node is nonsingular, then the root node follows $(\CC, \G, \B, \Phi)_{\SS}$ as recursively defined by Theorem~\ref{th:consistency}.
\end{corollary}\vspace{-2mm}




\vspace{-4mm}
\subsection{Stability} \label{sec:stability}
\vspace{-2mm}

By the closure property above, we analyze the stability of \flow  when the leaf nodes are (structured) GDSs. For compactness, we will abuse the notation to write $\Mb = \Mb_r$. Suppose $\Mb$ is nonsingular and let $(\CC, \G, \B, \Phi)_{\SS}$ be the resultant structured GDS at the root node. 
We consider a Lyapunov candidate
$V(\q, \qd) = \frac{1}{2} \qd^\t \G(\q,\qd) \qd + \Phi(\q)$
and derive its rate using properties of structured GDSs.
\begin{restatable}{proposition}{propositionLyapunovTimeDerivative}
 \label{pr:Lyapunov time derivative}
For $(\CC, \G, \B, \Phi)_{\SS}$,  $\dot\V(\q,\qd) = - \qd^\t \B(\q,\qd) \qd$. 
\end{restatable}\vspace{-1mm}
\noindent Proposition~\ref{pr:Lyapunov time derivative} directly implies the stability of structured GDSs by invoking LaSalle's invariance principle~\cite{khalil1996noninear}. Here we summarize the result without proof.
\begin{corollary} \label{cr:stability}
For $(\CC, \G, \B, \Phi)_{\SS}$, if $\G(\q,\qd), \B(\q,\qd) \succ 0 $,  the system converges to a forward invariant set $\CC_\infty \coloneqq \{(\q,\qd) : \nabla_\q \Phi(\q) = 0, \qd = 0 \}$. 
\end{corollary}\vspace{-1mm}
To show the stability of \flow, we need to further check when the assumptions in Corollary~\ref{cr:stability} hold. 
The condition  $\B(\q,\qd) \succ 0 $ is easy to satisfy: by Theorem~\ref{th:consistency},  
$\B(\q,\qd) \succeq 0$; to strictly ensure definiteness, we can copy $\CC$ into an additional child node with a (small) positive-definite damping matrix. The condition on $\Gb(\q,\qd) \succ 0 $ can be satisfied similarly.
In addition, we need to verify the assumption that $\M$ is nonsingular. Here we provide a sufficient condition. When satisfied, it implies the global stability of \flow. 
\vspace{-1mm}
\begin{restatable}{theorem}{theoremVelocityMetric}
\label{th:condition on velocity metric}
Suppose every leaf node is a GDS with a metric matrix in the form
$\Rb(\x) +  \Lb(\x)^\t \D(\x, \xd) \Lb(\x)$ for differentiable functions $\Rb$, $\Lb$, and $\D$ satisfying
\ifLONG
\else
$\Rb(\x) \succeq 0$, $\D(\x,\xd) = \diag
((d_{i}(\x,\dot{y}_i))_{i=1}^n) \succeq 0$, and $ \dot{y}_i \partial_{\dot{y}_i} d_{i}(\x,\dot{y}_i) \geq 0$,
\fi
where $\x$ is the coordinate of the leaf-node manifold and $\yd = \Lb \xd \in \R^n$. 
It holds $\bm\Xi_\G(\q,\qd) \succeq 0$. If further $\G(\q,\qd), \B(\q,\qd) \succ 0$, then $\M\in\R^{d\times d}_{++}$, and the global RMP generated by \flow converges to the forward invariant set $\CC_\infty$ in Corollary~\ref{cr:stability}.
\end{restatable}\vspace{-1mm}
\noindent A particular condition in Theorem~\ref{th:condition on velocity metric} is when all the leaf nodes with velocity dependent metric are 1D. Suppose $x\in\R$ is its coordinate and $g(x,\dot{x})$ is its metric matrix. The sufficient condition essentially boils down to $g(x,\dot{x})\geq0$ and $\dot{x} \partial_{\dot{x}} g(x,\dot{x})\geq 0 $. This means that, given any  $x \in \R$, $g(x,0) = 0$, $g(x,\dot{x})$ is non-decreasing when $\dot{x}>0$, and non-increasing when $\dot{x}<0$. 
This condition is satisfied by the collision avoidance policy in Section~\ref{sec:example RMPs}.

\vspace{-4mm}
\subsection{Invariance }  \label{sec:geometric properties}
\vspace{-2mm}

\newcommand{\lsup}[2]{^{\scriptstyle #2}{#1}}
\newcommand{\pr}{\mathrm{pr}}
\newcommand{\ppartial}[1]{\frac{\partial}{\partial #1}}
\def\conn{\lsup{\nabla}{G}}
\newcommand{\Conn}[1]{\lsup{\nabla}{#1}}
\def\d{\mathrm{d}}

We now discuss the coordinate-free geometric properties of $(\CC, \G, \B, \Phi)_\SS$ generated by  \flow. Due to space constraint, we only summarize the results 
(please see  Appendix~\ref{app:coordinate-free notation} and,
e.g.,~\cite{lee2009manifolds}). Here we assume that $\G$ is positive-definite. 

We first present the coordinate-free version of GDSs (i.e. the structure is trivial) by using a geometric object called \emph{affine connection}, which defines how tangent spaces on a manifold are related.
Let $T\CC$ denote the tangent bundle of $\CC$, which is a natural manifold to describe the state space. 
We first show that a GDS on $\CC$ can be written in terms of a unique, asymmetric affine connection $\conn$ that is compatible with a Riemannian metric $G$ (defined by $\G$) on $T\CC$. It is important to note that $G$ is defined on $T\CC$ \emph{not} the original manifold $\CC$. As the metric matrix in a GDS can be velocity dependent, we need a larger manifold.

\begin{restatable}{theorem}{theoremGeometricAcceleration} \label{th:geometric acceleration}
Let $G$ be a Riemannian metric on $T\CC$ such that, for $s = (q,v) \in T\CC$,  $G(s) = G^v_{ij}(s) dq^i \otimes  d q^j +  G^a_{ij} dv^i \otimes  dv^j$, where $G^v_{ij}(s)$ and $G^a_{ij}$ are symmetric and positive-definite, and $G^v_{ij}(\cdot)$ is differentiable.
Then there is a unique affine connection $\conn$ that is compatible with $G$ and satisfies, 
$\Gamma_{i,j}^k = \Gamma_{ji}^k$, 
$\Gamma_{i,j+d}^k = 0$, 
and $\Gamma_{i+d,j+d}^k = \Gamma_{j+d,i+d}^k$, for $i,j = 1,\dots,d$ and $k = 1,\dots, 2d$.
In coordinates, if $G_{ij}^v(\dot{q})$ is identified as $\G(\q,\qd)$, then 
 $\pr_{3} (\conn_{\ddot{q}} \ddot{q})$ can be written as $\ab_\G \coloneqq \qdd +  \Gb(\q,\qd)^{-1} (\bm\xi_{\G}(\q,\qd) + \bm\Xi_{\G}(\q,\qd) \qdd )$, where $\pr_{3}: (\q,\vb,\ub,\ab) \mapsto \ub$ is a projection.
\end{restatable}\vspace{-1mm}
\noindent We call $\pr_{3} (\lsup{\nabla}{G}_{\dot{q}} \dot{q})$ the \emph{geometric acceleration} of $q(t)$ with respect to $\conn$. It is a coordinate-free object, because $\pr_{3}$ is defined independent of the choice of chart of $\CC$. By Theorem~\ref{th:geometric acceleration}, it is clear that a GDS can be written abstractly as 
$ \pr_3(\lsup{\nabla}{G}_{\ddot{q}} \ddot{q}) = (\pr_3 \circ G^{\sharp} \circ F) (s)$,
where $F: s \mapsto -d\Phi(s) - B(s) $ defines the covectors due to the potential function and damping, and $G^{\sharp} : T^*T\CC \to TT\CC$  denotes the inverse of $G$. In coordinates, 
it reads as 
$
\qdd +  \Gb(\q,\qd)^{-1} (\bm\xi_{\G}(\q,\qd) + \bm\Xi_{\G}(\q,\qd) \qdd )  =  -\Gb(\q,\qd)^{-1} (\nabla_\q \Phi(\q) + \Bb(\q,\qd)\qd )
$, which is exactly~\eqref{eq:GDS}.

Next we present a coordinate-free representation of \flow.
\vspace{-1mm}
\begin{restatable}{theorem}{theoremAbstractConsistency}\label{th:consistency abstract}
Suppose $\CC$ is related to $K$ leaf-node task spaces by maps $\{\psi_i: \CC \to \TT_i\}_{i=1}^K$ and  the $i$th task space $\TT_i$ has an affine connection $\Conn{G_i}$ on $T \TT_i$, as defined in Theorem~\ref{th:geometric acceleration}, and a covector function $F_i$ defined by some potential and damping as described above. 
Let $\lsup{\bar{\nabla}}{G} = \sum_{i=1}^{K} T\psi_i^* {\Conn{G_i}}$ be the pullback connection, $G = \sum_{i=1}^K T\psi_i^* G_i$ be the pullback metric, and $F = \sum_{i=1}^{K} T\psi_i^* F_i$ be the pullback covector, where $T\psi_i^*: T^*T\TT_i \to T^*T\CC$. 
Then $\lsup{\bar{\nabla}}{G}$ is compatible with $G$, and 
$\pr_{3} (\lsup{\bar{\nabla}}{G} _{\ddot{q}} \ddot{q})=  (\pr_3 \circ G^{\sharp} \circ F) (s) $ can be written as $ \qdd +  \Gb(\q,\qd)^{-1} (\bm\eta_{\G;\SS}(\q,\qd) + \bm\Xi_{\G}(\q,\qd) \qdd ) =  -\Gb(\q,\qd)^{-1} (\nabla_\q \Phi(\q) + \Bb(\q,\qd)\qd )$.
In particular, if $G$ is velocity-independent, then $\lsup{\bar{\nabla}}{G} = \conn$.
\end{restatable}\vspace{-1mm}
\noindent Theorem~\ref{th:consistency abstract} says that the structured GDS $(\CC, \G, \B, \Phi)_\SS$ can be written abstractly, without coordinates, using the pullback of task-space covectors, metrics, and asymmetric affine connections (that are defined in Theorem~\ref{th:geometric acceleration}).
In other words, the recursive calls of \pullback in the backward pass of \flow is indeed performing ``pullback'' of geometric objects. 
Theorem~\ref{th:consistency abstract} also shows, when $G$ is velocity-independent, the pullback of connection and the pullback of metric commutes.
In this case, $\lsup{\bar{\nabla}}{G} = \conn$, which is equivalent to the Levi-Civita connection of $G$. The loss of commutativity in general is due to the asymmetric definition of the connection in Theorem~\ref{th:geometric acceleration}, which however is necessary to derive a control law of acceleration, without further referring to  higher-order time derivatives.

\vspace{-3mm}
\subsection{Related Approaches} \label{sec:related work}
\vspace{-1mm}

While here we focus on the special case of \flow with GDSs, this family already covers a wide range of reactive policies commonly used in practice. 
For example, when the task metric is Euclidean (i.e. constant), \flow recovers OSC (and its variants)~\cite{khatib1987unified,sentis2006whole,Peters_AR_2008,UdwadiaGaussPrincipleControl2003,lo2016virtual}. When the task metric is only configuration dependent, \flow can be viewed as performing energy shaping to combine multiple SMSs in geometric control~\cite{bullo2004geometric}. 
Further, \flow allows using velocity dependent metrics, generating behaviors all those previous rigid mechanics-based approaches fail to model.
We also note that \flow can be easily modified to incorporate exogenous time-varying inputs (e.g. forces to realize impedance control~\cite{albu2002cartesian} or learned perturbations as in DMPs~\cite{IjspeertDMPs2013}).
In computation, the structure of \flow in natural-formed RMPs resembles the classical Recursive Newton-Euler algorithm~\cite{walker1982efficient,Featherstone08} (see Appendix~\ref{app:relationship with dynamics}). 
Alternatively, the canonical form of \flow in~\eqref{eq:least-square problem of pullback} resembles Gauss' Principle~\cite{Peters_AR_2008,UdwadiaGaussPrincipleControl2003}, but with a curvature correction $\bm\Xi_\G$ on the inertia matrix (suggested by Theorem~\ref{th:consistency}) to account for velocity dependent metrics. 
Thus, we can view \flow as a natural generalization of these approaches to a broader class of non-Euclidean behaviors.

\vspace{-4mm}
\section{Experiments} \label{sec:experiments}
\vspace{-3mm}

We perform controlled experiments to study the curvature effects of nonlinear
metrics, which is important for stability and collision avoidance.  We then
perform several full-body experiments 
 (video: 
{\small \url{https://youtu.be/Fl4WvsXQDzo}})
to demonstrate the capabilities of \flow
on high-DOF manipulation problems in clutter, and implement an integrated
vision-and-motion system on two physical robots.

\vspace{-4mm}
\subsection{Controlled Experiments}
\label{sec:1DExample}
\begin{figure}[t]
	\centering
	\subfloat[\label{fig:1d_z}]{
		\includegraphics[trim={5 5 0 25},clip, width=0.24\columnwidth,keepaspectratio]{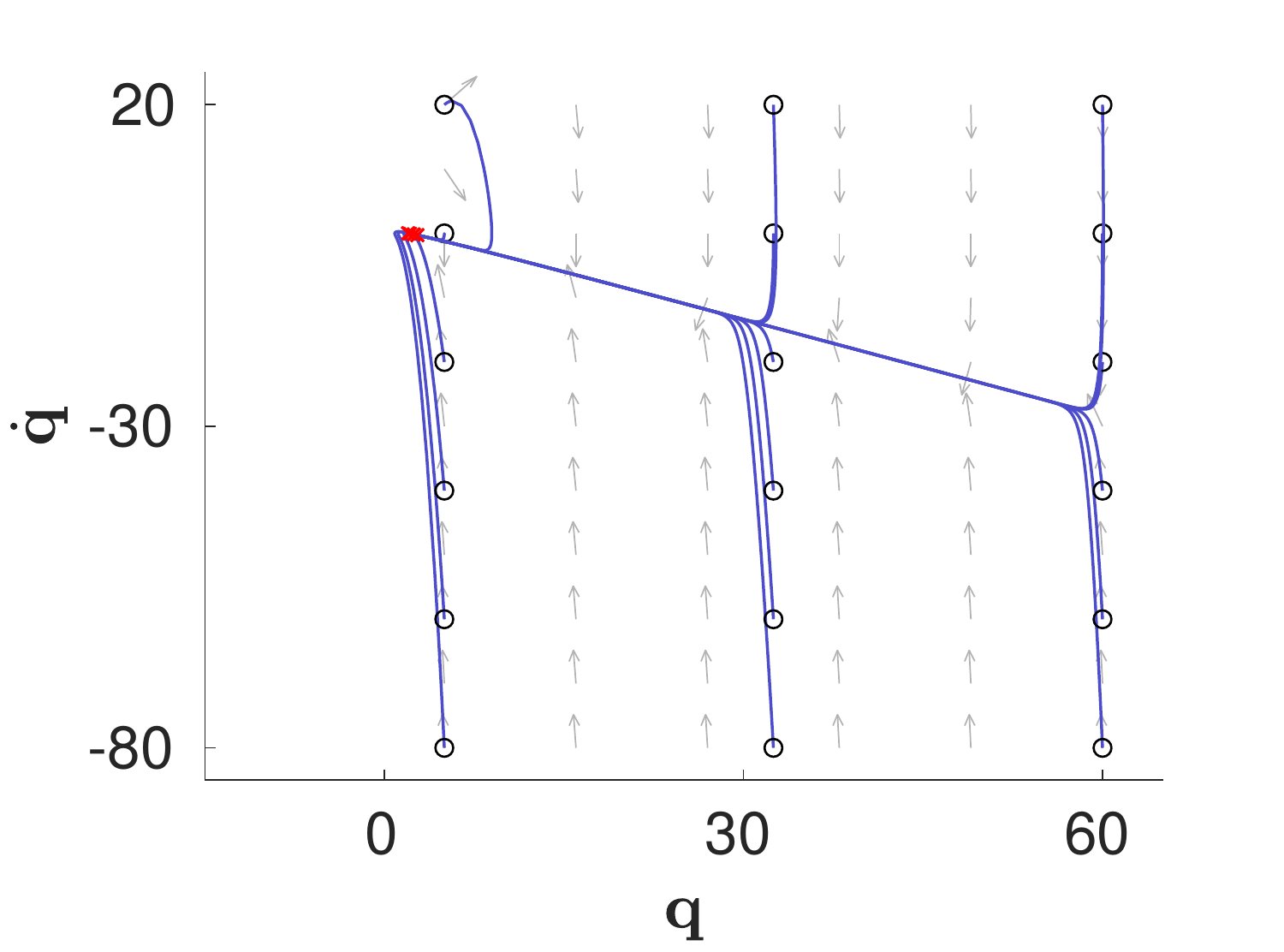}
	}\hspace{-4mm}
	\subfloat[\label{fig:1d_x}]{
		\includegraphics[trim={5 5 0 25},clip, width=0.24\columnwidth,keepaspectratio]{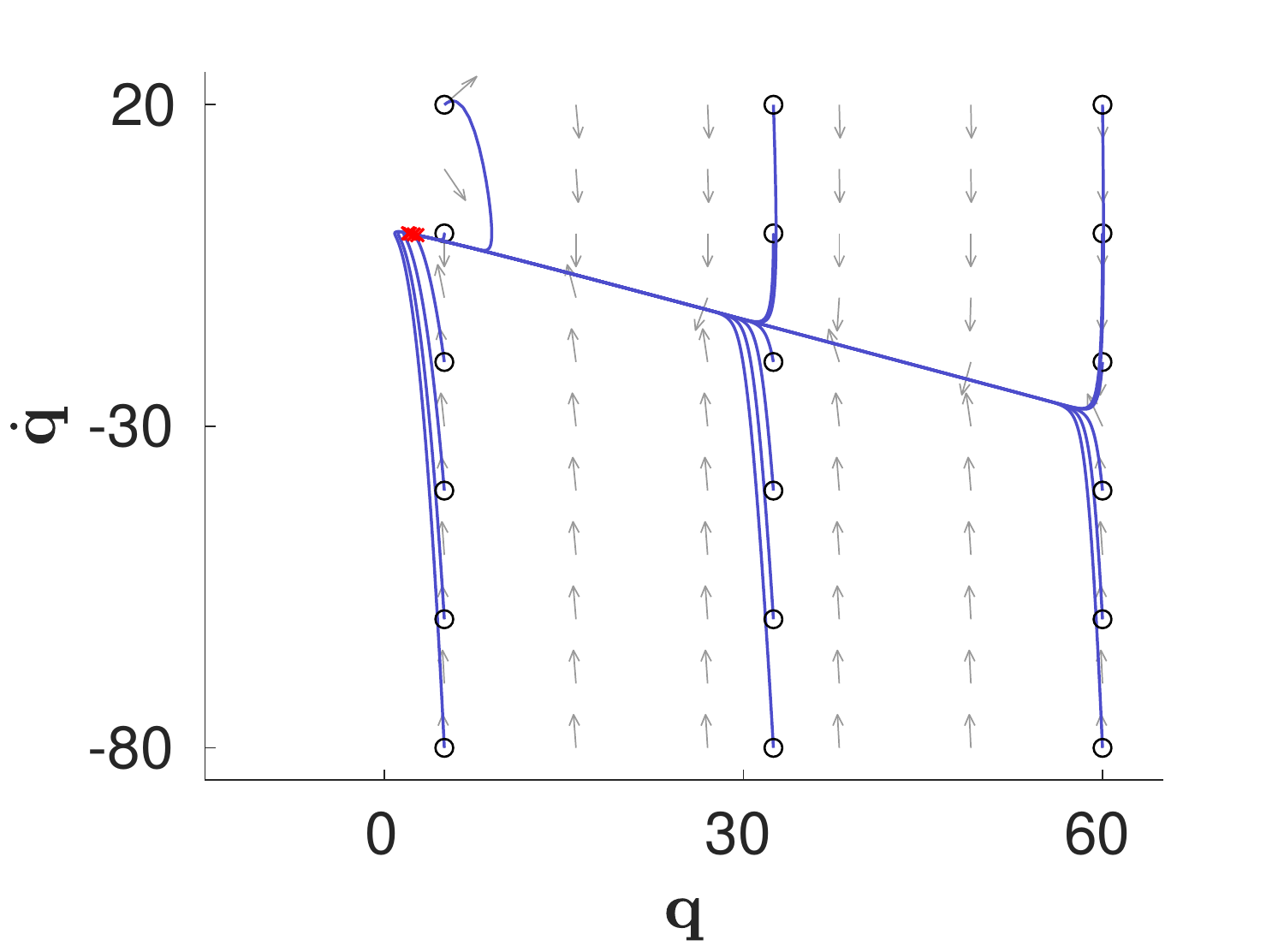}
	}\hspace{-4mm}
	\subfloat[\label{fig:1d_alpha1}]{
		\includegraphics[trim={5 5 0 25},clip, width=0.24\columnwidth,keepaspectratio]{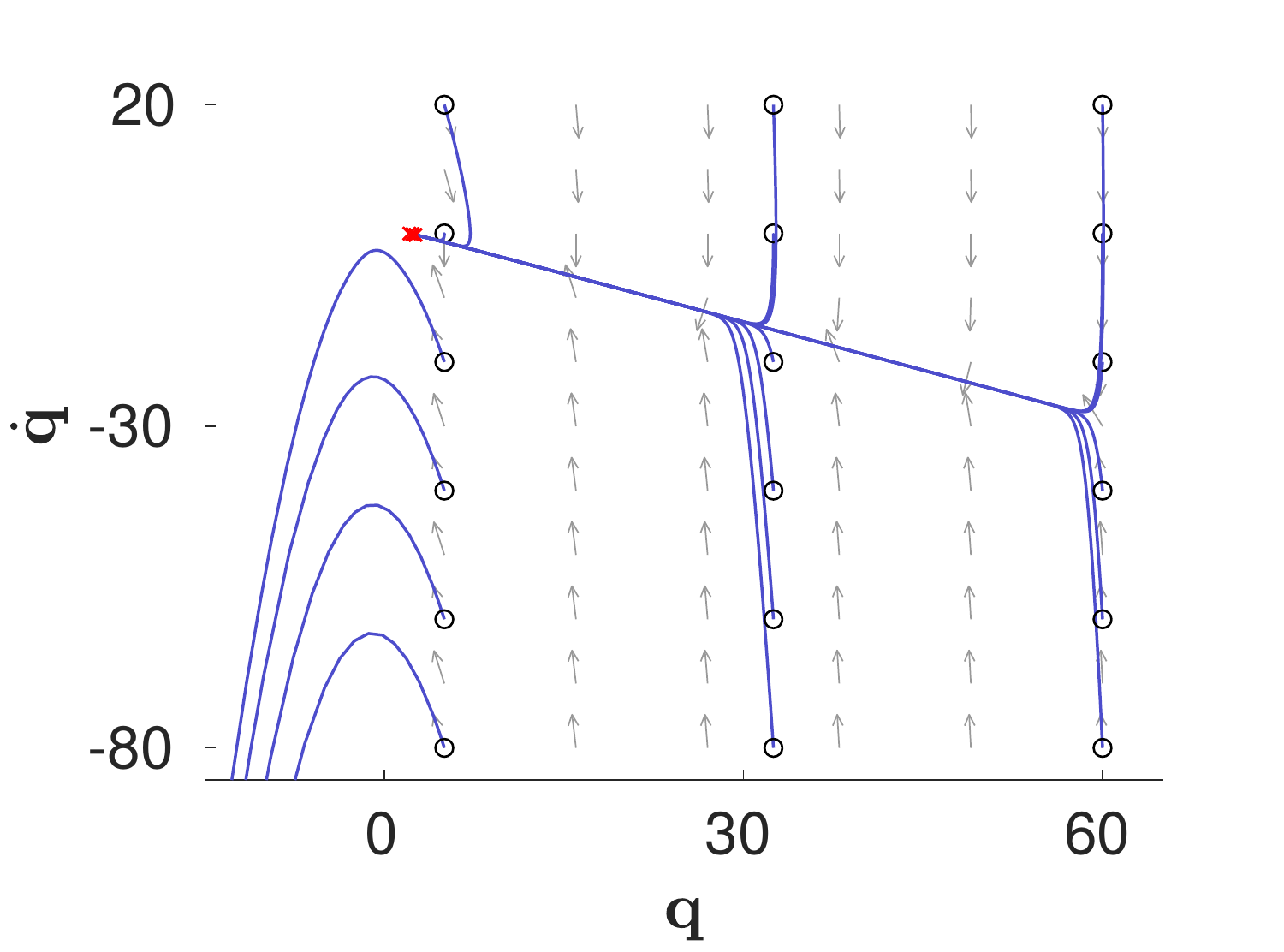}
	}\hspace{-4mm}
	\subfloat[\label{fig:1d_alpha1_damp}]{
		\includegraphics[trim={5 5 0 25},clip, width=0.24\columnwidth,keepaspectratio]{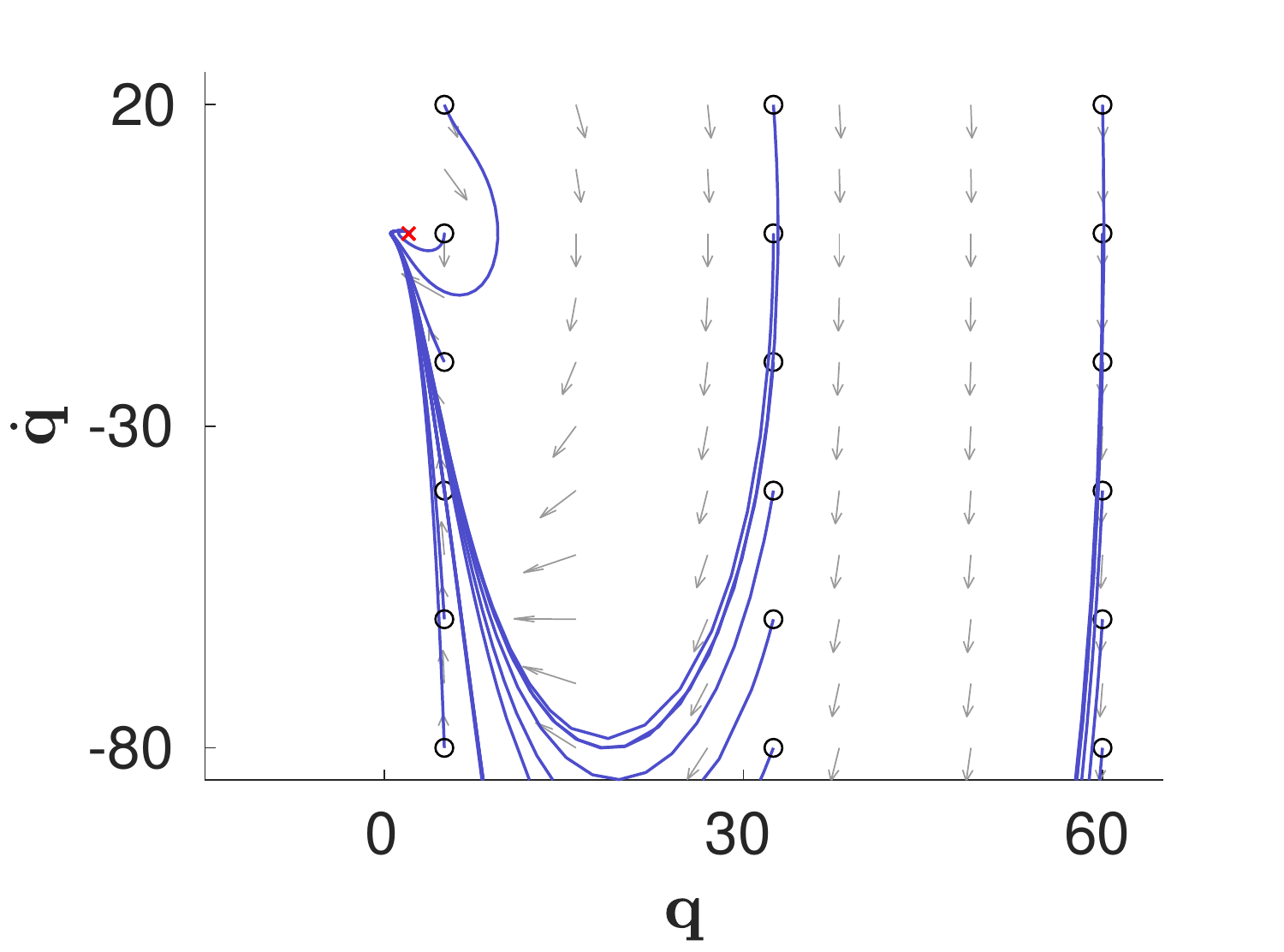}
	}
	\vspace{-3mm}
	\caption{\small Phase portraits (gray) and integral curves (blue; from black circles to red crosses) of 1D example. (a) Desired behavior. 
	(b) With curvature terms. (c) Without curvature terms. (d) Without curvature terms but with nonlinear damping.}
	\label{fig:1d}
	\vspace{-5mm}
\end{figure}

\vspace{-1mm}
\subsubsection{1D Example} 
Let $\q \in \R$. We consider a barrier-type task map $\x = 1/\q$ and define a GDS in~\eqref{eq:GDS} with $\G = 1$, $\Phi(\x) = \frac{1}{2}(\x - \x_0)^2$, and $\B = (1 + 1/\x)$, where $\x_0 > 0$. Using the GDS, we can define an RMP $[- \nabla_\x \Phi - \Bb\xd - \bm\xi_{\G}, \M]^\R$, where
$\M$ and $\bm\xi_{\G}$ are defined according to Section~\ref{sec:GDS}.
We use this example to study the effects of $\dot{\J}\qd$ in \pullback~\eqref{eq:natural pullback}, where we define $\J = \partial_\q \x$. Fig.~\ref{fig:1d} compares the desired behavior (Fig.~\ref{fig:1d_z}) and the behaviors of correct/incorrect \pullback. If \pullback is performed correctly with $\Jd \qd$, the behavior matches the designed one (Fig.~\ref{fig:1d_x}). By contrast, if $\Jd \qd$ is ignored, the observed behavior becomes inconsistent and unstable (Fig.~\ref{fig:1d_alpha1}).
While the instability of neglecting $\dot{\J}\qd$ can be recovered with a damping $\B = (1 + \frac{\xd^2}{\x})$ nonlinear in $\xd$ (suggested in~\cite{lo2016virtual}), the behavior remains inconsistent (Fig.~\ref{fig:1d_alpha1_damp}).

\begin{figure}[t] 
	\centering
	\hspace{-4mm}
	\subfloat[\label{fig:2d_obs_nocorr}]{
		\includegraphics[trim={40 5 130 15},clip, width=0.2\columnwidth,keepaspectratio]{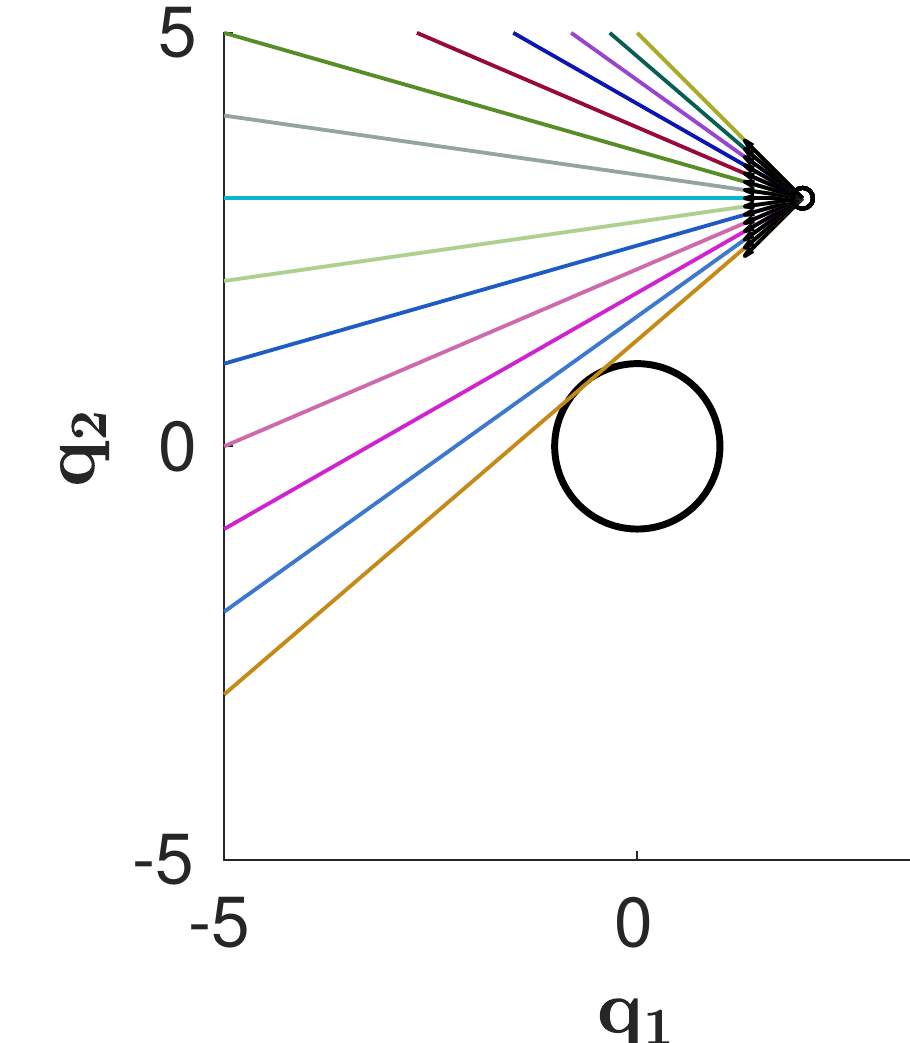}
	}\hspace{-4mm}
	\subfloat[\label{fig:2d_obs}]{
		\includegraphics[trim={40 5 130 15},clip, width=0.2\columnwidth,keepaspectratio]{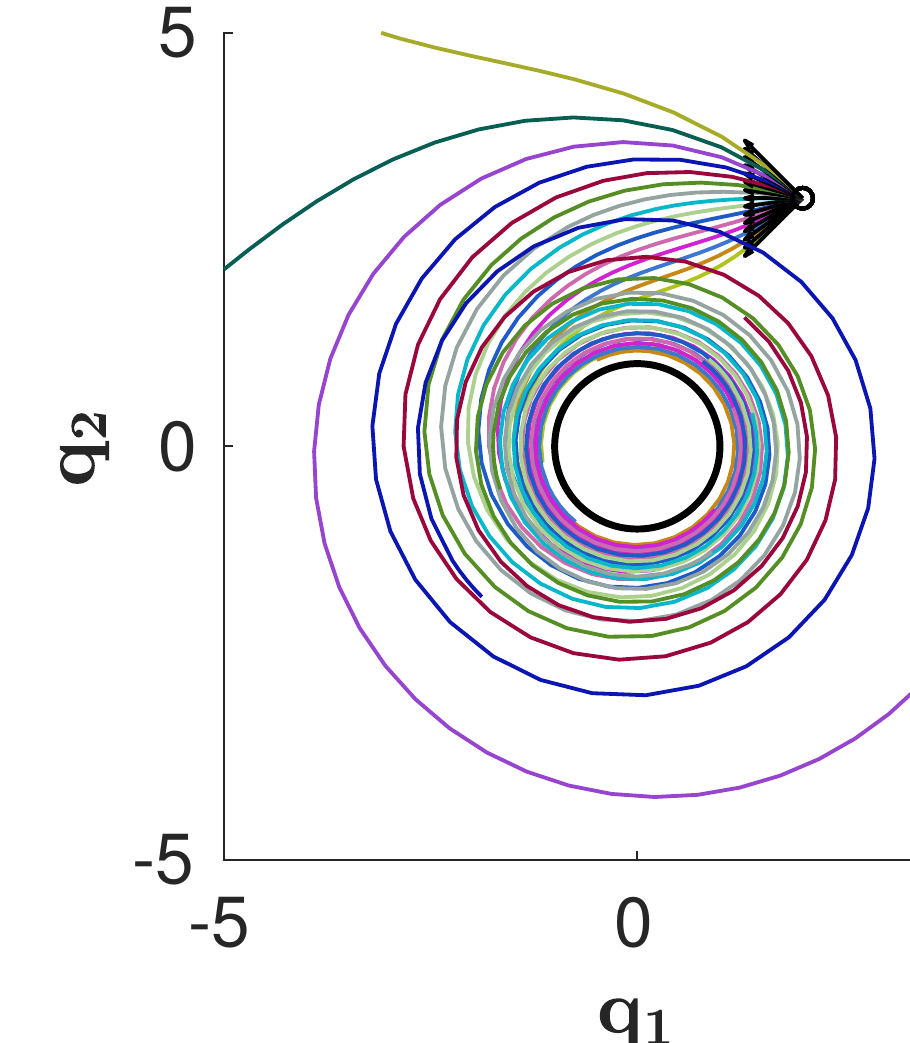}
	}\hspace{-3mm}
	\subfloat[\label{fig:2d_obs_pot_nocorr}]{
		\includegraphics[trim={40 5 130 15},clip, width=0.2\columnwidth,keepaspectratio]{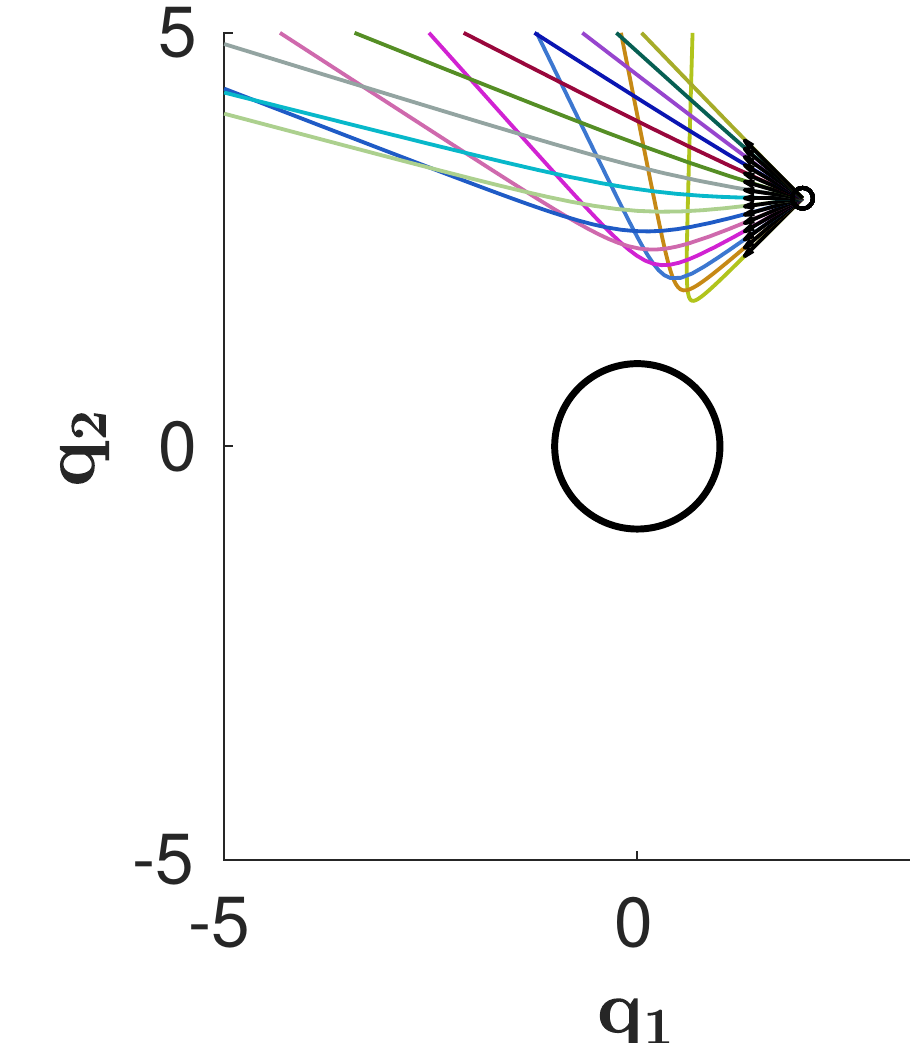}
	}\hspace{-4mm}
	\subfloat[\label{fig:2d_obs_pot}]{
		\includegraphics[trim={40 5 130 15},clip, width=0.2\columnwidth,keepaspectratio]{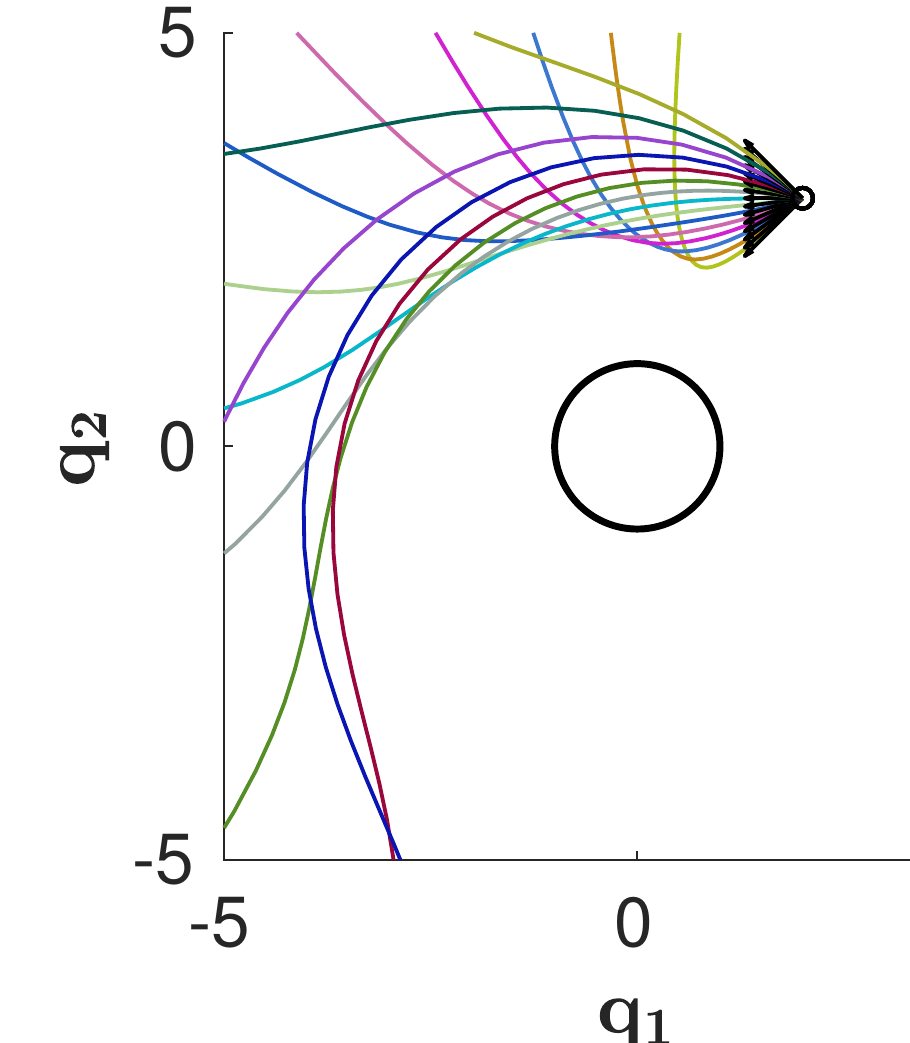}
	}\hspace{-4mm}
	\subfloat[\label{fig:2d_full}]{
		\includegraphics[trim={40 5 130 15},clip, width=0.2\columnwidth,keepaspectratio]{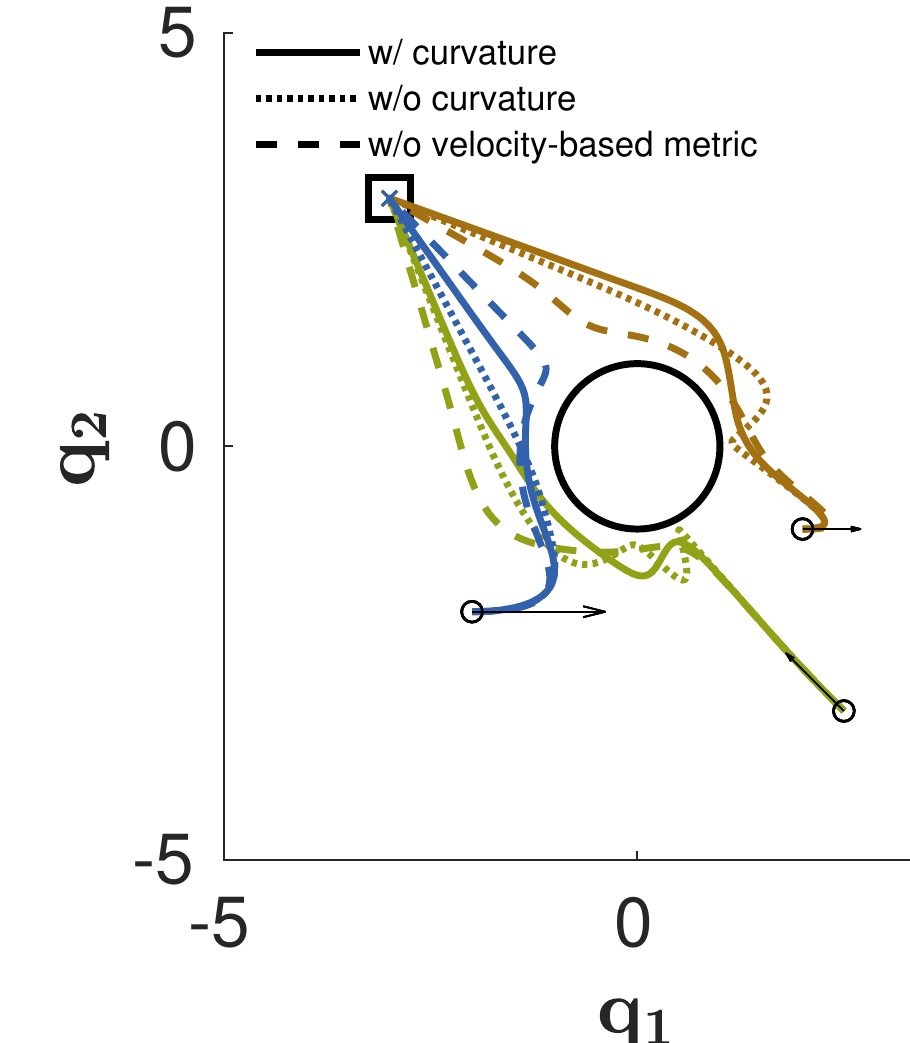}		
	}
	\vspace{-3mm}
	\caption{\small 
		2D example; initial positions (small circle) and velocities (arrows). (a-d) Obstacle (circle) avoidance: (a) w/o curvature terms and w/o potential. (b) w/ curvature terms and w/o potential. (c) w/o curvature terms and w/ potential. (d) w/ curvature terms and w/ potential. (e) Combined obstacle avoidance and goal (square) reaching.}
	\vspace{-3mm}
	\label{fig:2DOrbits}
\end{figure}
\begin{figure}[t]\vspace{-1mm}
	\centering
	\includegraphics[trim={95 0 100 10},clip,width=0.95\columnwidth,keepaspectratio]{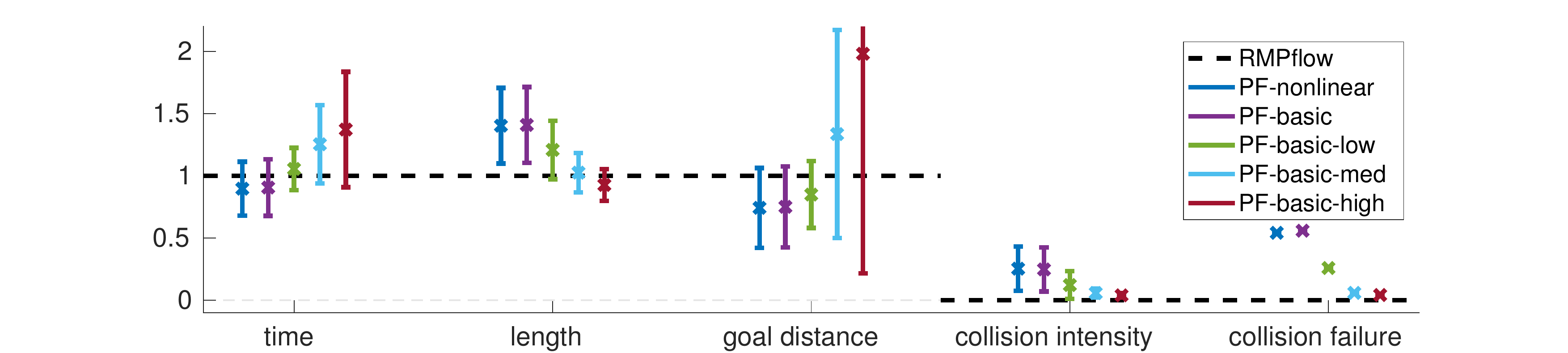}
	\vspace{-3mm}
	\caption{\small Results for reaching experiments. Though some methods achieve a shorter goal distance than \flow in successful trials, they end up in collision in most the trials. }
	\label{fig:reach}
	\vspace{-7mm}
\end{figure}

\vspace{-5mm} 
\subsubsection{2D Example}
We consider a 2D goal-reaching task with collision avoidance and 
study the effects of velocity dependent metrics. 
First, we define an RMP (a GDS as in Section~\ref{sec:example RMPs}) in $\x = d(\q)$ (the 1D task space of the distance to the obstacle). We pick a metric $\G(\x,\xd) = w(\x) u(\xd)$, where $w(\x) = 1/\x^4$ increases if the particle is \emph{close} to the obstacle 
and $u(\xd) = \epsilon + \min(0, \xd) \xd$ (where $\epsilon \geq 0$), increases if it moves \emph{towards} the obstacle. 
As this metric is non-constant, the GDS has curvature terms $\bm\Xi_{\G} = \frac{1}{2}\xd w(\x) \partial_\xd u(\xd)$ and $\bm\xi_{\G} = \frac{1}{2}\xd^2 u(\xd) \partial_\x w(\x)$.  
These curvature terms along with $\Jd \qd$ produce an acceleration that lead to 
natural obstacle avoidance behavior, coaxing the system toward isocontours of the obstacle (Fig.~\ref{fig:2d_obs}).
On the other hand, when the curvature terms are ignored, 
the particle travels in straight lines with constant velocity (Fig.~\ref{fig:2d_obs_nocorr}). 
To define the full collision avoidance RMP, we introduce a barrier-type potential $\Phi(\x) = \frac{1}{2}\alpha w(\x)^2$ to create extra repulsive forces, where $\alpha\geq 0$. A comparison of the curvature effects in this setting is shown in Fig.~\ref{fig:2d_obs_pot_nocorr} and~\ref{fig:2d_obs_pot} (with $\alpha = 1$).
Next, we use \flow to combine the collision avoidance RMP above (with $\alpha=0.001$) and an attractor RMP.
Let $\q_g$ be the goal. The attractor RMP is a GDS in the task space $\y = \q - \q_g$ with a metric $w(\y) \I$, a damping $\eta w(\y) \I$, and a potential that is zero at $\y=0$, where $\eta > 0$ (see Appendix~\ref{apx:Attractors}).  
Fig.~\ref{fig:2d_full} shows the trajectories of the combined RMP. The combined non-constant metrics generate a behavior that transitions smoothly towards the goal while heading away from the obstacle. When the curvature terms are ignored (for both RMPs), the trajectories oscillate 
near the obstacle. In practice, this can result in jittery behavior on manipulators. 
When the metric is not velocity-based ($\G(\x) = w(\x)$) the behavior is less efficient in breaking free from the obstacle to go toward the goal.

\begin{figure}[t]
	\centering	
	\begin{tabular}{cccc}
	\includegraphics[height=0.16\columnwidth]{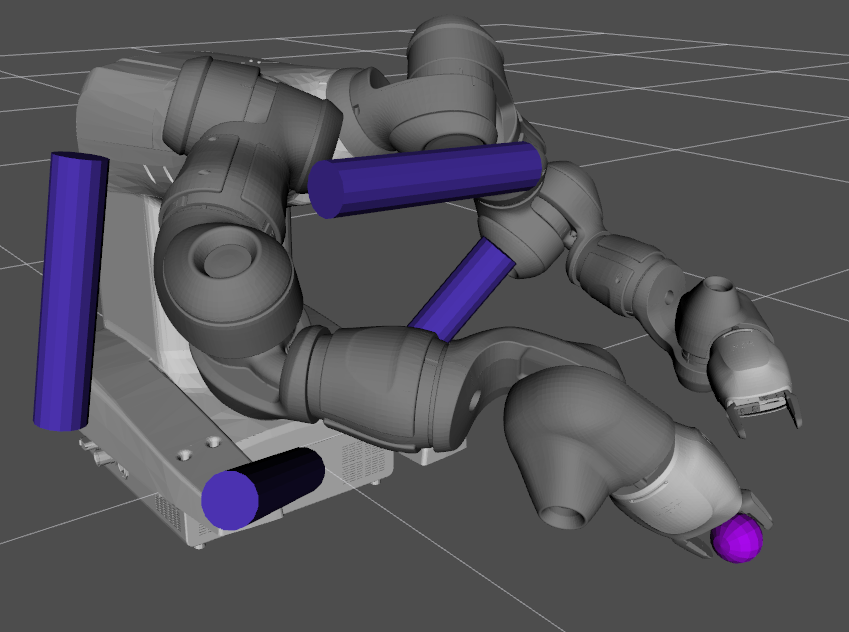} &
	\includegraphics[height=0.16\columnwidth]{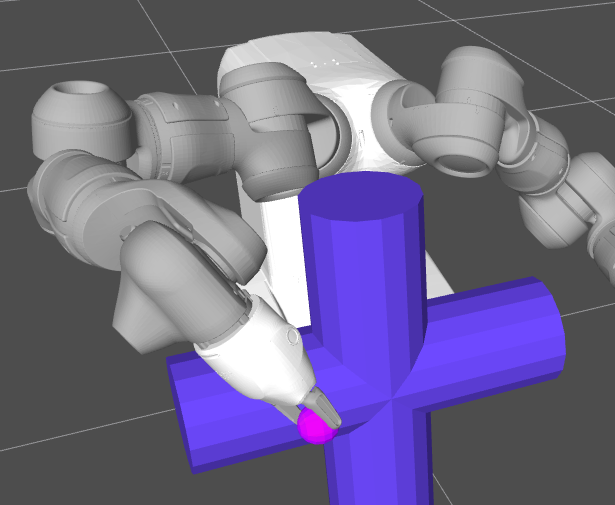} &
	\includegraphics[height=0.16\columnwidth]{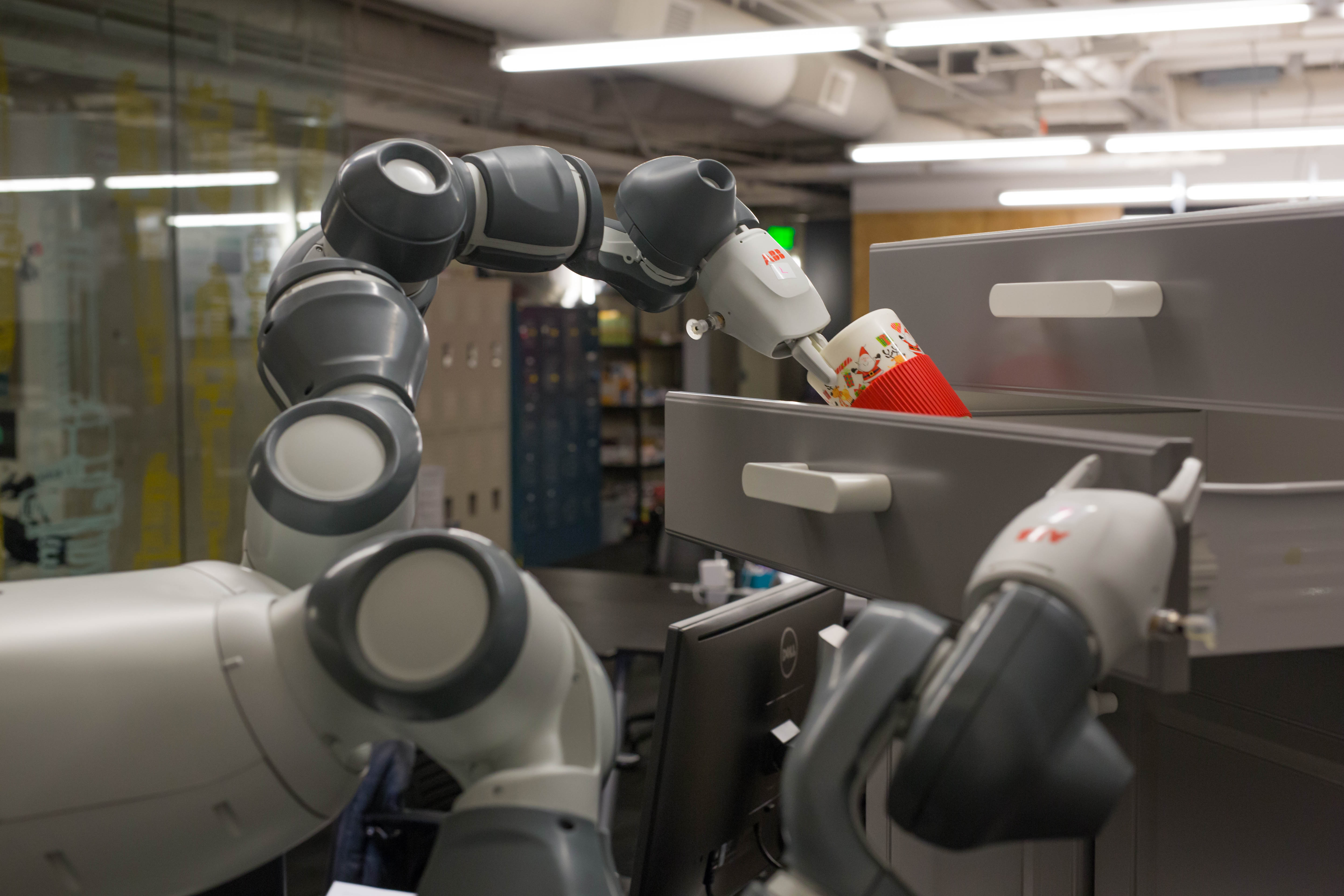} &
	\includegraphics[height=0.16\columnwidth]{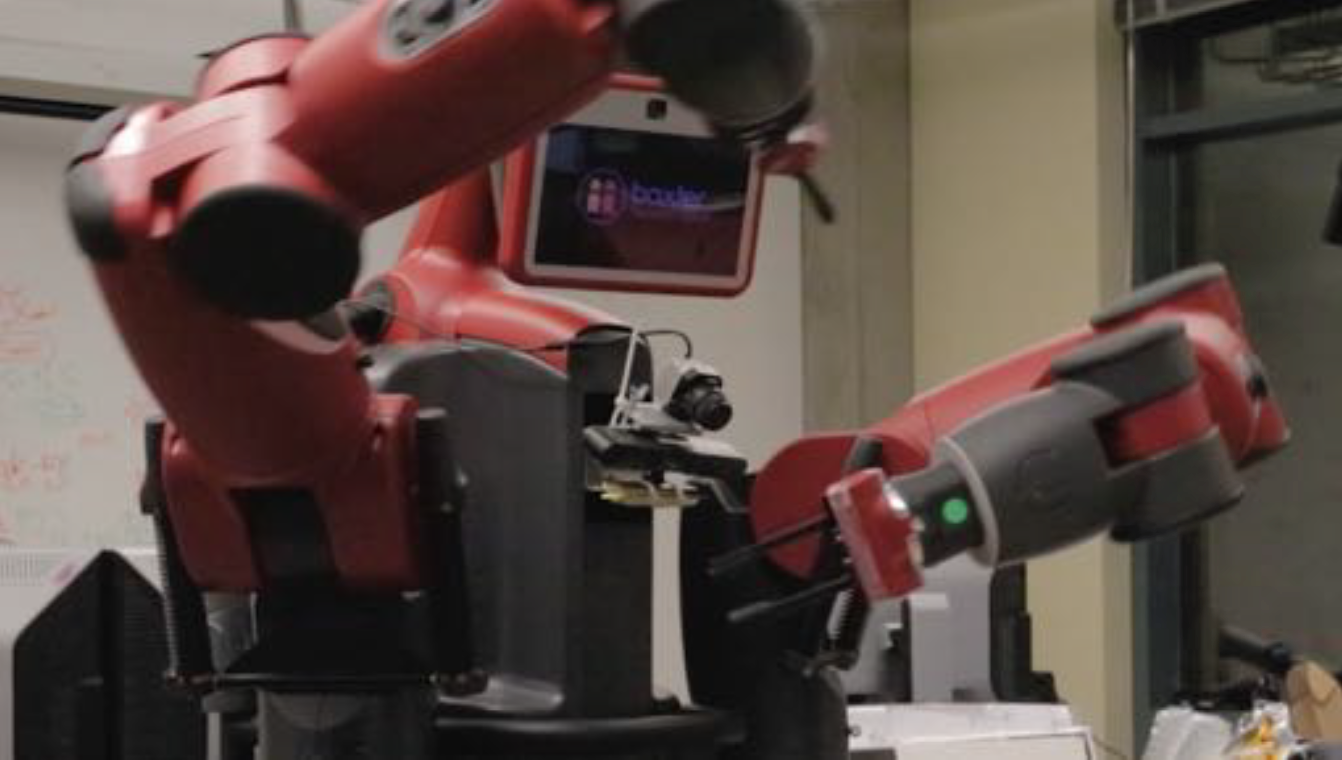} \\
	\multicolumn{2}{c}{simulated worlds} &
	\multicolumn{2}{c}{real-world experiments}
	\end{tabular}
	\caption{\small Two of the six simulated worlds in the reaching experiments (left), and 
    the two physical dual-arm
    platforms in the full system experiment (right). }
	\label{fig:robots}
	\vspace{-6mm}
\end{figure}

\vspace{-4mm}
\subsection{System Experiments}
\vspace{-1mm}
\subsubsection{Reaching-through-clutter Experiments}
We compare \flow with OSC, (i.e. potential fields (PF) with dynamics reshaping),
denoted as PF-basic, 
and a variant, denoted PF-nonlinear, which scales the 
collision-avoidance weights nonlinearly as a function of obstacle proximity.  
We highlight the results here; 
Appendix~\ref{apx:ReachingExperiment} provides additional details, and the supplementary
video shows footage of the trials.
In both baselines, the collision-avoidance task spaces are specified by control points along the robot's body (rather than the distance space used in
\flow) with an isotropic metric $\G = w(\x)\I$
(here $w(\x) = w_o\in\R_+$ for PF-basic and $w(\x)\in[0,w_o]$ for PF-nonlinear, where
$w_o$ is the max metric size used in \flow). 
The task-space policies of both variants follow GDSs, but without the curvature terms 
(see Appendix~\ref{apx:ReachingExperiment}).

Fig.~\ref{fig:reach} summarizes their performance. We measure time-to-goal,
C-space path length (assessing economy of motion), achievable distance-to-goal
(efficacy in solving the problem), collision intensity (percent time in
collision \emph{given} a collision), collision failures (percent trials with
collisions).  The isotropic metrics, across multiple settings, fail to match
the speed and precision achieved by \flow. Higher-weight settings tend to have
fewer collisions and better economy of motion, but at the expense of
efficiency. Additionally, adding nonlinear weights as in PF-nonlinear does not
seem to help.  The decisive factor of \flow's performance is rather its
non-isotropic metric, which 
encodes directional importance around obstacles in combing policies.

\vspace{-5mm}
\subsubsection{System Integration for Real-Time Reactive Motion Generation}
We present an integrated system for vision-driven dual arm manipulation on two robotic platforms, the ABB YuMi robot and the Rethink Baxter robot (Fig.~\ref{fig:robots}) (see the supplementary video).
Our system uses the real-time optimization-based tracking algorithm DART~\cite{Sch15DAR} to communicate with the RMP system,
receiving prior information on robot configuration and sending tracking updates of world state.  
The system is tested in multiple real-world manipulation problems, like picking up
trash in clutter, reactive manipulation of a cabinet with human perturbation,
active lead-through (compliant guiding of the arms with world-aware collision controllers) 
and pick-and-place of objects into a drawer which the robot opens and closes.
Please see Appendix~\ref{apx:IntegratedSystem} for the details of the experiments.
%
\vspace{-4mm}
\section{Conclusion}
\vspace{-3mm}
We propose an efficient policy synthesis framework, \flow, for generating policies with non-Euclidean behavior, including motion with velocity dependent metrics that are new to the literature.
In design, \flow is implemented as a computational graph, which can geometrically consistently combine subtask policies into a global policy for the robot. 
In theory, we provide conditions for stability and show that \flow is intrinsically coordinate-free. 
In the experiments, we demonstrate that \flow can generate smooth and natural motion for various tasks, when proper subtask RMPs are specified. Future work is to further relax the requirement on the quality of designing subtask RMPs by introducing learning components into \flow for additional flexibility.

\vspace{-4mm}

\bibliographystyle{splncs03_unsrt}
\bibliography{refs}
\newpage

\input{rmp-wafr-2018-appendix}

\end{document}

%% file: shortcuts.tex
\usepackage{amsmath}
\usepackage{amsfonts}
\usepackage{amssymb}
\usepackage{amsthm}
\usepackage{bm}
\usepackage{bbm}
\usepackage{mathtools}
\usepackage{enumitem}
\usepackage{thmtools,thm-restate}
\usepackage{algorithm}
\usepackage{algorithmic}
\usepackage{color}
\usepackage{graphicx}
\usepackage{comment}

\def\CC{\mathcal{C}}

\def\MM{\mathcal{M}}\def\NN{\mathcal{N}}

\def\SS{\mathcal{S}}\def\TT{\mathcal{T}}

\def\Ab{\mathbf{A}}\def\Bb{\mathbf{B}}\def\Cb{\mathbf{C}}
\def\Db{\mathbf{D}}
\def\Gb{\mathbf{G}}\def\Hb{\mathbf{H}}
\def\Jb{\mathbf{J}}\def\Lb{\mathbf{L}}
\def\Mb{\mathbf{M}}\def\Nb{\mathbf{N}}
\def\Rb{\mathbf{R}}

\def\ab{\mathbf{a}}
\def\eb{\mathbf{e}}\def\fb{\mathbf{f}}
\def\gb{\mathbf{g}}
\def\jb{\mathbf{j}}\def\lb{\mathbf{l}}
\def\mbb{\mathbf{m}}
\def\nb{\mathbf{n}}
\def\qb{\mathbf{q}}
\def\ub{\mathbf{u}}
\def\vb{\mathbf{v}}\def\xb{\mathbf{x}}

\def\Rbb{\mathbb{R}}

\def\R{\Rbb}

\def\diag{\mathrm{diag}}

\def\t{\top}
\def\*{\star}

\newcommand{\norm}[1]{ \| #1 \|  }

\newcommand{\lr}[2]{ \left\langle #1, #2 \right\rangle}
\DeclareMathOperator*{\argmin}{arg\,min}

%% file: commands.tex
\usepackage[dvipsnames]{xcolor}

\newcommand{\half}{{\scriptstyle \frac{1}{2}}}

\newcommand{\trihat}[1]{{\stackrel{\scriptscriptstyle\Delta}{#1}}}

\newcommand{\ma}{\mathbf{a}}

\newcommand{\q}{\mathbf{q}}
\newcommand{\qd}{{\dot{\q}}}
\newcommand{\qdd}{{\ddot{\q}}}

\newcommand{\vv}{\mathbf{v}}

\newcommand{\x}{\mathbf{x}}
\newcommand{\xd}{{\dot{\x}}}
\newcommand{\xdd}{{\ddot{\x}}}
\newcommand{\y}{\mathbf{y}}
\newcommand{\yd}{{\dot{\y}}}
\newcommand{\ydd}{{\ddot{\y}}}
\newcommand{\z}{\mathbf{z}}
\newcommand{\zd}{{\dot{\z}}}

\newcommand{\f}{\mathbf{f}}

\newcommand{\J}{\mathbf{J}}
\newcommand{\Jd}{{\dot{\J}}}

\newcommand{\A}{\mathbf{A}}
\newcommand{\B}{\mathbf{B}}
\newcommand{\C}{\mathbf{C}}
\newcommand{\D}{\mathbf{D}}

\newcommand{\G}{\mathbf{G}}

\newcommand{\I}{\mathbf{I}}

\newcommand{\M}{\mathbf{M}}

\newcommand{\mT}{\mathbf{T}}
\newcommand{\V}{\mathbf{V}}

\newcommand{\wt}[1]{{\widetilde{#1}}}

%% file: rmp-wafr-2018-appendix.tex
\appendix
\section*{Appendices}

\section{Geometric Dynamical Systems} \label{app:GDSs}

Here we summarize details and properties of GDSs introduced in Section~\ref{sec:GDS}.

\subsection{From Geometric Mechanics to GDSs} \label{apx:WhyGeometricMechanics}
\ifAPP

Our study of GDSs is motivated by geometric mechanics. 
Many formulations of mechanics exist, including
Lagrangian mechanics~\cite{ClassicalMechanicsTaylor05} and the aforementioned Gauss's Principle of Least Constraint~\cite{udwadia1996analytical}-----They are all equivalent, implicitly sharing the same mathematical structure.  
In that sense, geometric mechanics, which models
physical systems as geodesic flow on Riemannian manifolds, is the most
explicit of these, revealing directly the underlying manifold structure and connecting to the broad mathematical tool set from Riemannian geometry.
These connections enable us here to generalize beyond the previous simple mechanical systems studied in~\cite{bullo2004geometric} to non-classical systems that more naturally describe robotic behaviors with non-Euclidean geometric properties.
\fi

\subsection{Degenerate GDSs}


Let us recall the definition of GDSs.
\begin{definition} \label{def:general GDS}
	Let $\B: \R^m \times \R^m \to \R^{m\times m}_{+}$ and let $\Gb: \R^m \times \R^m \to \R^{m\times m}_{+}$ and $\Phi: \R^m \to \R$ be differentiable.
	We say the tuple $(\MM, \Gb, \B, \Phi)$ is a \emph{GDS} if 
	\begin{align} \label{eq:GDS general}
		\Mb(\x,\xd) \xdd + \bm\xi_{\G}(\x,\xd)  = - \nabla_\x \Phi(\x) - \Bb(\x,\xd)\xd 
	\end{align}
	where  $\Mb(\x,\xd) = \Gb(\x,\xd) + \bm\Xi_{\G}(\x,\xd)$.
\end{definition}

\noindent  For degenerate cases, $\Mb(\x,\xd)$ can be singular and~\eqref{eq:GDS general} define rather a family of differential equations. Degenerate cases are not uncommon; for example, the leaf-node dynamics could have $\G$ being only positive semidefinite. 
Having degenerate GDSs does not change the properties that we have proved, but one must be careful about whether differential equation satisfying~\eqref{eq:GDS general} exist.
For example, the existence is handled by the assumption on $\Mb$ in Theorem~\ref{th:consistency} and the assumption on $\Mb_r$ in Corollary~\ref{cr:consistency}. For \flow, we only need that $\Mb_r$ at the root node is non-singular. In other words, the natural-form RMP created by \pullback at the root node can be resolved in the canonical-form RMP for policy execution.
A sufficient and yet practical condition is provided in Theorem~\ref{th:condition on velocity metric}. 


\subsection{Geodesic and Stability}
For GDSs, they possess a natural conservation property of kinematic energy, i.e. it travels along a geodesic defined by $\G(\x,\xd)$ when there is no external perturbations due to $\Phi$ and $\Bb$. Note $\G(\x,\xd)$ by definition may only be positive-semidefinite even when the system is non-degenerate; here we allow the geodesic to be defined for a degenerate metric, meaning a curve whose instant length measured by the (degenerate) metric is constant.

This geometric feature is an important tool to establish the stability of non-degenerate GDSs; 
We highlight this nice geometric property below, which is a corollary of Proposition~\ref{pr:Lyapunov time derivative}. 
\begin{corollary}
All non-degenerate GDSs in the form $(\MM, \Gb, 0, 0)$ travel on geodesics. That is, $\dot{K}(\x,\xd) = 0$, where $K(\x,\xd) = \frac{1}{2} \xd^\t \G(\x,\xd) \xd$. 
\end{corollary}
Note that this property also hold for degenerate GDSs provided that differential equations satisfying~\eqref{eq:GDS general} exist.


\subsection{Curvature Term and Coriolis Force}

The curvature term $\bm{\xi}_\G$ in GDSs is highly related to the Coriolis force in the mechanics literature. This is not surprising, as from the analysis in Section~\ref{sec:geometric properties} we know that $\bm{\xi}_\G$  comes from the Christoffel symbols of the asymmetric connection.  Recall it is defined as 
\begin{align*} 
\bm\xi_{\G}(\x,\xd) &\coloneqq \textstyle \sdot{\Gb}{\x}(\x,\xd) \xd - \frac{1}{2} \nabla_\x (\xd^\t \Gb(\x,\xd) \xd)
\end{align*} 
We show their relationship explicitly below.
\begin{restatable}{lemma}{coriolosIdentity}\label{lm:Coriolos identity}
 \label{lm:compact writing of Coriolis force}
	Let $\Gamma_{ijk} = \frac{1}{2}( \partial_{x_k}  G_{ij}  + \partial_{x_j}  G_{ik} - \partial_{x_j}  G_{jk})$ be the Christoffel symbol of the first kind with respect to $\Gb(\x,\xd)$, where the subscript $_{ij}$ denotes the $(i,j)$ element. Let 
	$	C_{ij}
	= 
	\sum_{k=1}^{d} \dot{x}_k \Gamma_{ijk}
	$ and define $\Cb(\x,\xd) = (C_{ij})_{i,j=1}^m$. Then $\bm\xi_{\Gb}(\x,\xd)  = \Cb(\x,\xd) \xd$. 
\end{restatable}
\begin{proof}[Proof of Lemma~\ref{lm:Coriolos identity}]
	Suppose $\bm\xi_\Gb = (\xi_i)_{i=1}^m$. We can compare the two definitions and verify they are indeed equivalent: 
	\begin{align*}
	\xi_i
	&= 
	\sum_{j,k=1}^{d} \dot{x}_j \dot{x}_k \partial_{x_j}  G_{ik} - \frac{1}{2}\sum_{j,k=1}^{d} \dot{x}_j\dot{x}_k \partial_{x_i}  G_{jk}\\
	&= 
	\frac{1}{2} \sum_{j,k=1}^{d} \dot{x}_j \dot{x}_k \partial_{x_k}  G_{ij} 
	+\frac{1}{2}\sum_{j,k=1}^{d} \dot{x}_j \dot{x}_k \partial_{x_j}  G_{ik} 
	- \frac{1}{2}\sum_{j,k=1}^{d} \dot{x}_j\dot{x}_k \partial_{x_i}  G_{jk}  \\
	&= \left( \Cb(\x,\xd) \xd \right)_i \qedhere
	\end{align*}
\end{proof}

\section{Proofs of \flow Analysis} \label{app:proof of analysis}
\subsection{Proof of Theorem~\ref{th:consistency}} \label{app:proof of consistency}
\theoremConsistency*
\begin{proof}[Proof of Theorem~\ref{th:consistency}]
	We will use the non-degeneracy assumption that $\Gb + \bm\Xi_{\Gb}$ (i.e. $\M$ as we will show) is non-singular, so that the differential equation specified by an RMP in normal form or a (structured) GDS is unique. This assumption is made to simplify writing. At the end of the proof, we will show that this assumption only needs to be true  at the root node of \flow.
	
	\noindent\textbf{The general case  }
	We first show the differential equation given by \pullback is equivalent to the differential equation of pullback structured GDS $(\MM, \G, \B, \Phi)_{\SS}$.
	Under the non-degeneracy assumption, suppose $\SS_i$ factorizes $\G_i$ as $
	\G_i = \Lb_i^\t \Hb_i \Lb_i
	$, where $\Lb_i$ is some Jacobian matrix. On one hand, for $\pullback$, because in the child  node $\ydd_i$ satisfies
	$(\G_i + \bm\Xi_{\G_i})\ydd_i = -\bm\eta_{\G_i;\SS_i} - \nabla_{\y_i} \Phi_i - \B_{i}\yd_i$ (where by definition 
	$
	\bm\eta_{\G_i;\SS_i}
	= \Lb_i^\t ( \bm\xi_{\Hb_i} + 
	(\Hb_i + \bm\Xi_{\Hb_i} ) 
	\dot{\Lb}_i \yd_i  )
	$),  the \pullback operator combines the child nodes into the differential equation at the parent node,
	\begin{align} \label{eq:diff eq of pullback operator}
		\Mb\xdd =  \sum_{i=1}^{K} \J_i^\t \M_i ( \ydd_i - \dot\J_i \xd  ) 
	\end{align}	
	where we recall $\Mb = \sum_{i=1}^{K} \J_i^\t \M_i \J_i$ is given by $\pullback$.		
	On the other hand, for $(\MM, \G, \B, \Phi)_{\SS}$ with $\SS$ preserving $\SS_i$, its dynamics satisfy
	\begin{align} \label{eq:diff eq of pullback structure GDS}
		\left(\Gb + \bm\Xi_{\Gb}\right) \xdd 
		+ \bm\eta_{\Gb;\SS}  = -\nabla_\x \Phi -\Bb\xd
	\end{align}
	where
	$\G$ is factorized by $\SS$ into
	\begin{align*}
		\G &= 
		\begin{bmatrix}
			\J_1 \\ \vdots \\ \J_K
		\end{bmatrix}^\t
		\begin{bmatrix}
			\G_1 & & \\
			& \ddots &\\
			& & \G_K
		\end{bmatrix}
		\begin{bmatrix}
			\J_1 \\ \vdots \\ \J_K
		\end{bmatrix}= 
		\begin{bmatrix}
			\Lb_1\J_1 \\ \vdots \\ \Lb_K\J_K
		\end{bmatrix}^\t
		\begin{bmatrix}
			\Hb_1 & & \\
			& \ddots &\\
			& & \Hb_K
		\end{bmatrix}
		\begin{bmatrix}
			\Lb_1\J_1 \\ \vdots \\ \Lb_K\J_K
		\end{bmatrix}
		\eqqcolon \bar{\J}^\t \bar{\Hb} \bar{\J}
	\end{align*}	
	and the curvature term $\bm\eta_{\G;\SS}$ by $\SS$ is given as 
	$
	\bm\eta_{\G;\SS} 
	\coloneqq  \bar\J^\t ( \bm\xi_{\bar\Hb} + 
	(\bar\Hb + \bm\Xi_{\bar\Hb} ) 
	\dot{\bar\J} \xd  )
	$.

	To prove the general statement, we will show~\eqref{eq:diff eq of pullback operator} and~\eqref{eq:diff eq of pullback structure GDS} are equivalent. 	
	First, we introduce a lemma to write $\bm\Xi_\G$ in terms of $\bm\Xi_{\G_i}$ (proved later in this section).
	\begin{lemma} \label{lm:pullback of Xi}
		Let $\MM$ and $\NN$ be two manifolds and let $\x$ and $\y(\x)$ be the coordinates. Define $\M(\x,\xd) = \J(\x)^\t \Nb(\y,\yd) \J(\x)$, where $\J(\x) = \partial_\x \y(\x)$. Then 
		\begin{align*}
			\bm\Xi_{\M}(\x,\xd) = \J^\t(\x) \bm\Xi_{\Nb} (\y,\yd) \J(\x)
		\end{align*}		
	\end{lemma}
	\noindent Therefore, we see that on the LHSs
	\begin{align*}
		(\G + \bm\Xi_{\G})  \xdd = \M \xdd
	\end{align*}
	and on the RHSs	
	\begin{align*} 		
		& \left( \sum_{i=1}^{K} \J_i^\t \M_i ( \ydd_i - \dot\J_i \xd  ) \right)\\
		&=\left( \sum_{i=1}^{K} \J_i^\t (-\bm\eta_{\G_i;\SS_i} - \nabla_{\y_i} \Phi_i - \B_{i}\yd_i  - (\G_i + \bm\Xi_{\G_i})\dot\J_i \xd  ) \right) \\
		&= \left( \sum_{i=1}^{K} \J_i^\t (- \Lb_i^\t ( \bm\xi_{\Hb_i} + 
		(\Hb_i + \bm\Xi_{\Hb_i} ) 
		\dot{\Lb}_i \yd_i  )  - (\G_i + \bm\Xi_{\G_i})\dot\J_i \xd)   \right) + \left( \sum_{i=1}^{K} \J_i^\t (- \nabla_{\y_i} \Phi_i - \B_{i}\yd_i )  \right)\\
		&= \left( \sum_{i=1}^{K}   -\bar\Jb_i^\t \bm\xi_{\Hb_i} 
		-\bar\Jb_i^\t(\Hb_i + \bm\Xi_{\Hb_i})	(\dot{\Lb}_i \J_i + \Lb_i\dot\J_i  )\xd  \right)  -\nabla_\x \Phi -\Bb\xd \\
		&= - \bm\eta_{\Gb;\SS} -\nabla_\x \Phi -\Bb\xd
	\end{align*}
	where the first equality is due to Lemma~\ref{lm:pullback of Xi}, the second equality is due to~\eqref{eq:diff eq of pullback operator}, and the third equality is due to the definition of structured GDSs.	
	The above derivations show the equivalence between the RHSs and LHSs of~\eqref{eq:diff eq of pullback operator} and~\eqref{eq:diff eq of pullback structure GDS}, respectively. Therefore, when the non-degenerate assumption holds, \eqref{eq:diff eq of pullback operator} and~\eqref{eq:diff eq of pullback structure GDS} are equivalent.

	\noindent\textbf{The special case  }
	With the closure of structured GDSs proved, we next show the closure of GDSs under \pullback, when the metric is only configuration-dependent. That is, we want to show that, when the metric is only configuration-dependent, the choice of structure does not matter.	
	This amounts to show that $\bm\xi_{\G} = \bm\eta_{\G;\SS}$ because by definition $\bm\Xi_i=0$ and $\bm\Xi=0$. 
	Below we show how $\bm\xi_{\G}$ is written in terms of $\bm\xi_{\G_i}$ and $\bm\Xi_{\G_{i}}$ for general metric matrices and specialize it to the configuration-dependent special case (proved later in this section).
	\begin{lemma} \label{lm:pullback of xi}
		Let $\MM$ and $\NN$ be two manifolds and $\x$ and $\y(\x)$ be the coordinates. Suppose $\M(\x,\xd)$ is structured as $\J(\x)^\t \Nb(\y,\dot\y) \J(\x)$, where $\J(\x) = \partial_\x \y(\x)$. Then 
		\begin{align*}
			\bm\xi_{\Mb}(\x,\xd) 
			&=  \J(\x)^\t \left( \bm\xi_{\Nb}(\y,\yd) + 
			(\Nb(\y,\dot\y) + 2\bm\Xi_{\Nb}(\y,\yd) ) 
			\dot\J(\x,\xd) \xd  \right)  \\
			&\quad - \dot\J(\x,\xd)^\t \bm\Xi_{\Nb}(\y,\yd)^\t \J(\x) \xd 
		\end{align*}
		When $\M(\x,\xd) = \M(\x)$, $\bm\xi_{\Mb} = \bm\eta_{\Mb;\SS}$ regardless of the structure of $\SS$. 
	\end{lemma}
	\noindent By Lemma~\ref{lm:pullback of xi}, we see that  structured GDSs are GDSs regardless of the chosen structure when the metric is only configuration dependent. Thus, the statement of the special case follows by combining Lemma~\ref{lm:pullback of xi} and the previous proof for structured GDSs.

	\noindent\textbf{Remarks: Proof of Corollary~\ref{cr:consistency}  }
	We note that the non-degenerate assumption does not need to hold for every nodes in \flow but only for the root node. This can be seen from the proof above, where we propagate the LHSs and RHSs \emph{separately}. Therefore, as long as the inertial matrix at the root node is invertible, the differential equation on the configuration space is well defined.
\end{proof}

\begin{proof}[Proof of Lemma~\ref{lm:pullback of Xi}] Let $\mbb_i$, $\nb_i$, $\jb_i$ be the $i$th column of $\M$, $\Nb$, and $\J$, respectively. Suppose $\MM$ and $\NN$ are of $m$ and $n$ dimensions, respectively.
	By definition of $\bm\Xi_{\M}$, 
	\begin{align*} 
		2\bm\Xi_{\M}(\x,\xd)=
		\sum_{i=1}^m \dot x_i \partial_{\xd}\mbb_i (\x,\xd)
		&=
		\J(\x)^\t  \sum_{i=1}^m \dot x_i \partial_{\xd} (\Nb(\y,\yd) \jb_i(\x)) \\
		&=
		\J(\x)^\t  \left( \sum_{i=1}^m \dot x_i \partial_{\yd} (\Nb(\y,\yd) \jb_i(\x)) \right) \J(\x)\\	
		&=
		\J(\x)^\t  \left( \sum_{j=1}^n \partial_{\yd}\nb_j(\y,\yd)  \sum_{i=1}^m \dot x_i J_{ji}(\x) \right) \J(\x)\\
		&= 
		\J(\x)^\t  \left( \sum_{j=1}^n y_j \partial_{\yd}\nb_j (\y,\yd)  \right) \J(\x)\\
		&= 2 \J(\x)^\t   \bm\Xi_{\Nb} (\y,\yd) \J(\x) \qedhere
	\end{align*}
\end{proof}

\begin{proof}[Proof of Lemma~\ref{lm:pullback of xi}]
	Before the proof, we first note a useful identity 
	$
	\partial_\x \yd 
	= \dot{\J}(\x,\xd)
	$.
	This can be derived simply by the definition of the Jacobian matrix $
	(\partial_\x \J(\x) \xd)_{ij} = \sum_{k=1}^{m} \dot{x}_k  \partial_{x_j} J_{ik}   = \sum_{k=1}^{m} \dot{x}_k  \partial_{x_j} \partial_{x_k} y_i 
	= \sum_{k=1}^{m} \dot{x}_k  \partial_{x_k} J_{ij} = (\dot{\J})_{ij}  $. 
	
	To prove the lemma, we derive $\bm\xi_{\Mb}$ by its definition 
	\begin{align*}
		\bm\xi_{\Mb}  &= \sdot{\Mb}{\x}(\x,\xd) \xd - \frac{1}{2} \nabla_\x (\xd^\t \Mb(\x,\xd) \xd) \\
		&= \dot\J(\x,\xd)^\t \Nb(\y,\yd) \J(\x) \xd  + \J(\x)^\t \Nb(\y,\yd) \dot\J(\x,\xd) \xd + \J(\x)^\t \sdot{\Nb}{\x}(\y,\yd) \J(\x) \xd 
		- \frac{1}{2} \nabla_\x (\xd^\t \J(\x)^\t \Nb(\y,\yd) \J(\x) \xd) \\
		&= \dot\J(\x,\xd)^\t \Nb(\y,\yd) \yd  + \J(\x)^\t \Nb(\y,\yd) \dot\J(\x,\xd) \xd + \J(\x)^\t \sdot{\Nb}{\x}(\y,\yd) \yd 
		- \frac{1}{2} \nabla_\x (\yd^\t \Nb(\y,\yd) \yd) \\
		&= \dot\J(\x,\xd)^\t \Nb(\y,\yd) \yd  + \J(\x)^\t \Nb(\y,\yd) \dot\J(\x,\xd) \xd + \J(\x)^\t \sdot{\Nb}{\x}(\y,\yd) \yd \\
		&\quad
		- \frac{1}{2}\J(\x)^\t \nabla_\y (\yd^\t  \Nb(\y,\yd) \yd) -  \dot\J(\x,\xd)^\t \Nb(\y,\yd) \yd -  \dot\J(\x,\xd)^\t \bm\Xi_{\Nb} (\y,\yd)^\t \J (\x)\xd\\	
		&=  \J(\x)^\t (\Nb(\y,\yd) \dot\J(\x,\xd) \xd + \sdot{\Nb}{\x}(\y,\yd) \J(\x) \xd - \frac{1}{2} \nabla_\y (\yd^\t  \Nb(\y,\yd) \yd) ) 
		-  \dot\J(\x,\xd)^\t \bm\Xi_{\Nb} (\y,\yd)^\t \J (\x)\xd
	\end{align*}
	In the second to the last equality above, we use $\partial_\x \yd = \dot{\J}(\x,\xd)$
	and derive
	\begin{align*}
		\frac{1}{2}\nabla_\x (\yd^\t  \Nb(\y,\yd) \yd) 
		&= \frac{1}{2} \J^\t \nabla_\y (\yd^\t  \Nb(\y,\yd) \yd) + \frac{1}{2} \nabla_\x (\yd) \nabla_{\yd} (\yd^\t  \Nb(\y,\yd) \yd) \\
		&= \frac{1}{2} \J^\t \nabla_\y (\yd^\t  \Nb(\y,\yd) \yd) + \dot\J(\x,\xd)^\t \Nb(\y,\yd) \yd + \frac{1}{2}  \dot\J(\x,\xd)^\t \nabla_{\yd} (\z^\t  \Nb(\y,\yd) \zd)|_{\z=\yd}\\
		&= \frac{1}{2} \J^\t \nabla_\y (\yd^\t  \Nb(\y,\yd) \yd) + \dot\J(\x,\xd)^\t \Nb(\y,\yd) \yd +  \dot\J(\x,\xd)^\t \bm\Xi_{\Nb} (\y,\yd)^\t \J (\x)\xd
	\end{align*}	
	as
	$
	\frac{1}{2}  \partial_{\yd} (\z^\t  \Nb(\y,\yd) \zd)|_{\z=\yd} =
	\frac{1}{2}
	\yd^\t \left( \sum_{i=1}^{n}  \dot{y}_i \partial_{\yd} \nb_i (\y,\yd) \right) 
	= 
	\yd^\t \bm\Xi_{\Nb} (\y,\yd)
	$, where $\nb_i$ is the $i$th column of $\Nb$. 
	
	To further simplify the expression, we note that by $
	\partial_\x \yd 
	= \dot{\J}(\x,\xd)
	$ we have
	\begin{align*}
		\sdot{\Nb}{\x}(\y,\yd) \yd
		&= \sum_{i=1}^{n} \dot{y}_i \partial_\x \nb_i(\y,\yd) \xd  \\
		&= \sum_{i=1}^{n} \dot{y}_i (\partial_\y \nb_i(\y,\yd) \Jb(\x) \xd  + \partial_{\yd} \nb_i (\y,\yd) \partial_\xb (\yd) \xd) \\
		&=  \sum_{i=1}^{n} \dot{y}_i \partial_\y \nb_i(\y,\yd) \yd  + \dot{y}_i \partial_{\yd} \nb_i (\y,\yd) \dot\J (\x,\xd)\xd \\
		&= \left( \sum_{i=1}^{n} \dot{y}_i \partial_\y \nb_i(\y,\yd)\right) \yd  +  \left( \sum_{i=1}^{n}  \dot{y}_i \partial_{\yd} \nb_i (\y,\yd) \right) \dot\J(\x,\xd) \xd\\
		&=  \sdot{\Nb}{\y}(\y,\yd) \yd + 2 \bm\Xi_{\Nb}(\y,\yd) \dot\J(\x,\xd) \xd
	\end{align*}
	Combining these two equalities, we can write
	\begin{align*}
		\bm\xi_{\Mb}(\x,\xd)
		&=  \J(\x)^\t \left( \sdot{\Nb}{\y}(\y,\yd) \yd  - \frac{1}{2} \nabla_\y (\yd^\t  \Nb(\y, \yd) \yd) + 
		(\Nb(\y,\yd) + 2 \bm\Xi_{\Nb}(\y,\yd) ) 
		\dot\J(\x,\xd) \xd  \right) \\
		&\quad -  \dot\J(\x,\xd)^\t \bm\Xi_{\Nb} (\y,\yd)^\t \J (\x)\xd
	\end{align*}
	Substituting the definition of $\bm\xi_{\Nb}(\y,\yd) = \sdot{\Nb}{\y}(\y,\yd) \yd  - \frac{1}{2} \nabla_\y (\yd^\t  \Nb(\y, \yd) \yd)$ proves the general statement.

	In the special case,  $\M(\x,\xd) = \M(\x)$ (which implies $\bm\Xi_{\M}=0$), 
	\begin{align*}
		\bm\xi_{\Mb}(\x,\xd) 
		&=  \J(\x)^\t \left( \bm\xi_{\Nb}(\y,\yd) + 
		\Nb(\y)
		\dot\J(\x,\xd) \xd  \right)
	\end{align*}
	We show this expression is equal to $\bm\eta_{\M;\SS}$ regardless of the structure $\SS$.
	This can be seen from the follows: If further $\Nb(\y) = \Lb(\y)^\t \Cb(\z) \Lb(\y)$ and $\M$ is structured as $(\Lb\Jb)^\t \Cb (\Lb\Jb)$ from some Jacobian matrix $\Lb(\y) = \partial_{\y} \z$, we can write
	\begin{align*}
		\bm\eta_{\M;\SS} 
		&= \J^\t\Lb^\t ( \bm\xi_{\bar\Cb} + 
		\Cb   
		\frac{d(\Lb\J)}{dt} \xd  )\\
		&=\J^\t ( \Lb^\t\bm\xi_{\bar\Cb} + 
		\Lb^\t\Cb  
		(\dot\Lb\J + \Lb\dot\J) \xd  )	\\	
		&=\J^\t \left( \Lb^\t(\bm\xi_{\bar\Cb} + \Cb \dot\Lb \yd )
		+ \Lb^\t\Cb \Lb\dot\J\xd  \right)   \\
		&=\J^\t \left( \bm\xi_{\Nb}
		+ \Nb\dot\J\xd  \right)  = \bm\xi_{\Mb} \qedhere
	\end{align*}
\end{proof}

\subsection{Proof of Proposition~\ref{pr:Lyapunov time derivative}} \label{app:proof of Lyapunov time derivative}
\propositionLyapunovTimeDerivative*
\begin{proof}[Proof of Proposition~\ref{pr:Lyapunov time derivative}]

	Let $K(\q,\qd)= \frac{1}{2} \qd^\t \G(\q,\qd) \qd $. Its time derivative can be written as
	\begin{align*}
		\frac{d}{dt} K(\q, \qd) 
		&= \qd^\t \left( \G(\q,\qd) \qdd + \frac{1}{2} (\frac{d}{dt}{\G}(\q,\qd))  \qd \right) \\
		&= \qd^\t \left( \G(\q,\qd) \qdd + \frac{1}{2} \sum_{i=1}^{d} \dot{q}_i \frac{d}{dt} \gb_i(\q,\qd)  \right) \\
		&= \qd^\t \left( \G(\q,\qd) \qdd  
		+ \frac{1}{2}\sum_{i=1}^d  \dot{q}_i \partial_{\q}  \gb_{i} (\q,\qd) \qd 
		+ \frac{1}{2}\sum_{i=1}^d  \dot{q}_i \partial_{\qd}  \gb_{i} (\q,\qd) \qdd \right) \\
		&= \qd^\t \left( (\G(\q,\qd) + \bm\Xi_\G(\q,\qd))\qdd  
		+ \frac{1}{2}\sdot{\G}{\q}(\q,\qd)\qd 
		\right) 
	\end{align*}
	where we recall $\Gb$ is symmetric and  $\sdot{\Gb}{\qb}(\q,\qd) \coloneqq  [\partial_{\q}  \gb_{i} (\q,\qd) \qd]_{i=1}^d$. Therefore, by definition $(\Gb(\q,\qd) + \bm\Xi_{\G}(\q,\qd))\qdd =( 
	- \bm\eta_{\G;\SS}(\q,\qd)  - \nabla_\q \Phi(\q) - \Bb(\q,\qd) \qd(\q,\qd)  )$, we can derive
	\begin{align*}
		\frac{d}{dt} \V(\q,\qd) 
		&= \frac{d}{dt} K(\q, \qd) + \qd^\t \nabla_\q \Phi(\q)   \\
		&= \qd^\t \left(-\bm\eta_{\G;\SS}(\q,\qd)  - \nabla_\q \Phi(\q) - \Bb(\q,\qd)\xd
		+ \frac{1}{2}\sdot{\G}{\q}(\q,\qd)\qd  + \nabla_\q\Phi(\q)  
		\right) \\
		&= -\qd^\t\Bb(\q,\qd)\qd + \qd^\t \left(- \bm\eta_{\G;\SS}(\q,\qd) 
		+ \frac{1}{2}\sdot{\G}{\q}(\q,\qd)\qd   
		\right) 
	\end{align*}
	
	To finish the proof, we use two lemmas below. 
	\begin{lemma} \label{lm:rate of the curvature term}
		$ \frac{1}{2}\qd^\t \sdot{\G}{\q}(\q,\qd) \qd = \qd^\t \bm\xi_\G(\q,\qd) $. 
	\end{lemma} 
	\begin{proof}[Proof of Lemma~\ref{lm:rate of the curvature term}]
		This can be shown by definition: 
		\begin{align*}
			\qd^\t \bm\xi_\G(\q,\qd) 		
			&=\qd^\t \left( \sdot{\G}{\q}(\q,\qd) \qd - \frac{1}{2} \nabla_{\q} (\qd^\t \Gb(\q,\qd) \qd) \right)  \\
			&= 
			\sum_{k=1}^{d} \dot{q}_k \left( \sum_{i,j=1}^{d} \dot{q}_i\dot{q}_j \partial_{q_j}  G_{k,i} - \frac{1}{2}\sum_{i,j=1}^{d} \dot{q}_i\dot{q}_j \partial_{q_k}  G_{i,j} \right)\\
			&= 
			\sum_{i,j,k=1}^{d} \dot{q}_i\dot{q}_j\dot{q}_k \partial_{q_j}  G_{k,i} - \frac{1}{2} \sum_{i,j,k=1}^{d} \dot{q}_i\dot{q}_j \dot{q}_k \partial_{q_k}  G_{i,j} \\
			&= 
			\sum_{i,j,k=1}^{d} \dot{q}_i\dot{q}_j\dot{q}_k \partial_{q_k}  G_{j,i} - \frac{1}{2} \sum_{i,j,k=1}^{d} \dot{q}_i\dot{q}_j \dot{q}_k \partial_{q_k}  G_{i,j} \\
			&= 
			\frac{1}{2}\sum_{i,j,k=1}^{d} \dot{q}_i\dot{q}_j\dot{q}_k \partial_{q_j}  G_{k,i} = \frac{1}{2}\qd^\t \sdot{\G}{\q}(\q,\qd) \qd
		\end{align*}
		where for the second to the last equality we use the symmetry  $G_{i,j}= G_{j,i}$.
	\end{proof}

	Using Lemma~\ref{lm:rate of the curvature term}, we can show another  equality. 
	\begin{lemma} \label{lm:rate of the structured curvature term}
		For all structure $\SS$, $\qd^\t \left(- \bm\eta_{\G;\SS}(\q,\qd) 
		+ \frac{1}{2}\sdot{\G}{\q}(\q,\qd)\qd   \right) = 0$
	\end{lemma}
	\begin{proof}[Proof of Lemma~\ref{lm:rate of the structured curvature term}]
		This can be seen from Lemma~\ref{lm:pullback of xi}. Suppose $\SS$ factorizes $\G(\q,\qd) = \J(\q)^\t\Hb(\x,\xd)\J(\q)$ where $\J(\q) = \partial_{\q} \x$.
		By Lemma~\ref{lm:pullback of xi}, we know 
		\begin{align*}
			\bm\xi_{\Gb}
			&=  \J^\t \left( \bm\xi_{\Hb} + (\Hb + 2\bm\Xi_{\Hb})\dot\J \xd  \right)  
			- \dot\J^\t \bm\Xi_{\Hb}^\t \J \xd 
		\end{align*}
		On the other hand, by definition, we have  $
		\bm\eta_{\G;\SS} 
		\coloneqq  \J^\t ( \bm\xi_{\Hb} + 
		(\Hb + \bm\Xi_{\Hb} ) 
		\dot{\J} \xd  )
		$. 
		Therefore, by comparing the two, we can derive, 
		\begin{align*}
			\qd^\t \bm\xi_{\G}  = \qd^\t \left( 
			\bm\eta_{\G;\SS} 
			+ \J^\t \bm\Xi_{\Hb}\dot\J \qd - \dot\J^\t \bm\Xi_{\Hb}^\t \J\qd \right) =  \qd^\t 	\bm\eta_{\G;\SS}
		\end{align*}
		Combing the above equality and Lemma~\ref{lm:rate of the curvature term} proves the equality.
	\end{proof}

	Finally,  we use Lemma~\ref{lm:rate of the structured curvature term} and the previous result and conclude 
	\begin{align*}
		\frac{d}{dt} \V(\q,\qd) 
		&= -\qd^\t\Bb(\q,\qd)\qd + \qd^\t \left(- \bm\eta_{\G;\SS}(\q,\qd) 
		+ \frac{1}{2}\sdot{\G}{\q}(\q,\qd)\qd   
		\right)  = -\qd^\t\Bb(\q,\qd)\qd \qedhere
	\end{align*}
	
\end{proof}

\subsection{Proof of Theorem~\ref{th:condition on velocity metric}} \label{app:proof of condition on velocity metric}

\theoremVelocityMetric*
\begin{proof}
	Let $\Ab(\x,\xd) = \Rb(\x) + \Lb(\x)^\t \Db(\x,\xd) \Lb(\x)$. 
	The proof of the theorem is straightforward, if we show that $\bm\Xi_{\Ab}(\x,\xd) \succeq 0$.
	To see this, suppose $\Lb = \R^{n \times m}$. Let $\bm\omega_j^\t$ be the $j$th row $\Lb$, respectively. By definition of $\bm\Xi_{\Ab}(\x,\xd)$ we can write
	\begin{align*}
		\bm\Xi_{\Ab}(\x,\xd) 
		&= \frac{1}{2}\sum_{i=1}^{m} \dot{x}_i \partial_{\xd} \ab_i(\x, \xd) \\
		&= \frac{1}{2}  \Lb(\x)^\t \sum_{i=1}^{m} \dot{x}_i \partial_{\xd} (\D(\x,\xd) \lb_i(\x))\\
		&= \frac{1}{2}  \Lb(\x)^\t \sum_{i=1}^{m} \sum_{j=1}^{n} \dot{x}_i \partial_{\xd} ( d_j(\x, \dot{y}_j) L_{ji}(\x) \eb_j )\\
		&= \frac{1}{2}  \Lb(\x)^\t  \sum_{j=1}^{n} \left( \sum_{i=1}^{m} L_{ji}(\x) \dot{x}_i \right) \partial_{\dot{y}_j} d_j(\x, \dot{y}_j)  \eb_j  \bm\omega_j^\t  \\
		&= \frac{1}{2}  \Lb(\x)^\t  \sum_{j=1}^{n} \dot{y}_j \partial_{\dot{y}_j} d_j(\x, \dot{y}_j)  \eb_j  \bm\omega_j^\t  \\
		&=   \Lb(\x)^\t \bm\Xi_\D (\x, \xd) \Lb(\x)  
	\end{align*}
	where $\eb_j$ the $j$th canonical basis and $\bm\Xi_\D (\x, \xd) = \frac{1}{2}\diag( (\partial_{\dot{y}_i} d_i(\x, \dot{y}_i) )_{i=1}^n)$. Therefore, under the assumption that $\partial_{\dot{y}_i} d_i(\x, \dot{y}_i) \geq 0 $, $\bm\Xi_{\Ab} (\x,\xd) \succeq   0$. 
	This further implies $\bm\Xi_{\G}(\q,\qd) \succeq 0$ by Theorem~\ref{th:consistency}.

	The stability of the entire system follows naturally from the rule of \pullback, which ensures that $\Mb_r(\q,\qd) = \Mb(\q,\qd) = \G(\q,\qd) + \bm\Xi_{\G}(\q,\qd) \succ 0$ given that the leaf-node condition is satisfied. Consequently, the condition in Corollary~\ref{cr:stability} holds and the convergence to $\CC_\infty$ is guaranteed. 
\end{proof}

\subsection{Notation for Coordinate-Free Analysis} \label{app:coordinate-free notation}
We introduce some extra notations for the coordinate-free analysis. Let $p_{T\CC}: T\CC \to \CC$ be the bundle projection.
Suppose $(U, (\q, \vb))$ is a (local) chart on $T\CC$. Let $\{\ppartial{q_i}, \ppartial{v_i}\}_{i=1}^d$ and $\{\d{q^i}, \d{v^i}\}_{i=1}^d$ denote the induced frame field and coframe field on $T\CC$.
For $ s \in U$, we write $s$ in  coordinate  as $(\q(q), \vb(s))$, if $\sum_{i=1}^{d} v_i(s) \frac{\partial}{\partial q_i} |_q \in T_{q} \CC$, where  $q = p_{T\CC}(s)\in \CC$.
With abuse of notation, we also write $s = (\q, \vb)$ for short unless clarity is lost.
Similarly, a chart $(\tilde{U}, (\q, \vb, \ub, \ab))$ can naturally be constructed on the double tangent bundle $TT\CC$, where $\tilde{U} = p_{TT\CC}^{-1}(U)$ and $p_{TT\CC}: TT\CC \to T\CC$ is the bundle projection: we write $h = (\q, \vb, \ub, \ab) \in TT\CC$ if  $\sum_{i=1}^{d} u_i(h) \frac{\partial}{\partial q_i} |_s +  a_i(h) \frac{\partial}{\partial v_i} |_s \in T_{s} T \CC$, where $s = p_{TT\CC}(h) $.
Under these notations, for a curve $q(t)$ on $\CC$, we can write $\ddot{q}(t) \in TT\CC$ in coordinate as $(\q(t), \qd(t), \qd(t),  \qdd(t))$. 
Finally, given Christoffel symbols $\Gamma_{i,j}^k$, an affine connection $\nabla$ on $TT\CC$ is defined via 
$
\textstyle\nabla_{ \ppartial{s_i} } \ppartial{s_j}
= \sum_{k=1}^{2d} \Gamma_{i,j}^k  \ppartial{s_k}
$, where $\ppartial{s_i} \coloneqq \ppartial{q_i}$ and  $ \ppartial{s_{i+d}} \coloneqq \ppartial{v_i}$ for $i=1,\dots,d$.

\subsection{Proof of Theorem~\ref{th:geometric acceleration}}

\theoremGeometricAcceleration*

\begin{proof}[Proof of Theorem~\ref{th:geometric acceleration}]
	

	We first show $\conn$ is unique, if it exists. That is, there is at most one affine connection that is compatible with the given the Riemannian metric $G$ and satisfies for $i,j = 1,\dots, d$ and $k=1,\dots, 2d$
	\begin{align*}
		\Gamma_{i,j}^k = \Gamma_{ji}^k, \qquad 
		\Gamma_{i,j+d}^k = 0,  \qquad 
		\Gamma_{i+d,j+d}^k = \Gamma_{j+d,i+d}^k, 
	\end{align*}
	Importantly, we note that this definition is coordinate-free, independent of the choice of chart on $\CC$.

	The uniqueness is easy to see. As $G$ is non-degenerate by definition, we recall there is an unique Levi-Civita connection, which is compatible with $G$ and satisfies the symmetric condition
	\begin{align*}
		\Gamma_{i,j}^k &= \Gamma_{j,i}^k, &&  \text{for } i,j = 1, \dots, 2d 
	\end{align*}
	Comparing our asymmetric condition and the symmetric condition of the Levi-Civita connection, we see that number of the linearly independent constraints are the same; therefore if there is a solution to the required asymmetric affine connection, then it is unique.  
	
	Next we show such a solution exists. We consider the candidate Christoffel symbols below and show that they satisfy the requirements: Consider an \emph{arbitrary} choice of chart on $\CC$. For $i,j,k=1,\dots,d$,
	\begin{align*}
		\Gamma_{i,j}^k &= \frac{1}{2} \sum_{l=1}^{d} G^{v\sharp}_{k,l}( \partial_{q_j} G^v_{l,i} +  \partial_{q_i} G^v_{l,j} - \partial_{q_l} G^v_{i,j} )  \\
		\Gamma_{i,j+d}^k &= 0, 
		\quad 
		\Gamma_{i+d,j}^k = \frac{1}{2} \sum_{l=1}^{d} G^{v\sharp}_{k,l} ( \partial_{v_i} G^v_{l,j}   ), \quad 
		\Gamma_{i+d,j+d}^k = 0\\
		\Gamma_{i,j}^{k+d} &= 0 
		, \quad
		\Gamma_{i,j+d}^{k+d} = 0
		, \quad
		\Gamma_{i+d,j}^{k+d} = 0
		, \quad
		\Gamma_{i+d,j+d}^{k+d} = 0
	\end{align*}
	where $G^{v\sharp}$ denotes the inverse of $G^v$, i.e. $\sum_{k=1}^{d} G^{v\sharp}_{i,k} G^v_{k,j} = \delta_{i,j}$. Note we choose not to adopt the Einstein summation notation, so the sparse pattern of the proposed Christoffel symbols are clear.
	
	It is clear that the above candidate Christoffel symbols satisfies the asymmetric condition. Therefore, to show it is a solution, we only need to show such choice is compatible with $G$. Equivalently, it means for arbitrary smooth sections of $TT\CC$, $X = \sum_{i=1}^{2d} X_i \ppartial{s_i}$, $Y = \sum_{i=1}^{2d} Y_i \ppartial{s_i}$, $Z = \sum_{i=1}^{2d} Z_i \ppartial{s_i}$, it holds\footnote{The section requirement on $Z$ can be dropped.}
	\begin{align} \label{eq:compatible condition}
		\conn_Z G(X,Y) =  G(\conn_Z X,Y) +  G(X, \conn_Z Y)
	\end{align}
	To verify~\eqref{eq:compatible condition}, we first write out $\conn_Z X$ using the chosen Christoffel symbols: 
	\begin{align} \label{eg:conn in coordinate}
		\conn_Z X &= \sum_{k=1}^{2d} \left( \conn_Z X_k  + \sum_{i,j=1}^{2d} \Gamma_{ij}^k Z_i X_j \right) \ppartial{s_k} \nonumber \\
		&= \sum_{k=1}^{d}  D_Z(X_k) \ppartial{q_k} + \sum_{k=1}^{d}  D_Z (X_{k+d}) \ppartial{v_k}   
		\\
		&\quad + \frac{1}{2} \sum_{k,l=1}^{d}  G^{v,kl} \left( \sum_{i,j,=1}^{d} ( \partial_{q_j} G^v_{li}  +  \partial_{q_i} G^v_{lj} - \partial_{q_l} G^v_{ij} ) Z_i X_j
		+ ( \partial_{v_i} G^v_{lj}   )  Z_{i+d} X_{j} \right) \ppartial{q_k}  \nonumber
	\end{align}
	where $D_Z(\cdot)$ denotes the derivation with respect to $Z$. 
	The above  implies
	\begin{align*}
		G(\conn_Z X,Y)  &= \sum_{j,k=1}^{d} G^v_{ki} Y_k D_Z(X_i)  + \sum_{j,k=1}^{d}  G^a_{kj} Y_{k+d} D_Z (X_{j+d}) 
		\\
		&\quad + \frac{1}{2}  \left( \sum_{i,j,k=1}^{d} ( \partial_{q_j} G^v_{ki} +  \partial_{q_i} G^v_{kj} - \partial_{q_k} G^v_{ij} ) Z_i X_j Y_k
		+ ( \partial_{v_i} G^v_{kj}   )  Z_{i+d} X_{j} Y_k\right)
	\end{align*}
	Similarly, we can derive $G(X, \conn_Z Y)$. Using the symmetry  $G_{ij}^v = G_{ji}^v$, we can combine the previous results together and write
	{\allowdisplaybreaks
		\begin{align*}
			& G(\conn_Z X,Y) +  G(X, \conn_Z Y)\\
			&= \sum_{j,k=1}^{d} G^v_{ki} Y_k D_Z(X_i)  + \sum_{j,k=1}^{d}  G^a_{kj} Y_{k+d} D_Z (X_{j+d})  + 
			\sum_{j,k=1}^{d} G^v_{ki} X_k D_Z(Y_i)  + \sum_{j,k=1}^{d}  G^a_{kj} X_{k+d} D_Z (Y_{j+d}) \\
			&\quad + \frac{1}{2}  \left( \sum_{i,j,k=1}^{d} ( \partial_{q_j} G^v_{ki} +  \partial_{q_i} G^v_{kj} - \partial_{q_k} G^v_{ij} ) Z_i X_j Y_k
			+ ( \partial_{v_i} G^v_{kj}   )  Z_{i+d} X_{j} Y_k\right)  
			\\
			&\quad +\frac{1}{2}  \left( \sum_{i,j,k=1}^{d} ( \partial_{q_j} G^v_{ki} +  \partial_{q_i} G^v_{kj} - \partial_{q_k} G^v_{ij} ) Z_i Y_j X_k
			+ ( \partial_{v_i} G^v_{kj}   )  Z_{i+d} Y_{j} X_k\right)   \\
			&= \sum_{i,j=1}^{d} G^v_{ij} D_Z(X_i) Y_j + G^v_{ij} X_i D_Z(Y_j) + \sum_{i,j=1}^{d} G^a_{ij} D_Z(X_{d+i}) Y_{d+j} + G^a_{ij} X_{d+i} D_Z( Y_{d+j}  ) \\
			&\quad  +  \sum_{i,j, k=1}^{d}  X_i Y_j Z_k \partial_{q_k} G^v_{ij}  +  X_i Y_j Z_{k+d} \partial_{v_k} G^v_{ij}\\
			&=  \sum_{i,j=1}^{d} D_Z(G^v_{ij}) X_i Y_j + G^v_{ij} D_Z(X_i) Y_j + G^v_{ij} X_i D_Z(Y_j) + \sum_{i,j=1}^{d} G^a_{ij} D_Z(X_{d+i}) Y_{d+j} + G^a_{ij} X_{d+i} D_Z( Y_{d+j}  )\\
			&= \conn_Z \left(\sum_{i,j=1}^{d} G^v_{ij} X_i Y_j + \sum_{i,j=1}^{d} G^a_{ij} X_{d+i} Y_{d+j} \right) = \conn_Z G(X,Y)  
		\end{align*}
	}
	Therefore $\conn$ is compatible with $G$. 
	
	So far we have proved the first statement of Theorem~\ref{th:geometric acceleration} that $\conn$ is the unique solution that is compatible with $G$ and satisfies the asymmetric condition. Below we show the expression of $\pr_3(\conn_{\ddot{q}} \ddot{q})$, where we recall $\ddot{q}(t)$ is a curve in $TT\CC$. We use~\eqref{eg:conn in coordinate}. By definition of $\pr_{3}$ it extracts the parts on $\{\ppartial{q_i}\}_{i=1}^d$. Therefore, 
	suppose we choose some chart on $\CC$ of interest and we can write $\pr_3(\conn_{\ddot{q}} \ddot{q})$ as 
	\begin{align*}
		\pr_3(\conn_{\ddot{q}} \ddot{q}) 
		&=  \sum_{k}^{d} \left( D_Z(X_k) + \sum_{l=1}^{d}\frac{1}{2} G^{v,kl} \sum_{i,j,=1}^{d} ( \partial_{q_j} G^v_{li}  +  \partial_{q_i} G^v_{lj} - \partial_{q_l} G^v_{ij} ) Z_i X_j
		+ ( \partial_{v_i} G^v_{lj}   )  Z_{i+d} X_{j} \right) \ppartial{q_k} \\
		&=  \sum_{k}^{d} \left( \ddot{q}_k + \sum_{l=1}^{d}\frac{1}{2} G^{v,kl} \sum_{i,j,=1}^{d} ( \partial_{q_j} G^v_{li}  +  \partial_{q_i} G^v_{lj} - \partial_{q_l} G^v_{ij} ) \dot{q}_i \dot{q}_j
		+ ( \partial_{v_i} G^v_{lj}   )  \ddot{q}_{i} \dot{q}_{j} \right) \ppartial{q_k} \\
		&=  \sum_{k}^{d} a_{\G;k} \ppartial{q_k} 
	\end{align*}
	where $a_{\G;k}$ is the $k$th element of $\ab_\G \coloneqq \qdd +  \Gb(\q,\qd)^{-1} (\bm\xi_{\G}(\q,\qd) + \bm\Xi_{\G}(\q,\qd) \qdd )$.
\end{proof}

\subsection{Proof of Theorem~\ref{th:consistency abstract}}
\theoremAbstractConsistency*

\begin{proof}[Proof of Theorem~\ref{th:consistency abstract}]
	Let $\TT = \TT_1 \times \cdots \times \TT_K$ and $\tilde{G}$ be the induced metric on $T\TT$ by $\{G_i\}_{i=1}^K$. In addition, let $\psi : \CC \to \TT$ be the equivalent expression of $\{\psi_i\}$. 
	Again we focus on the tangent bundle not the base manifold.
	Recall the definition of a pullback connection\footnote{We note the distinction between $\psi^*: T^*\TT \to T^*\CC$ and $T\psi^*: T^*T\TT \to T^*T\CC$.  } $T\psi^* {\Conn{\tilde{G}}}$ is
	\begin{align} \label{eq:pullback connection}
		T\psi_* (T\psi^* {\Conn{\tilde{G}}}_X Y) 
		= \pr_{T\psi_*}^{\tilde{G}} \left(  {\Conn{\tilde{G}}}_{T\psi_* X} {T\psi_* Y} \right) 
	\end{align}
	for all sections $X$ and $Y$ on $TT\CC$, where $\pr_{T\psi_*}^{\tilde{G}}$ is the projection onto the distribution spanned by $T\psi_*$ with respect to $\tilde{G}$, i.e. 
	$\tilde{G}(T\psi_* X, \pr_{T\psi_*}^{\tilde{G}}( Z)) = \tilde{G}(T\psi_* X,  Z)$ for all $X \in TT\CC$ and $Z \in TT\TT$. Note by the construction of the product manifold $\TT$, $T\psi^* {\Conn{\tilde{G}}} = \sum_{i=1}^{K} T\psi_i^* {\Conn{G_i}}$.

	\noindent\textbf{Compatibility  }
	We show  that $T\psi^* {\Conn{\tilde{G}}}$ is compatible with the pullback metric $G$. Let $X,Y,Z$ be arbitrary sections on $TT\CC$ and recall the definition of the pullback metric 
	\begin{align*}
		G(X,Y) = T\psi^*\tilde{G}(X,Y) = \tilde{G}(T\psi_*X, T\psi_*Y)
	\end{align*}
	To show that $T\psi^* {\Conn{\tilde{G}}}$ is compatible, we derive an expression of $G( T\psi^* {\Conn{\tilde{G}}}_Z X, Y)$: 	
	\begin{align*}
		G( T\psi^* {\Conn{\tilde{G}}}_Z X, Y) &=T\psi^*\tilde{G}( T\psi^* {\Conn{\tilde{G}}}_Z X, Y)\\
		&=\tilde{G}( T\psi_* (T\psi^* {\Conn{\tilde{G}}}_Z X), T\psi_*Y) \\
		&=\tilde{G}\left( \pr_{T\psi_*}^{\tilde{G}} \left(  {\Conn{\tilde{G}}}_{T\psi_* Z} {T\psi_* X} \right) , T\psi_*Y \right)\\
		&=\tilde{G}\left(  {\Conn{\tilde{G}}}_{T\psi_* Z} {T\psi_* X}  , T\psi_*Y \right)
	\end{align*}
	where we use~\eqref{eq:pullback connection} and the definition of projection. Using the above equation, we can see the compatibility easily: 
	\begin{align*}
		&G( T\psi^* {\Conn{\tilde{G}}}_Z X, Y) + G(X,  T\psi^* {\Conn{\tilde{G}}}_Z Y) \\
		&= 
		\tilde{G}\left(  {\Conn{\tilde{G}}}_{T\psi_* Z} {T\psi_* X}  , T\psi_*Y \right)
		+ \tilde{G}\left(T\psi_*X,   {\Conn{\tilde{G}}}_{T\psi_* Z} {T\psi_* Y} \right) \\
		&= {\Conn{\tilde{G}}}_{T\psi_* Z}  \tilde{G}(T\psi_* X, T\psi_* Y) \\
		&= \psi^*{\Conn{G}}_{Z}  \tilde{G}(T\psi_* X, T\psi_* Y) \\
		&= \psi^*{\Conn{G}}_{Z}  G(X,Y)
	\end{align*}

	\noindent\textbf{Coordinate expression  }
	The coordinate expression of the pullback metric can be derived by its definition in~\eqref{eq:pullback connection}, and the expression for the pullback covector is standard. 
	For the pullback connection, similar to the proof of Theorem~\ref{th:geometric acceleration}, we can show that
	$\pr_{3} (\lsup{\bar{\nabla}}{G} _{\ddot{q}} \ddot{q})$ can be written as $ \qdd +  \Gb(\q,\qd)^{-1} (\bm\eta_{\G;\SS}(\q,\qd) + \bm\Xi_{\G}(\q,\qd) \qdd )$. In other words, the structured GDS equations are the coordinate expression of the pullback connection $T\psi^* {\Conn{\tilde{G}}}$, where the structure structure $\SS$ is induced through the recursive application of \pullback in \flow. Note that this is in general different from the connection of the pullback metric $\conn$, which by Theorem~\ref{th:geometric acceleration} instead defines the unstructured GDS equation $ \qdd +  \Gb(\q,\qd)^{-1} (\bm\xi_{\G}(\q,\qd) + \bm\Xi_{\G}(\q,\qd) \qdd )$.

	\noindent\textbf{Commutability  }	
  	However, in the special case when $G$ is velocity-independent, we show that they are equivalent. That is, the pullback connection $T\psi^* {\Conn{\tilde{G}}}$ is equal to the connection of the pullback matrix $\conn$. This property is early shown in Theorem~\ref{th:consistency}, which shows that in the velocity-independent case there is no need to distinguish structures.   
	To prove this, we first note that $\conn$ becomes symmetric as $G$ is velocity-independent. As it is also compatible with $G$, we know that $\conn$ is the Levi-Civita connection with respect to $G$. (Recall $G$ is the Riemannian metric on the tangent bundle.)
	On the other hand, knowing that $T\psi^* {\Conn{\tilde{G}}}$ is compatible, to show that $\conn = T\psi^* {\Conn{\tilde{G}}}$ we only need to check if $T\psi^* {\Conn{\tilde{G}}}$ is symmetric.
	Without further details, we note this is implied by the proof of Theorem~\ref{th:consistency}. Therefore, we have $T\psi^* {\Conn{\tilde{G}}} = \conn$.
\end{proof}

\section{Relationship between \flow and Recursive Newton-Euler Algorithms}\label{app:relationship with dynamics}

The policy generation procedure of \flow is closely related to the algorithms~\cite{walker1982efficient} for computing forward dynamics (i.e. computing accelerations given forces) based on recursive Newton-Euler algorithm. In a summary, these algorithms computes the forward dynamics in following steps: 
\begin{enumerate}
\item It propagates positions and velocities from the base to the end-effector.
\item It computes the Coriollis force by backward propagating the inverse dynamics of each link under the condition that the acceleration is zero.
\item It computes the (full/upper-triangular/lower-triangular) joint inertia matrix. 
\item It solves a linear system of equations to obtain the joint acceleration.
\end{enumerate}
In~\cite{walker1982efficient}, they assume a recursive Newton-Euler algorithm (RNE) for inverse dynamics is given, and realize Step 1 and Step 2 above by calling the RNE subroutine. The computation of Step 3 depends on which part of the inertia matrix is computed. In particular, their Method 3 (also called the Composite-Rigid-Body Algorithm in~\cite[Chapter 6]{Featherstone08}) computes the upper triangle part of the inertia matrix by a backward propagation from the end-effector to the base.

\flow can also be used to compute forward dynamics, when we set the leaf-node GDS as the constant inertia system on the body frame of each link and we set the transformation in the \tree as the change of coordinates across of robot links. This works because we show GDSs cover SMSs as a special case, and at root node the effective dynamics is the pullback GDS, which in this case is the effective robot dynamics defined by the inertia matrix of each link. 

We can use this special case to compare \flow with the above procedure. 
We see that the forward pass of \flow is equivalent to Step 1, and the backward pass of \flow is equivalent of Step 2 and Step 3, and the final \resolve operation is equivalent to Step 4. 

Despite similarity, the main difference is that \flow computes the force and the inertia matrix in a \emph{single} backward pass to exploit shared computations. This change is important, especially, the number of subtasks are large, e.g., in avoiding multiples obstacles.
In addition, the design of \flow generalizes these classical computational procedures (e.g. designed only for rigid bodies, rotational/prismatic joints) to handle abstract and even non-Euclidean task spaces that have velocity-dependent metrics/inertias.
This extension provides a unified framework of different algorithms and results in an expressive class of motion policies.

\section{Designing Reactive Motion Policies for Manipulation} \label{apx:Practice}

In this section, we give some details on the RMPs examples discussed in Section~\ref{sec:example RMPs}, which are also used in our manipulation system in the full system experiments. 
We show that these commonly used motion policies are essentially GDSs with respect to some metric and potential function. 
To convert a differential equation back to its GDS formulation, we need to address the question of integrability of a vector field. This is done by showing that a hand-designed vector field is the negative gradient of some potential function. It is useful in these cases to remember that the necessary and sufficient condition on the integrability of a smooth vector field is that its Jacobian is symmetric.

\begin{figure}[t]\vspace{-4mm}
	\centering
	\includegraphics[width=0.9\linewidth]{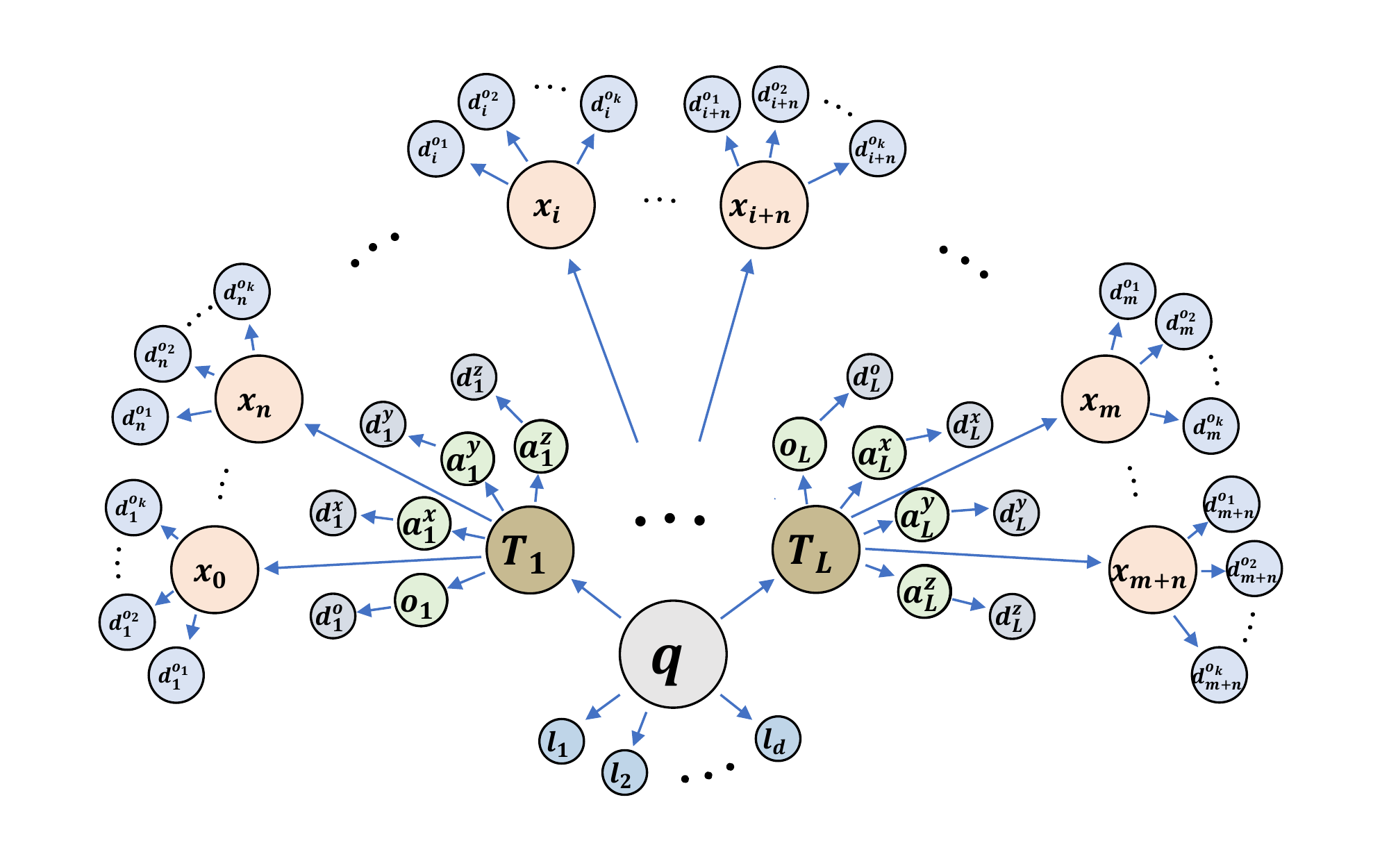}
	\vspace{-2mm}

    \caption{\small This figure depicts 
    the tree of task maps used in the experiments. 
    See Section~\ref{sec:TaskMapTree} for details.
    }
	\label{fig:rmpflow_taskmap_tree}
	\vspace{-6mm}
\end{figure}

\subsection{Task map and its Tree Structure} \label{sec:TaskMapTree}

Figure \ref{fig:rmpflow_taskmap_tree} 
depicts the tree of task maps used in the full-robot experiments.
The chosen structure emphasizes potential for parallelization over fully
exploiting the recursive nature of the kinematic chain, treating each link
frame as just one forward kinematic map step from the configuration
space.\footnote{We could possibly have saved some computation by defining the
forward kinematic maps recursively as $(\mT_{i+1},\q_{i+1},\ldots,\q_d) =
\psi_i(\mT_i,\q_i,\ldots,\q_d)$.} The configuration space $\q$ is linked to $L$
link frames $\mT_1,\ldots,\mT_L$ through the robot's forward kinematics. Each
frame has 4 frame element spaces: the origin $o_i$ and each of the axes
$\ma_i^x,\ma_i^y,\ma_i^z$, with corresponding distance spaces to targets
$d_i^o,d_i^x,d_i^y,d_i^z$ (if they are active). Additionally, there are a
number of obstacle control points $\x_j$ distributed across each of the links,
each with $k$ associated distance spaces $d_j^{o_1},\ldots,d_j^{o_k}$, one
for each obstacle $o_1,\ldots,o_k$. Finally, for each dimension of the
configuration space there's an associated joint limit space $l_1,\ldots,l_d$.

\subsection{Example: 1D Velocity-Dependent Metrics}

We start with an analysis of a simple 1-dimensional GDS with  a velocity-dependent metric to provide some intuition about the curvature terms $\xi$ and $\Xi$. 
This example will be used for constructing collision controllers later.

Let $z = d(\x) \in \R$ be a 1D task space; for instance, 
$d$ might be a distance function
with $\x\in\R^3$.
Let $g(z,\dot{z})$ denote a velocity dependent metric and let $\Phi(z)$ be a potential function. This choice defines a total energy (i.e. the Lyapunov function) 
$V(z,\dot{z}) = \frac{1}{2} g(z,\dot{z})\dot{z}^2 + \Phi(z)$.
It defines a GDS with an equation of motion under external force $f_{\mathrm{ext}} = -\partial_z \Phi - b(z, \dot{z})\dot{z}$ as
\begin{align}
\ddot{z} = \frac{1}{g + \Xi}
\left(-\partial_z\Phi - b(z,\dot{z})\dot{z} - \xi\right),
\end{align}
where $
\Xi = \frac{1}{2}  \dot{z}\frac{\partial g}{\partial\dot{z}}
$, 
$\xi$ is the curvature term (see below), and $b(z,\dot{z})>0$ is the damping coefficient.

Theorem~\ref{th:condition on velocity metric} provides a sufficient condition for stability. In this example, it requires $\dot{z}\frac{\partial g}{\partial \dot{z}}\geq 0$.  
Suppose the metric decomposes as
$g(z,\dot{z}) = w(z)u(\dot{z})$. The sufficient condition of Theorem~\ref{th:condition on velocity metric} becomes
\begin{align} \label{eqn:1DVelMetricCondition}
 2\Xi  = \dot{z}\frac{\partial g}{\partial \dot{z}} = w(z)\dot{z}\frac{d u}{d\dot{z}}\geq 0,
\end{align}
In other words, $u$ needs to change (as a function of $\dot{z}$) in the same direction as the velocity: $u$ either increases in the positive direction when velocity is positive and increases in the negative direction when velocity is negative; or it can be zero. 

Denoting $g_{\dot{z}}(z)$ as the corresponding
function of $\dot{z}$ that arises by fixing the value of $\dot{z}$, we can write 
the curvature terms as
\begin{align}
    \xi 
    &= \left(\frac{d}{dt} g_{\dot{z}}\right)\dot{z} 
        - \frac{d}{dz}\left(\frac{1}{2}g(z,\dot{z})\dot{z}^2\right) \\
    &=  u(\dot{z})\left(\frac{d}{dt} w(z)\right) \dot{z}
        - u(\dot{z})\frac{d}{dz}\left(\frac{1}{2}w(z)\dot{z}^2\right) \\
    &= u(\dot{z})\left(\frac{dw}{dz}\dot{z}^2 
        - \frac{1}{2} \frac{dw}{dz} \dot{z}^2\right) \\\label{eqn:SeparableCurvature}
    &= \frac{1}{2} u\frac{dw}{dz} \dot{z}^2.
\end{align}
Therefore $-\xi = - \frac{1}{2} u\frac{dw}{dz} \dot{z}^2$ is a force that always 
points along decreasing $w$.

\subsection{Collision Avoidance Controllers} \label{apx:Collision}

Here we derive a class of 1D collision controllers defined on the distance
space, and show that the curvature terms in the pullback to $\x\in\R^3$ define
nontrivial curving terms that induce the types of orbits that we see in Figure~\ref{fig:2DOrbits}.

Let $s = d(\x)$ for $\x\in\R^3$ denote a distance function 
$d:\R^3\rightarrow\R_+$. Let $g(s, \dot{s}) = w(s)u(\dot{s})$
denote a 1D separable velocity-dependent metric.
$w(s)$ is
defined
as a non-increasing
function in $s$, i.e. 
and $s_1 \leq s_2 \Rightarrow w(s_1) \geq w(s_2)$ and hence $\frac{dw}{ds}\leq 0$
for all $s\in\R_+$. Typically, $w(s) = 0$ for $s > r_w$ for some nominal radius 
of action $r_w > 0$.
For instance, we might choose 
\begin{align}
    w(s) = \frac{(r_w - s)_+^2}{s},
\end{align}
where $(v)_+ = \max\{0,v\}$.
For this function,
$w(r_w) = 0$ and differentiating it shows that 
$\frac{dw}{ds} = 1 - \frac{r_w^2}{s^2} < 0$ for $s\in(0,r_w)$ and $0$ for $s\geq r_w$.
The equality only holds at $s=r_w$ (i.e. when $w$ is tangent to the $s$ axis at $s = r_w$). 
Its Hessian can be shown as
$\frac{d^2w}{ds^2} = \frac{2r_w^2}{s^3} > 0$ for $s>0$ (i.e. positive definite).
Likewise we choose
\begin{align}
    u(\dot{s}) = 
    \begin{cases}
1 - \exp\left(-\frac{\dot{s}^2}{2\sigma^2}\right), & \mbox{for $\dot{s} < 0$}\\        
0, & \mbox{otherwise}
    \end{cases}
\end{align}
It is straightforward to show that this choice satisfies the condition in~\eqref{eqn:1DVelMetricCondition}, and
$u(\dot{s})\in[0,1)$ with a smooth transition to $0$ for $\dot{s} \geq 0$.

Suppose $\Phi: \R_+ \to \R$ is a potential function that is continuously differentiable. 
As discussed, the GDS would have an equation of motion in the form 
\begin{align*}
\ddot{s}^d 
&= -\frac{1}{m}\frac{d{\Phi}_0}{ds} - \frac{1}{m}\xi
\end{align*}
where we recall $m(s,\dot{s}) = g(s,\dot{s}) + \Xi(s,\dot{s})$ and $\Xi$ and $\xi$ are the curvature terms. 
Particularly, for our chosen product metric $g(s,\dot{s}) = w(s)u(\dot{s})$, we can write $\xi(s,\dot{s}) = \frac{1}{2} u(\dot{s})\frac{dw}{ds}\dot{s}^2$, $\Xi(s, \dot{s}) = \frac{1}{2}w(s)\dot{s}\frac{d u}{d\dot{s}}$ (cf. the previous section), and 
\begin{align*}
m(s, \dot{s}) &= w(s)u(s) + \frac{1}{2}w(s)\dot{s}\frac{d u}{d\dot{s}}\\
&= w(s)\left( u(s) + \frac{1}{2}\dot{s}\frac{d u}{d\dot{s}} \right) \eqqcolon w(s) \delta(s,\dot{s})
\end{align*}
which also has a product structure. 

This factorization shows that the equation of motion can also be written as 
\begin{align*}
\ddot{s}^d 
&= -\frac{1}{\delta} \left(\frac{1}{w}\frac{d{\Phi}}{ds}\right) - \frac{1}{m}\xi\\
&= -\frac{1}{\delta} \left(\frac{d{\wt\Phi}}{ds}\right) - \frac{1}{w}\tilde{\xi}
\end{align*}
for some function $\wt\Phi$ and $\tilde{\xi} = \frac{1}{2} \frac{dw}{ds}\dot{s}^2$ is the curvature term if $w$ is considered as a (velocity-independent) metric. This identification is possible because $\frac{1}{w}\frac{d{\Phi}}{ds}$ is continuous and $s$ is one-dimensional. Conversely, we can start with designing a continuous vector field $\frac{d{\wt\Phi}}{ds}$, i.e. choosing a vector field $\frac{d{\wt\Phi}}{ds}$ such that $\ddot{s}^d = -\frac{1}{\delta} \left(\frac{d{\wt\Phi}}{ds}\right) - \frac{1}{m}\xi$ has the desired behavior. And the above identification shows that this equation of motion is a GDS. 


To represent the above equation as an RMP, 
which is useful if we choose to design behavior by directly defining 
$\frac{d\wt{\Phi}}{ds}$,
we can write it in the natural form as 
\begin{align}
    \Big[-w\frac{d\wt{\Phi}}{ds} - \xi, m \Big]^{\R_+}.
\end{align}
Note that the curvature term $\xi$ behaves as a nonlinear damping term, slowing
the system (from the perspective of the configuration space) as it approaches
obstacles and vanishing when moving away from obstacles. 
Consequently, it biases the system toward curving along isocontours of the distance field. See Fig.~\ref{fig:2DOrbits} for a demonstration of these terms in isolation and in coordination with an obstacle repulsion potential.

\subsection{Attractors} \label{apx:Attractors}

We detail a couple of attractor options here, including two metrics that we have used in practice. We sometimes find the more complex of these metrics
works better on collision avoidance systems, since it expresses a desire for precision near the target while allowing orthogonal compliance further away
giving some freedom for obstacle avoidance.

\paragraph{Notation}
Let $\x$ be the coordinate of the task space (e.g. the coordinate of the task space). We denote the inertia matrix as $\M(\x)$, the forcing potential
function as $\Phi(\x)$, and the damping matrix as $\B(\x,\xd)$. In this section, for designing attractors, we will focus on the special case where $\M(\x) = \G(\x)$.

\subsubsection{Acceleration-based attractors and GDSs}

In many cases, it is straightforward to design a task space behavior 
in isolation in terms of desired accelerations (either by hand 
or through planning) \cite{ratliff2018riemannian}. 
But for stability guarantees we want these systems to be GDSs. 
Specifically, suppose we have a motion policy given as  $\f(\x, \xd)$. Define $\xdd^d = \f(\x, \xd)$ as a shorthand (i.e. the desired acceleration). We want to show that it can be written as 
\begin{align}\label{eq:motion policy of GDS}
\xdd^d = - \M(\x)^{-1}\Big(\nabla_\x\Phi(\x) + \bm{\xi}_\M(\x,\xd) + \B(\x, \xd) \xd\Big)
\end{align}
for some $\M$, $\B$, and $\Phi$, where $\bm{\xi}_\M$ is the associated curvature term of $\M$.
We can view the above decomposition as three parts: 
\begin{enumerate}
	\item The desired acceleration generated by the potential: $- \M(\x)^{-1}\nabla\Phi(\x)$ 
	\item The damping acceleration for stability:  $- \M(\x)^{-1} \B(\x, \xd) \xd$
	\item The curvature acceleration for consistent behaviors: $ - \M(\x)^{-1} \bm{\xi}_\M(\x,\xd) $
\end{enumerate}


To bridge the connection between the GDS formulation in~\eqref{eq:motion policy of GDS} and common motion policies given directly by $\f$, we consider particularly motion policy candidates  that can be written in terms of
\begin{align} \label{eq:candidate motion policy}
\xdd^d = -\nabla_\x \wt\Phi(\x) - \wt\B(\x,\xd)\xd  -\M(\x)^{-1}\bm{\xi}_\M(\x,\xd)
\end{align}
where  $\wt\Phi$ is another potential function and $\wt\B(\x,\xd)$ is another damping matrix. 
We show that it is possible to design $\wt\Phi$ and $\wt\B$ directly, and then choose some proper inertia matrix $\M(\x)$ such that~\eqref{eq:candidate motion policy} can be written as the GDS~\eqref{eq:motion policy of GDS} for some $\Phi$ and $\B$. That is, we show it possible to choose $\M$ such that
\begin{align*}
\nabla_x\Phi(\x) = \M(\x)\nabla_\x \wt\Phi(\x) \qquad \text{and} \qquad \B(\x,\xd) = \M(\x)\wt{\B}(\x,\xd)
\end{align*}
and $\Phi$ is a potential function and $\B$ is positive definite (without the need to 
derive them 
in closed form).
Moreover, we show that this strategy allows us to model some common acceleration-based attractors.

\subsubsection{Motion policy candidates}
As a motivating example of~\eqref{eq:candidate motion policy}, we consider the attractor proposed in~\cite{ratliff2018riemannian}.
Let $h_V^\alpha$ define a {\it soft-normalization} function
\begin{align} \label{eqn:SoftNormalization}
\theta_\alpha(\vv) = \vv / h_V^\alpha(\|\vv\|).
\end{align}
with
$
h_V^\alpha(\gamma) 
= \frac{1}{\alpha}\log(e^{\alpha\gamma} + e^{-\alpha\gamma}) 
= \gamma + \frac{1}{\alpha}\log(1 + e^{-2\alpha \gamma})
$ for some $\alpha > 0$,
so that $\theta_\alpha(\vv)$ approaches $\hat{\vv} = \frac{\vv}{\|\vv\|}$ for larger
$\vv$, but approaches zero smoothly as $\vv\rightarrow 0$.
Without loss of generality, let us consider the center of attraction is at $\x = 0$. The attractor considered in~\cite{ratliff2018riemannian} is given as  
\begin{align} \label{eqn:OrigAttractor}
\xdd^d = \f_{\mathrm{a}}(\x,\xd) \coloneqq -  \gamma_p \theta_\alpha(\x) - \gamma_d \xd,
\end{align}
for some $\gamma_p, \gamma_d >0$.

Inspecting~\eqref{eqn:OrigAttractor}, we can see that it resembles the form~\eqref{eq:candidate motion policy} modulus the last curvature term $-\M(\x)^{-1}\bm{\xi}(\x,\xd)$. 
Indeed we can identify $\tilde{\B} = \gamma_d \I$ and we show below the first term $\gamma_p \theta_\alpha(\x_0 - \x)$ is a derivative of some potential function. 
We do so by showing its Jacobian is symmetric. Using the notation above, we have
\begin{align*}
\frac{d}{d \x}\theta_\alpha(\x)
&= \frac{d}{d \x} \left( h_V^\alpha(\|\x\|)^{-1} \x \right) \\
&= h_V^\alpha(\|\x\|)^{-1} \I 
+ \x \left(-h_V^\alpha(\|\x\|)^{-2} \frac{d h_V^{\alpha}(s)}{ds}|_{s= \norm{\x}}  \frac{\partial}{\partial \x}\|\x\|\right) \\
&= h_V^\alpha(\|\x\|)^{-1}\I - \left( \|\x\| h_V^\alpha(\|\x\|)^{-2} \frac{d h_V^{\alpha}(s)}{ds}|_{s= \norm{\x}} \right) \hat{\x}\hat{\x}^T. 
\end{align*}
Both terms are symmetric, so the Jacobian is symmeric and this vector field
is the gradient of some potential function (say   $\wt{\Phi}_\mathrm{a}^1$), 
although we do not attempt to derive the potential function in closed form here.

In some cases, it is potentially more convenient to start designing~\eqref{eq:candidate motion policy}
 with a known potential function such as
\begin{align}
\wt{\Phi}_\mathrm{a}^2(\x) 
&= \frac{1}{\eta}\log\left(e^{\eta\|\x\|} + e^{-\eta\|\x\|}\right) \nonumber \\
&= \|\x\| + \frac{1}{\eta}\log\big(1 + e^{-2\eta\|\x\|}\big)
\end{align}
so the potential energy can be measured, where $\eta > 0$. 
This is a $\eta$-scaled softmax ($\eta$ defines the effective smoothing radius at the origin) over $\|\x\|$ and $-\|\x\|$, and the second expression is a numerically
robust version since $\|\x\|\geq 0$. Its negative gradient is
\begin{align} \label{eqn:sNumerical}
\nabla_\x \wt{\Phi}^2(\x) 
&= \frac{1}{\alpha}\frac{1}{e^{\alpha\|\x\|} + e^{-\alpha\|\x\|}}
\left(\alpha e^{\alpha\|\x\|}\hat{\x} - \alpha e^{-\alpha\|\x\|}\hat{\x}\right) \nonumber  \\
&= 
\left(
\frac{e^{\alpha\|\x\|} - e^{-\alpha\|\x\|}}{e^{\alpha\|\x\|} + e^{-\alpha\|\x\|}}
\right) \hat{\x} \nonumber \\
&= \left(
\frac{1 - e^{-2\alpha\|\x\|}}{1 + e^{-2\alpha\|\x\|}}
\right) \hat{\x}
= s_\alpha\big(\|\x\|\big) \hat{\x},
\end{align}
where $s_\alpha(0) = 0$ and $s_\alpha(r)\rightarrow 1$ as $r\rightarrow\infty$.
\eqref{eqn:sNumerical} again gives a numerically robust form since $\|\x\|\geq 0$.
Below, we denote abstractly the potential as just $\wt{\Phi}$ so we're agnostic
to the choice of $\wt{\Phi}_\mathrm{a}^i$, $i\in\{1,2\}$.

\subsubsection{Metric options}
Suppose we have chosen some potential $\wt\Phi$ and some damping $\wt{\B}$. We next consider admissible metric/inertia matrices  such that~\eqref{eq:candidate motion policy} can be written as~\eqref{eq:motion policy of GDS}. 
We first note that $\M = \I$ is an admissible choice (i.e. we recover $\Phi(\x) =\wt\Phi$ and $\B= \wt{\B}$, provided $\wt\B$ is positive definite). But this choice is not ideal when we wish to combine multiple motion policies, because we recall that the design of $\M$ designates the importance of each motion policy. Therefore, we would not want to restrict ourselves to the trivial choice $\M = \I$.

Here we present a family of metric matrices that are non-trivial and meaningful in practice, and yet is \emph{compatible} with the motion policy~\eqref{eq:candidate motion policy}.
Let us first define some useful functions to simplify the writing later on.
Let $\alpha(\x) = \exp(-\frac{\|\x\|^2}{2\sigma_\alpha^2})$ and $\gamma(\x) = \exp(-\frac{\|\x\|^2}{2\sigma_\gamma^2})$ for some $\sigma_\alpha, \sigma_\gamma \in \R$. 
We define a weight function $w(\x) = \gamma(\x)w_u + (1-\gamma(\x))w_l$, for $0\leq w_l\leq w_u<\infty$. Equivalently, it can be written as $w(\x) = \trihat{w}\gamma(\x) + w_l$ 
with $\trihat{w} \coloneqq w_u - w_l$.
Below we will need $\nabla_\x \log w(\x)$, so we derive it here. Noting $w = \trihat{w}\gamma + w_l$, we get
\begin{align*}
\nabla_\x  \log w (\x)
= \frac{\nabla_\x (\trihat{w}\gamma + w_l)}{w(\x)}
= \frac{\trihat{w}}{w}
\exp
\left(
-\frac{\|\x\|^2}{2\sigma_\gamma^2}\right)\left(-\frac{1}{\sigma_\gamma^2}\x
\right) 
= -\frac{\gamma\trihat{w}}{\sigma_\gamma^2 w} \x.
\end{align*}

We define two alternative metrics. The first metric trades off stretching 
the space in the direction toward the target when the robot is away from the goal, and
becoming increasingly Euclidean when the robot is close to the goal:
\begin{align}
    \M_{\mathrm{stretch}} = 
        w(\x)\Big(\big(1-\alpha(\x)\big)\nabla_\x\wt{\Phi}\nabla_\x\wt{\Phi}^\t + (\alpha(\x) + \epsilon)\I\Big),
\end{align}
where $\epsilon > 0$ induces a baseline Euclidean metric used far from the target to  fill out the metric's eigen-spectrum, and $\wt\Phi$ is the potential in~\eqref{eq:candidate motion policy}.
The second metric matrix is simply
\begin{align}
    \M_{\mathrm{uni}} = w(\x)\I.
\end{align}
We refer to these both generically as $\M$ below. Note again that 
we use $\M$ here rather than $\G$ since these metrics are velocity 
independent so that the inertia matrix $\M$ and the metric (typically 
denoted as $\G$) are the same.

\subsubsection{Compatibility between metrics and potentials}

We show the two metrics $\M_{\mathrm{stretch}}$ and $\M_{\mathrm{uni}}$ above are compatible with $\wt{\Phi}_\mathrm{a}^i$, $i\in\{1,2\}$. 
For simplicity, let us denote them just as $\M$ and $\wt\Phi$. We will show that there exists a potential $\Phi$ such that $\nabla_\x\Phi = \M\nabla_\x\wt{\Phi}$. In fact, our result applies to potentials more general than $\wt{\Phi}_\mathrm{a}^i$ for $i\in\{1,2\}$. It applies to all potentials $\wt{\Phi}(\x)$ such that $\nabla_\x \wt{\Phi}(\x) = \kappa(\|\x\|)\hat{\x}$ for some function
$\kappa:\R\rightarrow\R$, which includes $\wt{\Phi}_\mathrm{a}^i$ for $i\in\{1,2\}$ as special cases.

We prove the  existence by analyzing the Jacobian of $\M\nabla_\x\wt{\Phi}$ and showing that it is symmetric. 
We first note that a result of radial symmetry.
\begin{lemma} \label{lm:symmetry condition}
Let $\nabla_\x \wt{\Phi}(\x) = \kappa(\|\x\|)\hat{\x}$ for some
$\kappa:\R\rightarrow\R$ operating on the distance to the origin, and let $f$ be a differentiable function. Then the Jacobian matrix 
\begin{align} \label{eqn:ScaledAttractorPotentialJacobian}
\frac{\partial}{\partial\x} \big(f\nabla_\x\wt{\Phi}\big) 
&=  f(\x) \nabla_\x^2\wt{\Phi} + \kappa(\x)f'(\|\x\|)\hat{\x}{\hat{\x}}^\t
\end{align}
is symmetric. 
\end{lemma}
\begin{proof}
We first note that $\nabla f(\|\x\|) = f'(\|\x\|)\hat{\x}$ for 
all differentiable $f$. Then the results follow directly as the derivation below
\begin{align} 
\frac{\partial}{\partial\x} \big(f\nabla_\x\wt{\Phi}\big) 
&= f(\x) \nabla_\x^2\wt{\Phi} + \nabla\wt{\Phi}\nabla f^\t \nonumber \\
&=  f(\x) \nabla_\x^2\wt{\Phi} + \gamma(\x)f'(\|\x\|)\hat{\x}{\hat{\x}}^\t
\end{align}
because the Hessian $\nabla^2\wt{\Phi}$ is symmetric.
\end{proof}

Given Lemma~\ref{lm:symmetry condition}, showing symmetry of the Jacobian of $\M\nabla_\x\wt{\Phi}$ is straightforward, because both $\wt\Phi$ considered satisfy $\nabla\wt{\Phi} = \kappa(\|\x\|)\hat{\x}$ for some $\kappa$. 
First, we consider $\M_{\mathrm{uni}}$. We can write $\M_\mathrm{uni}\nabla\wt{\Phi} = w(\x)\nabla\wt{\Phi}$ and $w(\x) = \wt{w}(\|\x\|)$
with $\wt{w}(s) = \trihat{w}\exp(-\frac{s^2}{2\sigma_\gamma}) + w_l$, 
so its Jacobian is symmetric. 
%
Similarly,
\begin{align*}
    \M_\mathrm{stretch}\nabla\wt{\Phi} 
    &= 
        w(\x)\Big(
            \big(1-\alpha(\x)\big)\nabla_\x\wt{\Phi}\nabla_\x\wt{\Phi}^T 
            + \alpha(\x)
        \Big) \nabla_\x\wt{\Phi} \\
    &= w(\x) \Big(
        \big(1-\alpha(\x)\big)\|\nabla_\x\wt{\Phi}\|^2\nabla_\x\wt{\Phi} 
        + \alpha(\x)\nabla_\x\wt{\Phi}
    \Big) \\
    &=  \wt{w}\big(\|\x\|\big) h\big(\|\x\|\big) \nabla_\x\wt{\Phi},
\end{align*}
where $h(s) = (1-\wt{\alpha}(s))\kappa(s)^2 + \wt{\alpha}(s)$ with 
$\wt{\alpha}(s) = \exp(-\frac{s^2}{2\sigma_\alpha})$. 
This expression fits in the form considered in Lemma~\ref{lm:symmetry condition} and therefore it has a symmetric Jacobian.



\subsubsection{Compatibility between metrics and damping}
The condition for the damping part is relative straightforward. We simply need to choose $\wt{\B}$ such that
\begin{align*}
\B(\x,\xd) = \M(\x)\wt{\B}(\x,\xd) \succ 0 
\end{align*}
A sufficient condition is to set $\wt{\B}$ to share the same eigen-system as $\M$.

\subsubsection{Effects of the curvature term}

We have provided conditions for compatibility between metrics, potentials, and damping. We now consider the effects of the curvature acceleration $-\M(\x)^{-1}\bm{\xi}_\M(\x,\xd)$
\begin{align} \tag{\ref{eq:candidate motion policy}}
\xdd^d = -\nabla_\x \wt\Phi(\x) - \wt\B(\x,\xd)\xd  -\M(\x)^{-1}\bm{\xi}_\M(\x,\xd)
\end{align}
due to our non-trivial choice of metric matrix. 

For $\M_\mathrm{uni} = w(\x)\I$, this becomes
\renewcommand{\half}{{\frac{1}{2}}}
\begin{align*}
    \bm{\xi}_\M
    &= \frac{dw}{dt}\xd - \half \nabla_\x w\|\xd\|^2 \\
    &= \big(\xd^\t\nabla_\x w\big)\xd - \half \nabla_\x w\|\xd\|^2 \\
    &= \left(\big(\xd\xd^\t\big) - \half \|\xd\|^2\right) \nabla_\x w\\
    &= -\half\|\xd\|^2\left(\I - 2\ \hat{\xd}\hat{\xd}^T\right) \nabla_\x w
\end{align*}
where $\hat{\xd} = \frac{\xd}{\|\xd\|}$. That gives 
\begin{align*}
    -\M_{\mathrm{uni}}^{-1} \bm{\xi}_\M
    &= \half \|\xd\|^2\left(\I - 2\ \hat{\xd}\hat{\xd}^T\right) \frac{\nabla w}{w} \\
    &= \half \|\xd\|^2 H^r_\xd \big[\nabla\log w\big],
\end{align*}
where $H^r_\vv[\y] = \left(\I - 2\ \hat{\vv}\hat{\vv}^T\right)\y$ 
is the Householder reflection of $\y$ across the plane normal to $\vv$. In this case, 
it acts to align the system toward the goal and provides a bit of drag.

The derivatives of $\M_\mathrm{stretch}$ are similar but more complex. We recommend a combination of finite-differencing and automatic differentiation to systematize the calculations.

\subsubsection{Revisiting the attractor in~\cite{ratliff2018riemannian}} 
Let us revisit our motivating example
\begin{align} \tag{\ref{eqn:OrigAttractor}}
\xdd^d = \f_{\mathrm{a}}(\x,\xd) \coloneqq -  \gamma_p \theta_\alpha(\x) - \gamma_d \xd,
\end{align}
From using the results above, we see that~\eqref{eqn:OrigAttractor} fits in the form in~\eqref{eq:candidate motion policy} but missing the curvature term $-\M^{-1} \bm{\xi}_\M$. As we show in Section~\ref{sec:1DExample}, the curvature term provides correction for consistent behaviors and stability, which suggests that the original motion policy in~\cite{ratliff2018riemannian} could lose stability in general (e.g. when the velocity is high). Nevertheless, from the above analysis, we show that if we add the curvature correction back, i.e.,
\begin{align*}
\xdd = \f_{\mathrm{a}}(\x,\xd) - \M^{-1} \bm{\xi}_\M
\end{align*}
then the system is provably stable.

\subsection{Joint Limits} \label{apx:JointLimits}

We adopt a similar approach to handling joint limits as~\cite{ratliff2018riemannian}, but here we modify the velocity dependent components of the metric to match our theoretical requirements for stability
and fully derive the curvature terms introduced by the nonlinearities and velocity dependence.  We emphasize that, due to the invariance of RMPs to
reparameterization that results from our complete handling of curvature terms, the behavior of the joint limit RMPs and the way in which they interact with the rest of the system are independent of the specific implementation.  That said, we derive an analogous result here to the one presented in~\cite{ratliff2018riemannian} to show that these joint limit RMPs effectively scale the columns of each task space's Jacobian matrix to smoothly regulate their degrees of freedom as a function of joint limit proximity. In implementation, these RMPs can be treated as any other RMP.

\subsubsection{Integrating RMPs with joint limits}
We first define a class of joint limit metrics that
can be used in joint limit RMPs. We show, given a joint limit RMP, the RMP algebra defined in Section~\ref{sec:RMPAlgebra} can be seen as
producing the same Jacobian modification as described in~\cite{ratliff2018riemannian}.
We present the result more generally here as a lemma.
 Note that as diagonal entries of $\A$ approach infinity, entries of $\D$ approach zero and the corresponding column of $\wt{\J}$ vanishes.
\begin{lemma}[Effect of diagonal RMPs] \label{lemma:JointLimits}
Let 
$\{(\xdd_i^d,\M_i)^{\TT_i}\}_{i=1}^n$ be a collection of RMPs defined
on task spaces $\TT_i$. Let $\wt{\xdd}_i^d = \xdd_i^d-\Jd_i\qd$ and let
\begin{align}
    [\M\qdd^d,\M]^\mathcal{C} 
    = \Big[\sum_i\J_i^T\M_i\wt{\xdd}_i^d,\ \sum_i\J_i^T\M_i\J_i\Big]^\mathcal{C}
\end{align}
denote their normal form pullback and combination to space $\mathcal{C}$ through task maps 
$\psi_i:\mathcal{C}\rightarrow\TT_i$ with Jacobians 
$\J_i = \frac{\partial\psi_i}{\partial \q}$. Let $[\A\qdd_l^d, \A]^\mathcal{C}$
denote an RMP with 
diagonal\footnote{We choose this form for the diagonal 
dependent metric (without loss
of generality since it's positive definite), to be convenient notationally below.} a velocity-dependent metric 
$\A(\q, \qd) = \lambda\D(\q, \qd)^{-2}$, where $\lambda > 0$.
Then 
$[\M\qdd^d,\M]^\mathcal{C} + [\A\qdd_l^d, \A]^\mathcal{C}$
has metric $\wt{\M} = \D^{-1}\left(\sum_i \wt{\J}^T\M_i\wt{\J} + \lambda\I\right)\D^{-1}$
and motion policy
\begin{align}
    \qdd_c^d 
    &= \D\left(\sum_i \wt{\J}^T\M_i\wt{\J} + \lambda\I\right)^\dagger
        \Big(\sum_i \wt{\J}^T\M_i\wt{\xdd}_i^d + \lambda\D^{-1}\qdd_l^d\Big) 
    \\\label{eqn:JointLimitsLeastSquares}&= 
    \D\left[
        \argmin_{\qdd} \left(\frac{1}{2}\sum_i\|\wt{\xdd}_i^d - \wt{\J}_i\qdd\|^2_{\M_i}
        + \frac{\lambda}{2}\|\wt{\qdd}_l^d - \qdd\|^2\right)\right],
\end{align}
with $\wt{\J} = \J\D$ and $\wt{\qdd}_l^d = \D^{-1}\qdd_l^d$.
\end{lemma}
\begin{proof} Writing out the sum we get
\begin{align*}
    [\M\qdd^d,&\M]^\mathcal{C} + [\A\qdd_l^d, \A]^\mathcal{C}
\\&=
    \left[\sum_i\J_i^T\M_i\wt{\xdd}_i^d + \lambda\D^{-2}\qdd_l^d,\
        \sum_i\J_i^T\M_i\J_i + \lambda\D^{-2}\right]^\mathcal{C}.
\\&=
    \left[\D^{-1}\Big(\sum_i\wt{\J}_i^T\M_i\wt{\xdd}_i^d + \lambda\D^{-1}\qdd_l^d\Big),\
        \D^{-1}\Big(\sum_i\wt{\J}_i^T\M_i\wt{\J}_i + \lambda\I\Big)\D^{-1}\right]^\mathcal{C}.
\end{align*}
This gives the expression for the metric, and the motion policy can be obtained by $\resolve$:
\begin{align*}
\wt{\qdd}_d 
&= 
    \D\left(\sum_i\wt{\J}_i^T\M_i\wt{\J}_i + \lambda\I\right)^{\dagger}
    \left(\sum_i\wt{\J}_i^T\M_i\wt{\xdd}_i^d + \lambda\D^{-1}\qdd_l^d\right)
\\&=
    \D\ \resolve\left(\Big[\sum_i\wt{\J}_i^T\M_i\wt{\xdd}_i^d + \lambda\D^{-1}\qdd_l^d,
        \sum_i\wt{\J}_i^T\M_i\wt{\J}_i + \lambda\I\Big]^\mathcal{C}\right),
\end{align*}
which is equivalent to the least squares form in~\eqref{eqn:JointLimitsLeastSquares}.
\end{proof}

\subsubsection{A class of velocity-dependent joint limit metrics} \label{sec:JointLimitMetrics}

Here we develop a velocity dependent metric to represent joint limits. We construct
it for each joint independently, denoting the joint angle by $q\in[l_l,l_u]$. 
Let $a(q, \dot{q})$ denote a one-dimensional velocity-dependent metric on $q$. We want 
$a\rightarrow\infty$ as $q$ is close to the joint limit and 
$\dot{q}$ heads toward the joint limit. 
Such a metric can be constructed
using a form related to that given in \cite{ratliff2018riemannian},
choosing $a = b^{-2}$ for
\begin{align} \label{eqn:JointLimitInvMetric}
    b = 
    s \Big(\alpha_u d + (1-\alpha_u) 1\Big) 
    + (1-s) \Big(\alpha_l d + (1-\alpha_l) 1\Big).
\end{align}
with 
$s = \frac{q - l_l}{l_u - l_l}$,
$d = 4s(1-s) = 4\left(\frac{q-l_l}{l_u-l_l}\right)\left(\frac{l_u - q}{l_u - l_l}\right)$, 
and velocity gates $\alpha_u = 1 - e^{-\dot{q}_+^2/(2\sigma^2)}$ 
and $\alpha_l = 1 - e^{-\dot{q}_-^2/(2\sigma^2)}$ for $\sigma>0$.
(Choosing $a = b^{-2}$ makes intuitive sense with regard to Lemma~\ref{lemma:JointLimits}.)
Since $s = \frac{q - l_l}{l_u - l_l}$ and 
$1-s = \frac{l_u - q}{l_u-l_l}$, $s$ indicates whether $q$ is close to $l_u$ 
($s\rightarrow 1$ as $q\rightarrow l_u$)
while $1-s$ indicates whether it is close to $l_l$. Likewise, $\alpha_u$ indicates
whether $\dot{q}$ is moving in a positive direction and $\alpha_l$ indicates
a negative direction. Therefore, this equation encodes a smoothed binary logic
that can be read ``if close to the upper limit and moving in the positive 
direction use $d$,
and if close to the lower limit and moving negatively $d$; in 
all other cases use $1$.'' Said another way, ``if close to either limit 
and moving toward it, use $d$, otherwise use $1$.'' Note that $\sup d = 1$
and $d\rightarrow 0$ as $q\rightarrow\{l_l,l_u\}$, so $a = b^{-2}$ has the 
desired property discussed above. 
All that remains to be shown is that this choice of $a(q, \dot{q})$
satisfies the condition $\dot{q}\frac{\partial a}{\partial\dot{q}}\geq 0$ of 
Theorem~\ref{th:condition on velocity metric}.

\begin{lemma} The velocity-dependent metric defined by $a = b^{-2}$ with $b$ given 
by~\eqref{eqn:JointLimitInvMetric} satisfies 
the sufficient condition of Theorem~\ref{th:condition on velocity metric} for stability,
i.e.
$\frac{\partial a}{\partial \dot{q}}\dot{q}\geq0$ for all $\dot{q}$.
\end{lemma}
\begin{proof}
We start by noting
\begin{align}
    \frac{\partial a}{\partial \dot{q}}\dot{q} = \frac{\partial}{\partial\dot{q}}b^{-2}\dot{q} 
    = -2b^{-3} \frac{\partial b}{\partial\dot{q}}\dot{q}.
\end{align}
Since $b \geq 0$ we have $\frac{\partial b}{\partial\dot{q}}\dot{q}\leq0$ implies 
$\frac{\partial a}{\partial \dot{q}}\dot{q}\geq 0$. We can rearrange $b$ to be more 
transparent
to derivatives with respect to $\dot{q}$:
\begin{align}
    b &= s\Big(\alpha_u(d-1) + 1\Big) + (1-s)\Big(\alpha_l(d-1)+1\Big) \\
    &= -s(1-d)\alpha_u - (1-s)(1-d) \alpha_l + c,\\
    &= -\gamma_u\alpha_u - \gamma_l\alpha_l + c,
\end{align}
where $c$ is independent of $\dot{q}$ and where
$\gamma_u,\gamma_l\geq0$ and both independent of $\dot{q}$.
Therefore, we have
\begin{align} \label{eqn:JointLimitInvMetricDeriv}
    \frac{\partial b}{\partial\dot{q}}\dot{q} = -\gamma_u\frac{\partial\alpha_u}{\partial\dot{q}} 
        - \gamma_l\frac{\partial\alpha_l}{\partial\dot{q}}.
\end{align}
Since $\frac{\partial\alpha_u}{\partial\dot{q}} > 0$ for $\dot{q} > 0$ and $0$ otherwise,
while $\frac{\partial\alpha_l}{\partial\dot{q}} < 0$ for $\dot{q} < 0$ and $0$ otherwise,
Equation~\ref{eqn:JointLimitInvMetricDeriv} implies 
$\frac{\partial b}{\partial\dot{q}}\dot{q}\leq0$ for all $\dot{q}$ and hence
$\frac{\partial a}{\partial\dot{q}}\dot{q}\geq 0$ for all $\dot{q}$.
\end{proof}     

We note that there are other choices for joint limit metrics, including those
used for obstacle avoidance. In fact, one way to create joint limit controllers
would be to treat each joint limit as an obstacle. We choose to use the above
limit controller due to its intuitive interpretation as a velocity-dependent
modification of a controller designed in a space $u$ with the relationship $q =
(l_u - l_l)\sigma(u) + l_l$ with $\sigma(u) = 1/ (1+e^{-u})$.

\subsubsection{Motion policies for joint-limit avoidance}

The differential equation $\qdd_l^d$ from the joint limit RMPs (see Lemma~\ref{lemma:JointLimits}) encodes the curvature terms
from the metric $\A$. Denoting those as $\A^{-1}\bm\xi_\A$ with 
$\bm\xi_\A = \sdot{\A}{\q}\qd - \frac{1}{2}\nabla_\q\left(\qd^T\A\qd\right) 
= \diag\big(\frac{1}{2}\frac{d\A_{ii}}{dq_i}\dot{q}_i^2\big)_i$, 
we often choose this differential
equation to be 
\begin{align}
    \qdd_l^d = \big(\eta_p(\q_0 - \q) - \eta_d\qd\big) - \A^{-1}\bm\xi_\A,
\end{align}
for $\eta_p,\eta_d\geq0$.
As shown in Lemma~\ref{lemma:JointLimits}, this differential equation can be viewed as a transformation  $\wt{\qdd}_l^d = \D^{-1}\qdd_l^d = \A^{\frac{1}{2}}\qdd_l^d$ 
in the final joint limit 
corrected expression. Since $\A$ becomes large near joint limits, this transformation effectively scales up the $i$th dimension of $\qdd_l^d$ when $q_i$ nears a joint limit and $\dot{q}_i$
is headed toward it.

\section{Details of the Reaching-through-clutter Experiments} \label{apx:ReachingExperiment}

We give some details on the reaching experiments here.

\subsection{Experimental method}

We set up a collection of clutter-filled environments with cylindrical
obstacles of varying sizes in simulation as depicted in Fig.~\ref{fig:robots}, and tested the performance of \flow and two potential 
field methods on a modeled ABB YuMi robot.

\vspace{3mm}
\noindent Compared methods:
\begin{enumerate}
\item {\bf \flow:} We implement \flow using the RMPs  in Section~\ref{sec:example RMPs} and 
  detailed in Appendix~\ref{apx:Practice}. In particular, we place collision-avoidance
  controllers on distance spaces $s_{ij} = d_j(\x_i)$, where $j=1,\ldots,m$ indexes
  the world obstacle $o_j$ and $i=1,\ldots,n$ indexes the $n$ control point along the robot's
  body. Each collision-avoidance controller uses a weight function $w_o(\x)$ that 
  ranges from $0$ when the robot is far from the obstacle to $w_o^\mathrm{max}\gg 0$ when the robot is in contact with the obstacle's surface.
  Similarly, the attractor potential uses a weight function $w_a(\x)$ that 
  ranges from $w_a^{\mathrm{min}}$ far from the target to $w_a^{\mathrm{max}}$
  close to the target.
\item {\bf PF-basic:} This variant is a basic implementation of obstacle 
  avoidance potential fields with dynamics shaping. We use the 
  RMP framework to implement this variant by placing collision-avoidance
  controllers on the same body control points used in \flow but with isotropic
  metrics of the form 
  $\G_o^\mathrm{basic}(\x) = w_o^\mathrm{max}\I$ for each control point, with 
  $w_o^\mathrm{max}$ matching
  the value \flow uses. Similarly, the attractor uses the same attractor potential
  as \flow, but with a constant isotropic metric with the form
  $\G_a^\mathrm{basic}(\x) = w_a^\mathrm{max}\I$.
\item {\bf PF-nonlinear:} This variant matches PF-basic in construction, except
  it uses a \textit{nonlinear} isotropic metrics of the form 
  $\G_o^\mathrm{nlin}(\x_i) = w_o(\x)\I$ and $\G_a^\mathrm{nlin}(\x_i) = w_a(\x)\I$
  for obstacle-avoidance and attraction, respectively, using weight functions 
  matching \flow.
\end{enumerate}

\textit{A note on curvature terms:} PF-basic uses constant metrics, so has no curvature
terms; PF-nonlinear has nontrivial curvature terms arising from the spatially
varying metrics, but we ignore them here to match common practice from 
the OSC literature.

\textit{Parameter scaling of PF-basic:} Isotropic metrics do not express
spacial directionality toward obstacles, and that leads to 
an inability of the system to effectively trade off the
competing controller requirements. That conflict results in more collisions and 
increased instability. We,
therefore,
compare PF-basic under these baseline metric weights (matching \flow) 
with variants that 
incrementally strengthen collision avoidance controllers and C-space postural
controllers ($f_\mathcal{C}(\q, \qd) = \gamma_p(\q_0 - \q) - \gamma_d\qd$)
to improve these performance measures in the experiment. 
We use the following weight scalings (first entry denotes the obstacle 
metric scalar, and the second entry denotes the C-space metric scalar):
``low'' $(3, 10)$,
``med'' $(5, 50)$, and 
``high'' $(10, 100)$.

\textit{Environments:} 
We run each of these variants on $6$ obstacle environments with $20$ randomly 
sampled target locations each distributed on the opposite side 
of the obstacle field from the robot. Three of the environments use four smaller 
obstacles (depicted in panel 3 of Fig.~\ref{fig:robots}), and the remaining
three environments used two large obstacles (depicted in panel 4 of
Fig.~\ref{fig:robots}). Each environment used 
the same $20$ targets to avoid implicit sampling bias in target choice. 

\subsection{Performance measures}

We report results in Fig.~\ref{fig:reach} in terms of mean and one standard deviation error bars calculated across the $120$ trials for each of the following performance measures:\footnote{There is no guarantee of feasibility in planning 
problems in general, so in all cases, we measure performance relative to the performance
of \flow, which is empirically stable and near optimal across these 
problems.}
\begin{enumerate}
\item {\it Time to goal (``time''):} Length of time, in seconds, it takes for the 
    robot to reach a convergence state. This convergence state 
    is either the target, or its best-effort local minimum. If the system never converges, as in the case of many potential field trials for infeasible problems, the trial times 
    out after 5 seconds. This metric measures time-efficiency of the movement.
\item {\it C-space path length (``length''):} This is the total path length $\int\|\qd\|dt$ of the movement through the configuration space across the trial. This metric measures how economical the movement is. In many of the potential-field variants with lower weights, we see significant fighting among the controllers resulting in highly inefficient extraneous motions.
\item {\it Minimal achievable distance to goal (``goal distance''):} Measures how close,
    in meters, the system is able to get to the goal with its end-effector.
\item {\it Percent time in collision for colliding trials (``collision intensity''):}
    Given that a trial has a collision, this metric measures the fraction of time 
    the system is in collision throughout the trial. This metric indicates
    the intensity of the collision. Low values indicate short grazing collisions
    while higher values indicate long term obstacle penetration.
\item {\it Fraction of trails with collisions (``collision failure''):} Reports
    the fraction of trials with any collision event. We consider these to be 
    collision-avoidance controller failures.
\end{enumerate}

\subsection{Discussion}

In Figure~\ref{fig:reach}, we see that \flow outperforms each of these variants 
significantly, with some informative trends:
\begin{enumerate}
\item \flow never collides, so its collision intensity and collision failure 
    values are $0$. 
\item The other techniques, progressing from no scaling of 
    collision-avoidance and C-space controller weights to substantial scaling,
    show a profile of substantial collision in the beginning to fewer (but
    still non-zero) collision events in the end. 
    But we note that improvement in collision-avoidance
    is achieved at the expense of time-efficiency and the robot's 
    ability to reach the goal (it is too conservative).
\item Lower weight scaling of both PF-basic and PF-nonlinear 
    actually achieve some faster times and better goal
    distances, but that is because the system pushes directly 
    through obstacles, effectively
    ``cheating'' during the trial. \flow remains highly economical with its 
    best effort reaching behaviors while ensuring the trials remain collision-free.
\item Lower weight scalings of PF-basic are highly uneconomical in their motion 
    reflective of their relative instability. As the C-space weights on the posture
    controllers increase, the stability and economy of motion increase, but, again,
    at the expense of time-efficiency and optimality of the final reach.
\item There is little empirical difference between PF-basic and PF-nonlinear 
    indicating that the defining feature separating \flow from the potential 
    field techniques is its use of a highly nonlinear metric that explicitly
    stretches the space in the direction of  the obstacle as well as in the direction of the velocity 
    toward the target. Those stretchings
    penalize deviations in the stretched directions during combination
    with other controllers while allowing variation along orthogonal directions.
    By being more explicit about how controllers should instantaneously trade
    off with one another, \flow is better able to mitigate the otherwise conflicting
    control signals.
\end{enumerate}

\subsubsection{Summary:} Isotropic metrics do not effectively convey how each collision and
attractor controller should trade off with one another, resulting in a conflict
of signals that obscure the intent of each controller making simultaneous
collision avoidance, attraction, and posture maintainence more difficult.
Increasing the weights of the controllers can improve their effectiveness, but at
the expence of descreased overall system performance. The resulting motions
are slower and less effective in reaching the goal
in spite of more stable behavior and fewer collisions. A key
feature of \flow is its ability to leverage highly nonlinear metrics that
better convey information about how controllers should trade off with one
another, while retaining provable stability guarantees. In combination, these
features result in efficient and economical obstacle avoidance behavior while
reaching toward targets amid clutter.

\section{Details of integrated system} \label{apx:IntegratedSystem}

We demonstrate the integrated vision and motion system on two physical
dual arm manipulation platforms: a Baxter robot from Rethink Robotics, and 
a YuMi robot from ABB. Footage of our fully integrated
system (see start of Section~\ref{sec:experiments} for the link) depicting tasks such 
as pick and place amid clutter, reactive manipulation of a cabinet drawers
and doors with human interaction, \emph{active} leadthrough with collision 
controllers running, and pick and place into a cabinet drawer.\footnote{We have also run the RMP
portion of the system on an ABB IRB120 and a dual arm Kuka manipulation
platform with lightweight collaborative arms. Only the two platforms mentioned
here, the
YuMi and the Baxter, which use the full motion and vision integration, are
shown in the video for economy of space.}

This full integrated system, shown in the supplementary video, uses the RMPs
described in Section~\ref{sec:example RMPs} (detailed in
Appendix~\ref{apx:Practice}) with a slight modification that the curvature
terms are ignored. Instead, we maintain theoretical stability by using sufficient damping terms as described in Section~\ref{sec:1DExample} and by
operating at slower speeds. Generalization of these RMPs between embodiments was anecdotally
pretty consistent, although, as we demonstrate in our experiments, we would
expect more empirical deviation at higher speeds. For these manipulation tasks,
this early version of the system worked well as demonstrated in the video.

For visual perception, we leveraged consumer depth cameras along with two levels
of perceptual feedback:
\begin{enumerate}
\item {\it Ambient world:} For the Baxter system we create a voxelized 
    representation of the unmodeled ambient world, and use distance fields
    to focus the collision controllers on just the closest obstacle points
    surrounding the arms. This methodology is similar in nature 
    to \cite{2017_rss_system}, except we found empirically that attending 
    to only the closest point to a skeleton representation resulted in 
    oscillation in concaved regions where distance functions might result
    in nonsmooth kinks. We mitigate this issue by finding the closest points
    to a \emph{volume} around each control point, effectively smoothing 
    over points of nondifferentiability in the distance field.
\item {\it Tracked objects:} We use the Dense Articulated Real-time
    Tracking (DART) system of \cite{Sch15DAR} to track articulated
    objects in real time through manipulations. This system is able 
    to track both the robot and environmental objects, such 
    as an articulated cabinet, simulataneously to give accurate 
    measurements of their relative configuration effectively obviating 
    the need for explicit camera-world calibration. As long as the 
    system is initialized in the general region of the object locations 
    (where for the cabinet and the robot, that would mean even up to 
    half a foot of error in translation and a similar scale of error in rotation),
    the DART optimizer will snap to the right configuration when turned 
    on. DART sends information about object locations to the motion generation,
    and receives back information about expected joint configurations (priors) 
    from the 
    motion system generating a robust world representation usable in 
    a number of practical real-world manipulation problems.
\end{enumerate}

Each of our behaviors are decomposed as state machines that use visual feedback
to detect transitions, including transitions to reaction states as needed
to implement behavioral robustness. Each arm is represented as a separate robot
for efficiency, receiving real-time information about other arm's 
current state enabling coordination. Both arms are programmed simultaneously
using a high level language that provides the programmer a unified view
of the surrounding world and command of both arms.
